%% file: neurips_clean.tex
\documentclass{article}

     \PassOptionsToPackage{numbers, compress}{natbib}

\usepackage[final]{neurips_2022}

\usepackage[utf8]{inputenc} 
\usepackage[T1]{fontenc}    
\usepackage{hyperref}       
\usepackage{url}            
\usepackage{booktabs}       
\usepackage{amsfonts}       
\usepackage{nicefrac}       
\usepackage{microtype}      
\usepackage{xcolor}         

\input{preamble}

\title{On the Safety of Interpretable Machine Learning:\\A Maximum Deviation Approach}

\author{%
  Dennis Wei\\
  IBM Research\\
  \texttt{dwei@us.ibm.com} \\
  \And
  Rahul Nair\\
  IBM Research\\
  \texttt{rahul.nair@ie.ibm.com}\\
  \And
  Amit Dhurandhar\\
  IBM Research\\
  \texttt{adhuran@us.ibm.com}\\
  \AND
  Kush R. Varshney\\
  IBM Research \\
  \texttt{krvarshn@us.ibm.com} \\
  \And
  Elizabeth M. Daly\\
  IBM Research\\
  \texttt{elizabeth.daly@ie.ibm.com} \\
  \And
  Moninder Singh\\
  IBM Research\\
  \texttt{moninder@us.ibm.com} \\
}

\begin{document}

\maketitle

\begin{abstract}
  Interpretable and explainable machine learning has seen a recent surge of interest. We focus on safety as a key motivation behind the surge and make the relationship between interpretability and safety more quantitative. Toward assessing safety, we introduce the concept of \emph{maximum deviation} via an optimization problem to find the largest deviation of a supervised learning model from a reference model regarded as safe. We then show how interpretability facilitates this safety assessment. For models including decision trees, generalized linear and additive models, the maximum deviation can be computed exactly and efficiently. For tree ensembles, which are not regarded as interpretable, discrete optimization techniques can still provide informative bounds. For a broader class of piecewise Lipschitz functions, we leverage the multi-armed bandit literature to show that interpretability produces tighter (regret) bounds on the maximum deviation. We present case studies, including one on mortgage approval, to illustrate our methods and the insights about models that may be obtained from deviation maximization.
\end{abstract}

\section{Introduction}
\label{sec:intro}

Interpretable and explainable machine learning (ML) has seen a recent surge of interest because it is viewed as a key pillar in making models trustworthy, with implications on fairness, reliability, and safety \citep{Varshney2019}. 
In this paper, we focus on \emph{safety} as a key reason behind the demand for explainability. 
The motivation of safety has been discussed at a qualitative level by several authors \citep{Otte2013,kushtst,doshi2017towards,Rudin2019}. Its role is perhaps clearest 
in the dichotomy between directly interpretable models vs.\ post hoc explanations of black-box models. The former have been called ``inherently safe'' \citep{kushtst} and promoted as the only alternative
in high-risk applications \citep{Rudin2019}. The crux of this argument is that post hoc explanations leave a gap between the explanation and the model producing predictions. Thus, unusual data points may appear to be harmless based on the explanation, but truly cause havoc. We aim to go beyond these qualitative arguments and 
address the following questions quantitatively: 1) What does safety mean for such models, and 2) how exactly does interpretability aid safety? 

Towards answering the first question, we make a conceptual contribution in the form of an optimization problem, intended as a tool for assessing the safety of supervised learning (i.e.\ predictive) models. Viewing these models as functions mapping an input space to an output space, a key way in which these models can cause harm is through grossly unexpected outputs, corresponding to inputs that are poorly represented in training data. 
Accordingly, we approach safety assessment for a model by determining where it deviates the most from the output of a \emph{reference model} and by how much (i.e., its \emph{maximum deviation}). The reference model, which represents expected behavior and is deemed to be safe, could be a model well-understood by domain experts or one that has been extensively ``tried and tested.'' The maximization is done over a \emph{certification set}, a large subset of the input space intended to cover all conceivable inputs to the model. 
These concepts are discussed further in Section~\ref{sec:problem}.

Towards answering the second question, 
in Section~\ref{sec:model} we discuss computation of the maximum deviation for different model classes and show how this is facilitated by interpretability. For model classes regarded as interpretable, including trees, generalized linear and additive models, the maximum deviation can be computed exactly and efficiently by exploiting the model structure. For tree ensembles, which are not regarded as interpretable, discrete optimization techniques can exploit their composition in terms of trees to provide anytime bounds on the maximum deviation. The case of trees is also generalized in a different direction to a broader class of piecewise Lipschitz functions, which we argue cover many popular interpretable functions. Here we show that the benefit of interpretability is significantly tighter regret bounds on the maximum deviation compared with general black-box functions, leveraging 
results from the multi-armed bandit literature. 
More broadly, the development of tailored methods for additional model classes is beyond the scope of this first work on the maximum deviation approach (the black-box optimization of Section~\ref{sec:model:piecewise_Lipschitz} is applicable to all models but obviously not tailored). We discuss in Appendix \ref{sec:model:post_hoc} some possible approaches, and the research gaps to be overcome, for neural networks and to make use of 
post hoc explanations, which approximate a model locally \citep{lime,unifiedPI,CEM} or globally \citep{modelcompr,distill}.

In Section~\ref{sec:expt}, we present case studies that illustrate the deviation maximization methods in Section~\ref{sec:model} for decision trees, linear and additive models, and tree ensembles. It is seen that deviation maximization provides insights about models through studying the feature combinations that lead to 
extreme outputs. These insights can in turn direct further investigation and invite domain expert input. We also quantify 
how the maximum deviation depends on model complexity and the size of the certification set. For tree ensembles, we find that the obtained upper bounds on the maximum deviation are informative, showing that the maximum deviation does not increase with the number of trees in the ensemble. 

Overall, our discussion provides a more quantitative basis for safety assessment of predictive models and 
for preferring more interpretable models due to the greater ease of performing this assessment.

\section{Assessing Safety Through Maximum Deviation}
\label{sec:problem}

We are given a supervised learning model $f$, which is a function mapping an input feature space $\mathcal{X}$ to an output space $\mathcal{Y}$. We wish to assess the safety of this model by finding its largest deviation from a given reference model $f_0: \mathcal{X} \mapsto \mathcal{Y}$ representing expected behavior. To do this, we additionally require 1) a measure of deviation $D: \mathcal{Y} \times \mathcal{Y} \mapsto \mathbb{R}_+$, where $\mathbb{R}_+$ is the set of non-negative reals, and 2) a certification set $\certset \subseteq \mathcal{X}$ over which the deviation is maximized. Then the problem to be solved is 
\begin{equation}\label{eqn:maxDev}
    \max_{x\in\certset} \; D(f(x), f_0(x)).
\end{equation}
The deviation is worst-case because the maximization is over all $x \in \certset$; further implications of this are discussed in Appendix~\ref{sec:discuss}. {Note that \eqref{eqn:maxDev} is different than typical robust \emph{training} 
where the focus is to learn a model that minimizes some worst case loss, as opposed to finding regions in $\mathcal{X}$ where two already trained models differ significantly.} 

We view problem~\eqref{eqn:maxDev} as only a \emph{means} toward the goal of evaluating safety. 
In particular, a large deviation value is not necessarily indicative of a safety risk, as two models may differ significantly for valid reasons. For example, one model may capture a useful pattern that the other does not. We thus think that it would be overly simplistic to regard the maximum deviation as just another metric to be optimized in selecting models. What large deviation values do indicate, however, is a (possibly) sufficient reason for further investigation. Hence, the maximizing solutions in \eqref{eqn:maxDev} (i.e., the $\argmax$) are of as much operational 
interest as the maximum values (this will be illustrated in Section~\ref{sec:expt}).

We now elaborate on elements in problem formulation \eqref{eqn:maxDev}.

\textbf{Output space $\mathcal{Y}$.} In the case of regression, $\mathcal{Y}$ is the set of reals $\mathbb{R}$ or an interval thereof. In the case of binary classification, while $\mathcal{Y}$ could be $\{0, 1\}$ or $\{-1, +1\}$, these limit the possible deviations to binary values as well (``same'' or ``different''). Thus to provide more informative results, we take $\mathcal{Y}$ to be the space of real-valued scores that are thresholded to produce a binary label. For example, $y$ could be a predicted probability in $[0, 1]$ or a log-odds ratio in $\mathbb{R}$. Similarly for multi-class classification with $M$ classes, $\mathcal{Y} \subset \mathbb{R}^M$ could be a $M$-dimensional space of real-valued scores. {In Appendix~\ref{sec:problem:add}, we discuss considerations in choosing the deviation function $D$ as well as models that abstain.} 

\textbf{Reference model $f_0$.} The premise of the reference model is that it should capture expected behavior while being ``safe''. 
The simplest case 
is for $f_0$ to be a constant function representing a baseline value, for example zero or a mean prediction. We consider the more general case where $f_0$ may vary with $x$. 
Below we give several examples of reference models to address the natural question of how they might be obtained. The examples can be categorized as 1) existing domain-specific models, 2) interpretable ML models validated by domain knowledge, and 3) extensively tested and deployed models. The first two categories are prevalent in high-stakes domains where interpretability is critical.
\begin{enumerate}
    \item \textbf{Existing domain-specific models:} These models originate from an application domain and may not be based on ML at all. For example in consumer finance, several industry-standard models compute credit scores from a consumer's credit information (the FICO score is the best-known in the US). Similarly in medicine, scoring systems (sparse linear models with small integer coefficients) abound for assessing various risks (the CHADS$_2$ score for stroke risk is well-known, see the "Scoring Systems: Applications and Prior Art" section of \cite{rudin2018optimized} for a list of others). These models have been used for decades by thousands of practitioners so they are well understood. They may very well be improved upon by a more ML-based model, but for such a model to gain acceptance with domain experts, any large deviations from existing models need to be examined and understood.
    \item \textbf{Interpretable models validated by domain knowledge:} Here, an interpretable ML model is learned from data and is validated by domain experts in some way, for example by selecting important input features or by carefully inspecting the trained model. We provide two real examples: In semiconductor manufacturing, process engineers typically want decision trees \cite{ProfWeight} to model their respective manufacturing process (e.g.~etching, polishing, rapid thermal processing, etc.) since they are comfortable understanding and explaining them to their superiors, which is critical especially when things go out-of-spec. Hence, a tree built from data (or any model in general) would only be allowed to make automated measurement predictions if the features it highlights (viz.~pressures, gas flows, temperatures) make sense for the specific process. Similarly, in predicting failures of industrial assets such as wind turbines, some failure data is available to train models but experts in these systems (e.g.~engineers) may also be consulted. They have knowledge that can help validate the model, for example which components are most likely to cause failures or which environmental variables (e.g.~temperature) are most influential.
    \item \textbf{Extensively tested and deployed models:} A reference model may also be one that is not necessarily informed by domain knowledge but has been extensively tested, deployed, and/or approved by a regulator. For medical devices that use ML models, the US Food and Drug Administration (FDA) has instituted a risk-based regulatory system. Any system updates or changes, for instance changes in model architecture, retraining based on new data, or changes in intended use (e.g. use for pediatric cases for devices approved only for adults), need to either seek new approvals or demonstrate ``substantial equivalence'' by providing supporting evidence that the revised model is similar to a previously approved device. In the latter case, the reference model is the approved device and small maximum deviation serves as evidence of equivalence. As another example, consider a ML-based recommendation model for products of an online retailer or articles on a social network, where because of the scale, a tree ensemble may be used for its fast inference time as well as its modeling flexibility \cite{ilic2017evaluating}. In this case, a model that has been deployed for some time could be the reference model, since it has been extensively tested during this time even though human validation of it may be limited. When a new version of the model is trained on newer data or improved in some fashion, finding its maximum deviation from the reference model can serve as one safety check before deploying it in place of the reference model.
\end{enumerate}

\textbf{Certification set $\certset$.} The premise of the certification set is that it contains all inputs that the model might conceivably be exposed to. This may include inputs that are highly improbable but not physically or logically impossible (for example, a severely hypothermic human body temperature of 27\textdegree C). Thus, while $\certset$ might be based on the support set of a probability distribution or data sample, it does not depend on the likelihood of points within the support. The set $\certset$ may also be a strict superset of the training data domain. For example, a model may have been trained on data for males, and we would now like to determine its worst-case behavior on an unseen population of females.

For tabular or lower-dimensional data, $\certset$ might be the entire input space $\mathcal{X}$. For non-tabular or higher-dimensional data, the choice $\certset = \mathcal{X}$ may be too unrepresentative 
because the manifold of realistic inputs is lower in dimension. In this case, if we have a dataset $\{x_i\}_{i=1}^n$, one possibility is to use a union of $\ell_p$ balls centered at $x_i$,
\begin{equation}\label{eqn:unionBalls}
\certset = \bigcup_{i=1}^n \mathcal{B}_r^p[x_i], \qquad \mathcal{B}_r^p[x_i] = \{x \in \mathcal{X} : \lVert x - x_i \rVert_p \leq r\}.
\end{equation}
The set $\certset$ is thus comprised of points somewhat close to the $n$ observed examples $x_i$, but the radius $r$ does not have to be ``small''.

In addition to determining the maximum deviation over the entire set $\certset$, maximum deviations over subsets of $\certset$ (e.g., different age groups) may also be of interest. For example, Appendix~\ref{sec:expt_add:Adult} shows deviation values separately for leaves of a decision tree, which partition the input space.

\section{Related Work}

In previous work on safety and interpretability in ML, 
the authors of \citep{kushtst,MohseniYXYWW2021} give qualitative accounts suggesting that directly interpretable models are an inherently safe design because humans can inspect them to find spurious elements; in this paper, we attempt to make those qualitative suggestions more quantitative and automate some of the human inspection.
Furthermore, several other authors have highlighted safety as a goal for interpretability \citep{Otte2013,doshi2017towards,TomsettBHPC2018,gilpin2018explaining,Rudin2019}, but again without quantitative development. Moreover, the lack of consensus on how to measure interpretability motivates the relationship that we explore between interpretability and the ease of evaluating safety. 

In the area of ML verification, robustness certification methods aim to provide guarantees that the classification remains constant within a radius $\epsilon$ of an input point, while output reachability is concerned with characterizing the set of outputs corresponding to a region of inputs \citep{huang2020survey}.  
A major difference in our work is that we consider \emph{two} models, a model $f$ to be assessed and a reference $f_0$, whereas the above notions of robustness and reachability involve a single model. 
Another important difference is that our focus is \emph{global}, over a comprehensive set $\certset$, rather than local to small neighborhoods around input points; a local focus is especially true of neural network verification \cite{ehlers2017formal,weng2018towards,wong2018provable,dvijotham2018dual,singh2018fast,raghunathan2018semidefinite,katz2019marabou,tjeng2019evaluating,anderson2020strong,dathathri2020enabling}. We also study the role of interpretability in safety verification. 
Works in robust optimization applied to ML minimize the worst-case probability of error, but this worst case is over parameters of $f$ rather than values of $x$ \citep{lanckriet2002robust}. 
\citet{thomas2019preventing} present a framework where during model training, a set of safety tests is specified by the model designer in order to accept or reject the possible solution. 

We build on related literature on 
robustness and explainability that deals specifically with tree ensembles. Mixed-integer programming (MIP) and discrete optimization have been proposed to find the smallest input perturbation to `evade' a classifier \citep{kantchelian2016evasion} 
and to obtain counterfactual explanations \citep{parmentier2021optimal}. MIP approaches are computationally intensive however. To address this \citet{chen2019robustness} introduce graph based approaches for verification on trees. Their central idea, which we use, is to discretize verification computations onto a graph constructed from the way leaves intersect. The verification problem is transformed to finding all maximum cliques. \citet{devos2021versatile} expand on this idea by providing anytime bounds by probing unexplored nodes.

Safety has become a critical issue in reinforcement learning (RL) with multiple works focusing on making RL policies safe \citep{AmodeiOSCSM2016,vRL,iRL,wRL}. There are two broad themes \citep{RLsurvey}: (i) a safe and verifiable policy is learned at the outset by enforcing certain constraints, and (ii) post hoc methods are used to identify bad regimes or failure points of an existing policy. Our current proposal is complementary to these works as we focus on the supervised learning setup viewed from the lens of interpretability. Nonetheless, ramifications of our work in the RL context are briefly discussed in Appendix \ref{sec:discuss}.

\section{Deviation Maximization for Specific Model Classes}
\label{sec:model}

In this section, we discuss approaches to computing the maximum deviation \eqref{eqn:maxDev} 
for $f$ belonging to various model classes. We show the benefit of interpretable model structure in different guises. Exact and efficient computation is possible for decision trees, and generalized linear and additive models in Sections~\ref{sec:model:tree} and \ref{sec:model:linear_additive}. In Section~\ref{sec:model:ensemble}, the composition of tree ensembles in terms of trees allows discrete optimization methods to provide anytime bounds. For a general class of piecewise Lipschitz functions in Section~\ref{sec:model:piecewise_Lipschitz}, the application of multi-arm bandit results yields tighter regret bounds on the maximum deviation. {While some of the results in this section 
may be less surprising, one of our contributions is to identify precise properties that allow them to hold. We also show that intuitive measures of model complexity, such as the number of leaves or pieces or smoothness of functions, have an additional interpretation 
in terms of the complexity of maximizing deviation.} More broadly, the development of methods specific to additional model classes is beyond the scope of a single work. We discuss in Appendix~\ref{sec:model:post_hoc} possible approaches and the advances needed for neural networks (beyond applying the black-box methods of Section~\ref{sec:model:piecewise_Lipschitz}) and to make use of post hoc explanations. 

To develop mathematical results and efficient algorithms, we will sometimes assume that the reference model $f_0$ is from the same class as $f$. 
In Appendix \ref{sec:discuss}, we discuss the case where 
$f_0$ may not be globally interpretable, but may be so in local regions. We will also sometimes assume that the certification set $\certset$ and other sets are Cartesian products. This means that $\certset = \prod_{j=1}^d \certset_j$, where for a continuous feature $j$, $\certset_j = [\underline{X}_j, \overline{X}_j]$ is an interval, and for a categorical feature $j$, $\certset_j$ is a set of categories. {We mention relaxations of the Cartesian product assumption in Appendix~\ref{sec:model:linear_additive_add}.}

\subsection{Trees}
\label{sec:model:tree}

We begin with the case where $f$ and $f_0$ are both decision trees. 
A decision tree with $L$ leaves partitions the input space $\mathcal{X}$ into $L$ corresponding parts, which we also refer to as `leaves'. We consider only non-oblique trees. In this case, each leaf is described by a conjunction of conditions on individual features and is therefore a Cartesian product as defined above. 
With $\mathcal{L}_l \subset \mathcal{X}$ denoting the $l$th leaf and $y_l \in \mathcal{Y}$ the output value assigned to it, tree $f$ is described by the function 
\begin{equation}\label{eqn:ftree}
f(x) = y_l \quad \text{if } x \in \mathcal{L}_l, \quad l = 1,\dots,L,
\end{equation}
and similarly for tree $f_0$ with leaves $\mathcal{L}_{0m}$ and outputs $y_{0m}$, $m = 1,\dots, L_0$. As discussed in \citep{yang2017scalable,angelino2018learning}, \emph{rule lists} where each rule is a condition on a single feature are one-sided trees in the above sense.

The partitioning of $\mathcal{X}$ by decision trees and their piecewise-constant nature simplify the computation of the maximum deviation \eqref{eqn:maxDev}. 
Specifically, the maximization can be restricted to pairs of leaves $(l, m)$ for which the intersection $\mathcal{L}_l \cap \mathcal{L}_{0m} \cap \certset$ is non-empty. The intersection of two leaves $\mathcal{L}_l \cap \mathcal{L}_{0m}$ is another Cartesian product, and we assume that it is tractable to determine whether $\certset$ intersects a given Cartesian product (see examples in Appendix~\ref{sec:model:tree_add}). 

For visual representation and later use in Section~\ref{sec:model:ensemble}, it is useful to define a bipartite graph, with $L$ nodes representing the leaves $\mathcal{L}_l$ of $f$ on one side and $L_0$ nodes representing the leaves $\mathcal{L}_{0m}$ of $f_0$ on the other. Define the edge set $\mathcal{E} = \{(l, m) : \mathcal{L}_l \cap \mathcal{L}_{0m} \cap \certset \neq \emptyset\}$; clearly $\lvert \mathcal{E} \rvert \leq L_0 L$. 
Then 
\begin{equation}\label{eqn:maxDevTree}
\max_{x\in\certset} \; D(f(x), f_0(x)) = \max_{(l,m) \in \mathcal{E}} D(y_l, y_{0m}).
\end{equation}

We summarize the complexity of deviation maximization for decision trees as follows. {This is a slight refinement of \cite[Thm.~1]{chen2019robustness} in the case $K=2$, see Appendix~\ref{sec:model:tree_add} for details.}
\begin{proposition}\label{prop:tree}
Let $f$ and $f_0$ be decision trees as in \eqref{eqn:ftree} with $L$ and $L_0$ leaves respectively, 
and $\mathcal{E}$ be the bipartite edge set of leaf intersections defined above. Then the maximum deviation \eqref{eqn:maxDev} can be computed with $\lvert \mathcal{E} \rvert$ evaluations as shown in \eqref{eqn:maxDevTree}.
\end{proposition}

\subsection{Linear and additive models}
\label{sec:model:linear_additive}

In this subsection, we assume that $f$ is a generalized additive model (GAM) given by 
\begin{equation}\label{eqn:GAM}
    f(x) = g^{-1}\left(\sum_{j=1}^d f_j(x_j) \right),
\end{equation}
where each $f_j$ is an arbitrary function of feature $x_j$. In the case where $f_j(x_j) = w_j x_j$ for all continuous features $x_j$, where $w_j$ is a real coefficient, \eqref{eqn:GAM} is a generalized linear model (GLM). We discuss the treatment of categorical features in Appendix~\ref{sec:model:linear_additive_add}. The invertible link function $g: \mathbb{R} \mapsto \mathbb{R}$ is furthermore assumed to be monotonically increasing. This assumption 
is satisfied by common GAM link functions: identity, logit ($g(y) = \log(y / (1-y))$), and logarithmic. 

Equation~\eqref{eqn:GAM} implies that $\mathcal{Y} \subset \mathbb{R}$ and the deviation $D(y, y_0)$ is a function of two scalars $y$ and $y_0$. For this scalar case, we make the following intuitively reasonable assumption throughout the subsection.
\begin{assumption}\label{ass:monotone}
For $y, y_0 \in \mathcal{Y} \subseteq \mathbb{R}$,
    1) $D(y, y_0) = 0$ whenever $y = y_0$;
    2) $D(y, y_0)$ is monotonically non-decreasing in $y$ for $y \geq y_0$ and non-increasing in $y$ for $y \leq y_0$.
\end{assumption}

Our approach 
is to exploit the additive form of \eqref{eqn:GAM} by reducing problem~\eqref{eqn:maxDev} to the optimization 
\begin{equation}\label{eqn:maxAdditive}
    M_{\pm}(f, \mathcal{S}) = \maxmin_{x \in \mathcal{S}} \; \sum_{j=1}^d f_j(x_j), 
\end{equation}
for different choices of $\mathcal{S} \subset \mathcal{X}$ and where $+$ corresponds to $\max$ and $-$ to $\min$. 
We discuss below how 
this can be done for two types of reference model $f_0$: decision tree (which includes the constant case $L_0 = 1$) and additive. For the first case, we prove the following result in Appendix~\ref{sec:model:linear_additive_add}: 
\begin{proposition}\label{prop:additive}
Let $f$ be a GAM as in \eqref{eqn:GAM} and $\mathcal{S}$ be a subset of $\mathcal{X}$ where $f_0(x) \equiv y_0$ is constant. Then if Assumption~\ref{ass:monotone} holds, 
\[
\max_{x\in\mathcal{S}} \; D(f(x), f_0(x)) = \max_{\sigma\in\{+,-\}} D\left( g^{-1}(M_\sigma(f, \mathcal{S})), y_0\right).
\]
\end{proposition}

\textbf{Tree-structured $f_0$.} Since $f_0$ is piecewise constant over its leaves $\mathcal{L}_{0m}$, $m = 1,\dots,L_0$, we take $\mathcal{S}$ to be the intersection of $\certset$ with each $\mathcal{L}_{0m}$ in turn and apply Proposition~\ref{prop:additive}. The overall maximum is then obtained as the maximum over the leaves,
\begin{equation}
\max_{x\in\certset} \; D(f(x), f_0(x)) =\\ 
\max_{m=1,\dots,L_0} \max_{\sigma\in\{+,-\}} D\left( g^{-1}(M_\sigma(f, \mathcal{L}_{0m}\cap\certset)), y_{0m}\right).
\label{eqn:maxDevAdditiveTree}
\end{equation}
This reduces \eqref{eqn:maxDev} to solving $2 L_0$ instances of \eqref{eqn:maxAdditive}.

\textbf{Additive $f_0$.} For this case, we make the additional assumption that the link function $g$ in \eqref{eqn:GAM} is the identity function, as well as Assumption~\ref{ass:difference} below. The implication of these assumptions is discussed in Appendix~\ref{sec:model:linear_additive_add}.
\begin{assumption}\label{ass:difference}
$D(y, y_0) = D(y - y_0)$ is a function only of the difference $y - y_0$.
\end{assumption}
\noindent Then $f_0(x) = \sum_{j=1}^d f_{0j}(x_j)$ and the difference $f(x) - f_0(x)$ is also additive. Using Assumptions~\ref{ass:difference}, \ref{ass:monotone} and a similar argument as in the proof of Proposition~\ref{prop:additive}, the maximum deviation is again obtained by 
maximizing and minimizing an additive function, resulting in two instances of \eqref{eqn:maxAdditive} with $\mathcal{S} = \certset$:
\[
\max_{x\in\certset} \; D(f(x), f_0(x)) = \max_{\sigma\in\{+,-\}} D\bigl( M_\sigma(f - f_0, \certset)\bigr).
\]

\textbf{Computational complexity of \eqref{eqn:maxAdditive}.} 
For the case of nonlinear additive $f$, we additionally assume that $\certset$ is a Cartesian product. 
It follows that $\mathcal{S} = \prod_{j=1}^d \mathcal{S}_j$ is a Cartesian product (see Appendix~\ref{sec:model:linear_additive_add} for the brief justification) and \eqref{eqn:maxAdditive} separates into one-dimensional optimizations over $\mathcal{S}_j$, 
\begin{equation}\label{eqn:maxAdditiveSep}
\maxmin_{x \in \mathcal{S}} \; \sum_{j=1}^d f_j(x_j) = \sum_{j=1}^d \maxmin_{x_j \in \mathcal{S}_j} f_j(x_j).
\end{equation}

The computational complexity of \eqref{eqn:maxAdditiveSep} is thus $\sum_{j=1}^d C_j$, where $C_j$ is the complexity of the $j$th one-dimensional optimization. We discuss different cases of $C_j$ in Appendix~\ref{sec:model:linear_additive_add}; the important point is that the overall complexity is linear in $d$. 

In the GLM case where $\sum_{j=1}^d f_j(x_j) = w^T x$, problem~\eqref{eqn:maxAdditive} is simpler and it is less important that $\certset$ be a Cartesian product. In particular, if $\certset$ is a convex set, so too is $\mathcal{S}$ (again see Appendix~\ref{sec:model:linear_additive_add} for justification). Hence \eqref{eqn:maxAdditive} is a convex optimization problem.

\subsection{Tree ensembles}
\label{sec:model:ensemble}

We now extend the idea used for single decision trees in Section \ref{sec:model:tree} to tree ensembles. This class covers several popular methods such as Random Forests and Gradient Boosted Trees. It can also cover \emph{rule} ensembles \citep{friedman2008,dembczynski2010} as a special case, as explained in Appendix~\ref{app:ensembles:algo}. We assume $f$ is a tree ensemble consisting of $K$ trees and $f_0$ is a single decision tree. Let $\mathcal{L}_{l_k}$ denote the $l$th leaf of the $k$th tree in $f$ for $l=1,\hdots,L_k$, and $\mathcal{L}_{0m}$ be the $m$th leaf $f_0$, for $m=1,\hdots,L_0$. Correspondingly let $y_{l_k}$ and $y_{0m}$ denote the prediction values associated with each leaf. 

Define a graph $\mathcal{G}(\mathcal{V}, \mathcal{E})$, where there is a vertex for each leaf in $f$ and $f_0$, i.e.
\begin{equation}
    \mathcal{V} = \{l_k | \forall k = 1,\dots, K,\, l=1,\hdots,L_k\} \cup \{m | m=1,\hdots,L_0\}.
\end{equation}
Construct an edge for each overlapping pair of leaves in $\mathcal{V}$, i.e.
\begin{equation}
    \mathcal{E} = \{(i, j) | \mathcal{L}_i \cap \mathcal{L}_j \neq \emptyset,\, \forall (i,j)\in V,\, i\ne j \}.
\end{equation}

This graph is a $K+1$-partite graph as leaves within an individual tree do not intersect and are an independent set. Denote $M$ to be the adjacency matrix of $\mathcal{G}$. Following \citet{chen2019robustness}, a maximum clique $S$ of size $K+1$ on such a graph provides a discrete region in the feature space with a computable deviation. A clique is a subset of nodes all connected to each other; a maximum clique is one that cannot be expanded further by adding a node. The model predictions $y_c$ and $y_{0c}$ can be ensembled from leaves in $S$. Denote by $D(S)$ the deviation computed from the clique $S$. Maximizing over all such cliques solves \eqref{eqn:maxDev}. However, complete enumeration is expensive, so informative bounds, either using the merge procedure in \citet{chen2019robustness} or the heuristic function in \citet{devos2021versatile} can be used. We use the latter which exploits the $K+1$-partite structure of $\mathcal{G}$. 

Specifically, we adapt the anytime bounds of \citet{devos2021versatile} as follows. At each step of the enumeration procedure, an intermediate clique $S$ contains selected leaves from trees in $\left[1, \hdots, k\right]$ and unexplored trees in $\left[k+1, \hdots,K+1\right]$. For each unexplored tree, we select a valid candidate leaf that maximizes deviation, i.e.
\begin{equation}
    v_k = \argmax_{l_k,\,l_k\cap i\ne \emptyset,\, \forall i\in S} D(S\cup l_k).
    \label{eq:tree-ensemble:nodesel}
\end{equation}
Using these worst-case leaves, a heuristic function 
\begin{equation}
    H(S) = D(S^\prime) = D(S\bigcup\limits_{m=k+1}^{K+1} v_m)
    \label{eq:tree-ensemble:heuristic}
\end{equation}
provides an upper (dual) bound. In practice, this dual bound is tight and therefore very useful during the search procedure to prune the search space. Each $K+1$ clique provides a primal bound, so the search can be terminated early before examining all trees if the dual bound is less than the primal bound. We adapt the search procedure of \citet{mirghorbani2013finding} to include the pruning arguments. Appendix \ref{app:ensembles:algo} presents the full algorithm. Starting with an empty clique, the procedure adds a single node from each tree to create an intermediate clique. If the size of the clique is $K+1$ the primal bound is updated. Otherwise, the dual bound is computed. A node compatibility vector is used to keep track of all feasible additions.When the search is terminated at any step, the maximum deviation is bounded by $(D_{lb}, D_{ub})$.

The algorithm works for the entire feature space. When the certification set $\certset$ is a union of balls as in \eqref{eqn:unionBalls}, some additional considerations are needed. First, we can disregard leaves that do not intersect with $\certset$ during the graph construction phase. A validation step to ensure that the leaves of a clique all intersect with the same ball in $\certset$ is also needed.

\subsection{Piecewise Lipschitz Functions}
\label{sec:model:piecewise_Lipschitz}

We saw the benefits of having specific (deterministic) interpretable functions as well as their extensions in the context of safety. Now consider a richer class of functions that may also be randomized with finite variance. In this case let $f$ and $f_0$ denote the mean values of the learned and reference functions respectively. We consider the case where each function is either interpretable or black box, where the latter implies that query access is the only realistic way of probing the model. This leads to three cases where either both functions are black box or interpretable, or one is black box. What we care about in all these cases\footnote{For simplicity assume $D(\cdot,\cdot)$ to be the identity function.} is to find the maximum (and minimum) of a function $\Delta(x)=f(x)-f_0(x)$. Let us consider finding only the maximum of $\Delta$ as the other case is symmetric. Given that $f$ and $f_0$ can be random functions $\Delta$ is also a random function and if $\Delta$ is black box a standard way to optimize it is either using Bayesian Optimization (BO) \citep{ucb} or tree search type bandit methods \citep{bubeck, carpentier}. We repurpose some of the results from this latter literature in our context showcasing the benefit of interpretability from a safety standpoint. To do this we first define relevant terms.

\begin{definition}[Simple Regret \citep{bubeck}]
\label{def:sr}
If $f^*_{\certset}$ denotes the optimal value of the function $f$ on the certification set $\certset$, then the simple regret $r_q^{\certset}$ after querying the $f$ function $q$ times and obtaining a solution $x_q$ is given by,
 $   r_q^{\certset}(f) = f^*_{\certset}-f(x_q)$.
\end{definition}

\begin{definition}[Order $\beta$ c-Lipschitz]
\label{def:lc}
Given a (normalized) metric $\ell$ a function $f$ is c-Lipschitz continuous of order $\beta>0$ if for any two inputs $x$, $y$ and for $c>0$ we have,
$|f(x)-f(y)| \le c\cdot\ell(x,y)^{\beta}$.   
\end{definition}

\begin{definition}[Near optimality dimension \citep{bubeck}]
\label{def:nod}
If $\mathcal{N}(\certset,\ell,\epsilon)$ is the maximum number of $\epsilon$ radius balls one can fit in $\certset$ given the metric $\ell$ and $\certset_{\epsilon}=\{x\in \certset|f(x)\ge f^*_{\certset}-\epsilon\}$, then for $c>0$ the c-near optimality dimension is given by,
$\upsilon = \max\left(\limsup_{\epsilon\rightarrow 0}\frac{\ln \mathcal{N}(\certset_{c\epsilon},\ell,\epsilon)}{\ln (\epsilon^{-1})},0 \right) $.
\end{definition}
Intuitively, simple regret measures the deviation between our current best and the optimal solution. The Lipschitz condition bounds the rate of change of the function. Near optimality dimension measures the set size for which the function has close to optimal values. 
The lower the value of $\upsilon$, the easier it is to find the optimum. We now define what it means to have an interpretable function.

\begin{assumption}[Characterizing an Interpretable Function]
\label{ass:int}
If a function $f$ is interpretable, then we can (easily) find $1\le m \ll n$ partitions $\{\certset^{(1)},...,\certset^{(m)}\}$ of the certification set $\certset$ such that the function $f^{(i)}=\{f(x)|x\in \certset^{(i)}\}$ $\forall i\in \{1,...,m\}$ in each partition is c-Lipschitz of order $\beta$.
\end{assumption}

\emph{Note that the (interpretable) function overall does not have to be c-Lipschitz of bounded order, rather only in the partitions}. This assumption is motivated by observing different interpretable functions. For example, in the case of decision trees the $m$ partitions could be its leaves, where typically the function is a constant in each leaf ($c=0$). For rule lists as well a fixed prediction is usually made by each rule. For a linear function one could consider the entire input space (i.e. $m=1$), where for bounded slope $\alpha$ the function would also satisfy our assumption ($c=\alpha$ and $\beta=1$). {Examples of models that are not piecewise constant or globally Lipschitz are oblique decision trees (Murthy et al., 1994), regression trees with linear functions in the leaves, and functional trees.} Moreover, $m$
is likely to be small so that the overall model is interpretable (viz. shallow trees or small rules). With the above definitions and Assumption \ref{ass:int} we now provide the simple regret for the function $\Delta$.

\noindent\textbf{1. Both black box models:} If both $f$ and $f_0$ are black box then it seems no gains could be made in estimating the maximum of $\Delta$ over standard results in bandit literature. Hence, using Hierarchical Optimistic Optimization (HOO) with assumptions such as $\certset$ being compact and $\Delta$ being weakly Lipschitz \citep{bubeck} with near optimality dimension $\upsilon$ the simple regret after $q$ queries is:
\begin{equation}
    \label{eq:bbb}
    r_q^{\certset}(\Delta)\le O\left(\left(\frac{\ln(q)}{q}\right)^{\frac{1}{\upsilon+2}}\right)
\end{equation}

\noindent\textbf{2. Both interpretable models:} If both $f$ and $f_0$ are interpretable, then for each function based on Assumption \ref{ass:int} we can find $m_1$ and $m_0$
partitions of $\certset$ respectively where the functions are $c_1$ and $c_0$-Lipschitz of order $\beta_1$ and $\beta_0$  respectively. Now if we take non-empty intersections of these partitions where we could have a maximum of $m_1m_0$ partitions, the function $\Delta$ in these partitions would be $c=2\max(c_0,c_1)$-Lipschitz of order $\beta=\min(\beta_0,\beta_1)$ as stated next (proof in appendix).
\begin{proposition}
\label{prop:lip}
If functions $h_0$ and $h_1$ are $c_0$ and $c_1$ Lipschitz of order $\beta_0$ and $\beta_1$ respectively, then the function $h=h_0-h_1$ is c-Lipschtiz of order $\beta$, where $c=2\max(c_0,c_1)$ and $\beta=\min(\beta_0,\beta_1)$.
\end{proposition}
Given that $\Delta$ is smooth in these partitions with underestimated smoothness of order $\beta$, the simple regret after $q_i$ queries in the $i^{\text{th}}$ partition $\certset^{(i)}$ with near optimality dimension $\upsilon_i$ based on HOO is:
$    r_{q_i}^{\certset^{(i)}}(\Delta)\le O\left(q_i^{-1/(\upsilon_i+2)} \right)$,
where $\upsilon_i\le \frac{d}{\beta}$.
If we divide the overall query budget $q$ across the $\pi \le m_0m_1$ non-empty partitions equally, then the bound will be scaled by $\pi^{1/(\upsilon_i+2)}$ when expressed as a function of $q$. Moreover, the regret for the entire $\certset$ can then be bounded by the maximum regret across these partitions leading to \begin{equation}
    \label{eq:bint}
    r_q^{\certset}(\Delta)\le O\left(\left(\frac{\pi}{q}\right)^{\frac{\beta}{d+2\beta}}\right)
\end{equation}
Notice that for a model to be interpretable $m_0$ and $m_1$ are likely to be small (i.e. shallow trees or small rule lists or linear model where $m=1$) leading to a ``smallish" $\pi$ and $\upsilon$ can be much $>>\frac{d}{\beta}$ in case 1. Hence, interpretability reduces the regret in estimating the maximum deviation.

\noindent\textbf{3. Black box and interpretable model:} Making no further assumptions on the black box model and assuming $\Delta$ satisfies properties mentioned in case 1, the simple regret has the same behavior as \eqref{eq:bbb}. This is expected as the black box model could be highly non-smooth.

\section{Case Studies}
\label{sec:expt}

We present case studies to serve three purposes: 1) show that deviation maximization can lead to insights about models, 2) illustrate the maximization methods developed in Section~\ref{sec:model}, and 3) quantify the dependence of the maximum deviation on model complexity and certification set size (mostly in Appendix~\ref{sec:expt_add}). Two datasets are featured: a sample of US Home Mortgage Disclosure Act (HMDA) data (see Appendix~\ref{sec:expt_add:HMDA} for details), meant as a proxy for a mortgage approval scenario, and the UCI Adult Income dataset \citep{Dua:2019}, a standard tabular dataset with mixed data types. A subset of results is shown in this section with full results, experimental details, and an additional Lending Club dataset in Appendix~\ref{sec:expt_add}. 
Since these are binary classification datasets, we take the deviation function $D$ to be the absolute difference between predicted probabilities of class 1. For the certification set $\certset$, we consider a union of $\ell_\infty$ balls \eqref{eqn:unionBalls} centered at test set instances. {While we have used the test set here, any not necessarily labelled dataset would suffice.} 
The case $r = 0$ yields a finite set consisting only of the test set, while $r \to \infty$ corresponds to $\certset$ being the entire domain $\mathcal{X}$. We reiterate that the dependence of the certification set on a chosen dataset is only on (an expanded version of) the support of the dataset and not on other aspects of the data distribution.

To demonstrate insights from deviation maximization, we study the solutions that maximize deviation (the $\argmax$ in \eqref{eqn:maxDev}) and discuss three examples below. 

\textbf{Identification of an artifact:} The first example comes from the Adult Income dataset, where the reference model $f_0$ is a decision tree (DT) 
and $f$ is an Explainable Boosting Machine (EBM) \citep{nori2019interpretml}, a type of GAM 
(plots of both in Appendix~\ref{sec:expt_add:Adult}). Here, the capital loss feature is the largest contributor to the maximum deviation (the discussion below Table~\ref{tab:HMDA_GAM_max} explains how this is determined), and Table~\ref{tab:Adult_GAM} in Appendix~\ref{sec:expt_add:Adult} shows that as the certification set radius $r$ increases, the maximizing values of capital loss converge to the interval $[1598, 1759]$. The plot of the GAM shape function $f_j$ for capital loss in Figure~\ref{fig:Adult_GAM_CapitalLoss} shows that this interval corresponds to a curiously low value of the function. This low region may be an artifact warranting further investigation since it seems anomalous compared to the rest of the function, and since individuals who report capital losses on their income tax returns to offset capital gains usually have high income (hence high log-odds score). Note that this potential artifact was automatically identified through deviation maximization.

For the next two examples, we consider a simplified mortgage approval scenario using the HMDA dataset. Suppose that a 
DT $f_0$ (shown in Figure~\ref{fig:HMDA_DT} in Appendix~\ref{sec:expt_add:HMDA}) has been trained to make final decisions on mortgage applications. Domain experts have determined that this DT is sensible and safe and are now looking to improve upon it by exploring EBMs. (Logistic regression (LR) models are deferred to Appendix~\ref{sec:expt_add:HMDA} because they do not have higher balanced accuracy than $f_0$.)

\textbf{Conflict between $f$, $f_0$:} We first examine solutions that result in the most positive difference between the predicted probabilities of an EBM $f$ with parameter \texttt{max\_bins}=32 and the DT $f_0$. These all occur in a leaf of $f_0$ (leaf 2 in Figure~\ref{fig:HMDA_DT}) where the applicant's debt-to-income (DTI) ratio is too high ($>$ 52\%) and $f_0$ predicts a low probability of approval. The other salient feature of the solutions is that they all have `preapproval'=1, indicating that a preapproval was requested, which is given a large weight by the EBM $f$ in favor of approval (see Figure~\ref{fig:HMDA_GAM_functions} in Appendix~\ref{sec:expt_add:HMDA}, and Table~\ref{tab:HMDA_GAM_max} for more feature values). Thus $f$ and $f_0$ are in conflict. Among different ways in which the conflict could be resolved, a domain expert might decide that $f_0$ remains correct in rejecting risky applicants with high DTI, even if there is a preapproval request and the new EBM model puts a high weight on it.

\begin{table*}[t]
    \small
    \centering
    \begin{tabular}{lrrrrrr}
    \toprule
$r$&debt\_to\_income (\%)&state&loan\_to\_value (\%)&aus\_1&prop\_value (000\$)&income (000\$)\\
\midrule
0.0&46.0&CA&95.0&3&415&77.0\\
0.1&[45.9 46.5]&none&[100.   100.92]&1&[120 120]&$\leq$28.5\\
0.2&[45.5 46.5]&none&[100.   100.92]&1&[120 120]&$\leq$28.5\\
0.4&[52. 52.]&none&[100.   100.92]&1&[120 120]&$\leq$28.5\\
0.6&[52. 52.]&none&[100.   100.92]&1&[120 120]&$\leq$28.5\\
0.8&[52. 52.]&none&[100.   100.92]&1&[120 120]&$\leq$28.5\\
    \bottomrule
    \end{tabular}
    \caption{Values of top 6 features that maximize difference in predicted probabilities between a decision tree reference model $f_0$ and an Explainable Boosting Machine $f$ (\texttt{max\_bins} $= 32$) on the HMDA dataset.  
    For radius $r > 0$, the maximizing values of continuous features form an interval because the corresponding EBM shape functions $f_j$ are piecewise constant.}
    \label{tab:HMDA_GAM_min}
    \vspace{-2mm}
\end{table*}

\begin{figure*}[th]
    \centering
    \begin{minipage}{0.5\textwidth}
        \includegraphics[width=\textwidth]{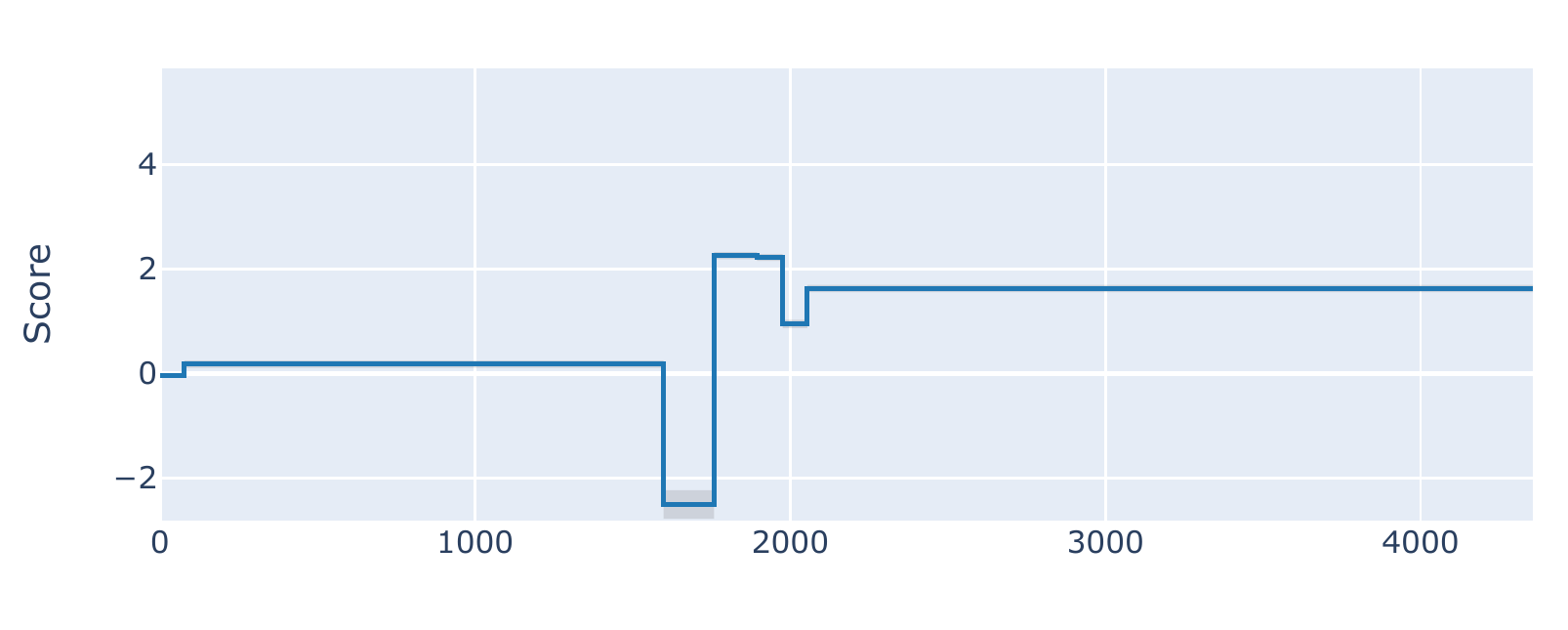}
        \captionof{figure}{EBM shape function $f_j$ for capital loss feature, showing anomalously low interval identified by deviation maximization.}
        \label{fig:Adult_GAM_CapitalLoss}
    \end{minipage}
    \hfill
    \begin{minipage}{0.45\textwidth}
        \includegraphics[width=\textwidth]{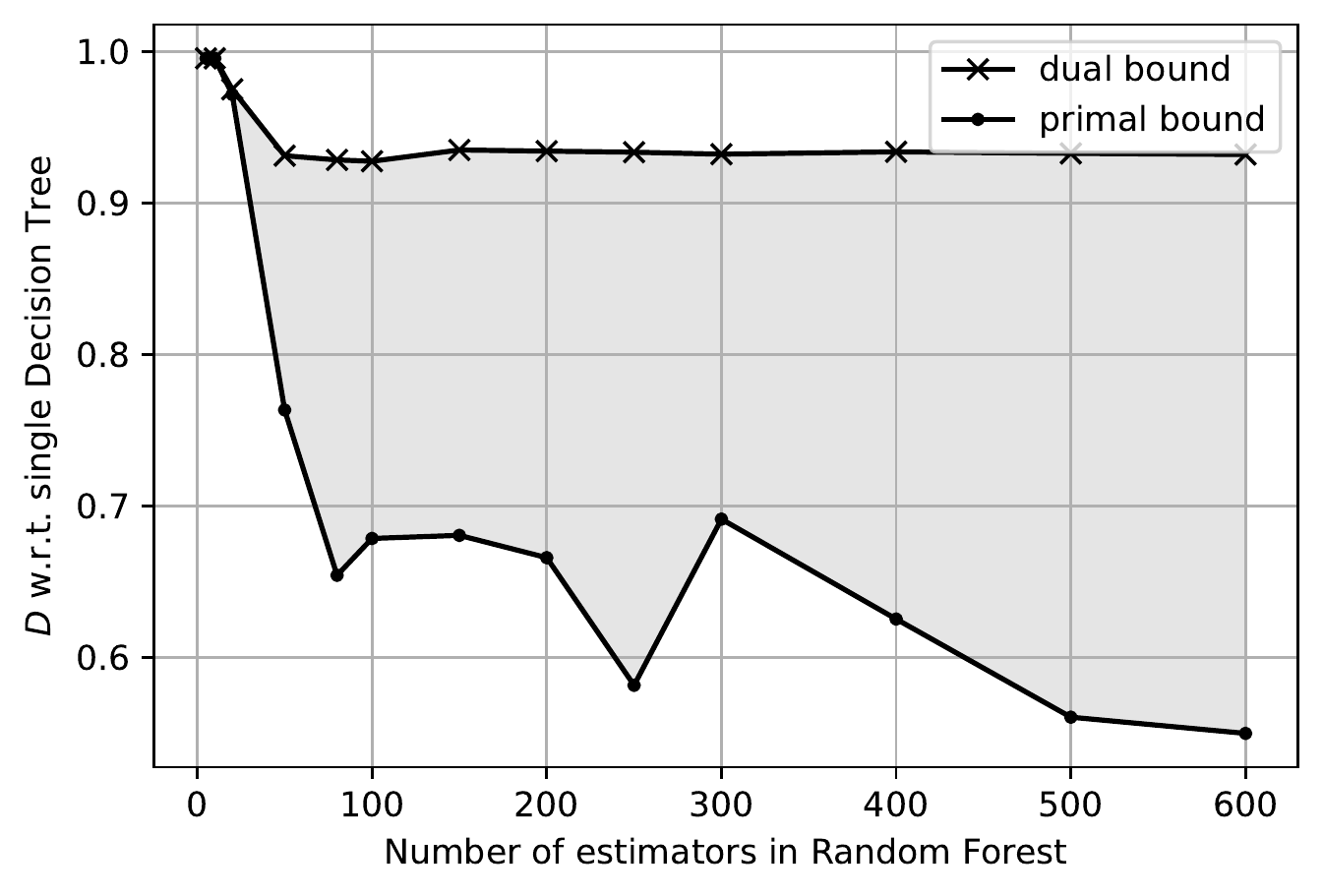}
        \captionof{figure}{Maximum deviation $D$ on the Adult Income dataset as a function of number of estimators in a Random Forest.}
        \label{fig:Adult_RF_Nest}
    \end{minipage}
    \vspace{-3mm}
\end{figure*}

\textbf{Trend toward extreme points, deviation can be good:} We now look at solutions that yield the most negative difference between the predicted probabilities of $f$ and $f_0$,  
for $r \leq 0.8$. Table~\ref{tab:HMDA_GAM_min} shows the 6 features that contribute most to the deviation (again see Appendix~\ref{sec:expt_add:HMDA} for details). All of these points lie in a leaf of $f_0$ (leaf 14 in Figure~\ref{fig:HMDA_DT}, denoted $\mathcal{L}_{14}$) that excludes several clearer-cut cases, with the result being a less confident predicted probability of $0.652$ from $f_0$. The feature values that maximize deviation tend toward extreme points of the region $\mathcal{L}_{14}$. Specifically, the values of the continuous features debt-to-income ratio, loan-to-value ratio, property value, and income all move in the direction of application denial. For the latter three features, the boundary of $\mathcal{L}_{14}$ is reached as soon as $r = 0.1$, whereas for debt-to-income ratio, this occurs at $r = 0.4$. The movement toward extremes is expected for this $f$ since its relevant shape functions $f_j$ are mostly increasing or decreasing, as seen in Figure~\ref{fig:HMDA_GAM_functions} in Appendix~\ref{sec:expt_add:HMDA}. The same behavior is observed in other GAM and LR examples in Appendix~\ref{sec:expt_add}. In this example, a domain expert might conclude that the large deviation is in fact desirable because $f$ is providing varying predictions in $\mathcal{L}_{14}$ in ways that make sense, as opposed to the constant given by $f_0$. {This shows that deviation maximization can work in both directions, identifying where the reference model and its assumptions may be too simplistic and giving an opportunity to improve the reference model.}

\textbf{Maximum deviation vs.~number of trees in a RF:}
In Figure~\ref{fig:Adult_RF_Nest}, we highlight one result from a set of such results in Appendix~\ref{sec:expt_add}, showing maximum deviation as a function of model complexity, here quantified by the number of estimators (trees) in a Random Forest (RF). This is a demonstration of the method in Section \ref{sec:model:ensemble}, which in general provides bounds on the maximum deviation. In this case, the upper (``dual'') bound is informative enough to actually show a decrease as the number of estimators increases. The larger number of estimators increases averaging and may serve to make the model smoother.

\section{Conclusion}
\label{sec:concl}

We have considered the relationship between interpretability and safety in supervised learning through two main contributions: First, the proposal of maximum deviation as a means toward assessing safety, and second, discussion of approaches to computing maximum deviation and how these are simplified by interpretable model structure. We believe that there is much more to explore in this relationship. Appendices~\ref{sec:discuss} and \ref{sec:model:post_hoc} provide further discussion of several topics and future directions.

\section*{Acknowledgements}
We thank Michael Hind for several early discussions on the topic of ML model risk assessment that inspired this work, and for his overall leadership on this topic. We also thank Dhaval Patel for a discussion on the industrial assets example in Section~\ref{sec:problem}.

\medskip

{
\small

\bibliography{biblio}
\bibliographystyle{unsrtnat}

}

\section*{Checklist}

\begin{enumerate}

\item For all authors...
\begin{enumerate}
  \item Do the main claims made in the abstract and introduction accurately reflect the paper's contributions and scope?
    \answerYes{}
  \item Did you describe the limitations of your work?
    \answerYes{} See Appendix~\ref{sec:discuss} in particular.
  \item Did you discuss any potential negative societal impacts of your work?
    \answerYes{} See Appendix~\ref{sec:discuss}.
  \item Have you read the ethics review guidelines and ensured that your paper conforms to them?
    \answerYes{}
\end{enumerate}

\item If you are including theoretical results...
\begin{enumerate}
  \item Did you state the full set of assumptions of all theoretical results?
    \answerYes{}
        \item Did you include complete proofs of all theoretical results?
    \answerYes{} 
\end{enumerate}

\item If you ran experiments...
\begin{enumerate}
  \item Did you include the code, data, and instructions needed to reproduce the main experimental results (either in the supplemental material or as a URL)?
    \answerNo{} The code is proprietary at this time due to our institutional obligations.
  \item Did you specify all the training details (e.g., data splits, hyperparameters, how they were chosen)?
    \answerYes{} See Appendix~\ref{sec:expt_add}.
        \item Did you report error bars (e.g., with respect to the random seed after running experiments multiple times)?
    \answerNA{} We report case studies assessing the safety of fixed models.
        \item Did you include the total amount of compute and the type of resources used (e.g., type of GPUs, internal cluster, or cloud provider)?
    \answerYes{} See Appendix~\ref{sec:expt_add}.
\end{enumerate}

\item If you are using existing assets (e.g., code, data, models) or curating/releasing new assets...
\begin{enumerate}
  \item If your work uses existing assets, did you cite the creators?
    \answerYes{}
  \item Did you mention the license of the assets?
    \answerNo{} The HMDA data comes from the US government and does not appear to have a license.
  \item Did you include any new assets either in the supplemental material or as a URL?
    \answerNo{}
  \item Did you discuss whether and how consent was obtained from people whose data you're using/curating?
    \answerNA{}
  \item Did you discuss whether the data you are using/curating contains personally identifiable information or offensive content?
    \answerYes{} None of the data contain personally identifiable information, and the US Consumer Finance Protection Bureau indicates that they have modified the HMDA data to protect applicant and borrower privacy.
\end{enumerate}

\item If you used crowdsourcing or conducted research with human subjects...
\begin{enumerate}
  \item Did you include the full text of instructions given to participants and screenshots, if applicable?
    \answerNA{}
  \item Did you describe any potential participant risks, with links to Institutional Review Board (IRB) approvals, if applicable?
    \answerNA{}
  \item Did you include the estimated hourly wage paid to participants and the total amount spent on participant compensation?
    \answerNA{}
\end{enumerate}

\end{enumerate}


\newpage
\appendix

\input{appendix_clean}

\end{document}

%% file: preamble.tex
\usepackage{mathtools}
\usepackage{amsthm}
\usepackage{verbatim}
\usepackage{subcaption}
\usepackage{enumitem}
\usepackage{algorithm}
\usepackage{algorithmic}

\newlist{inlinelist}{enumerate*}{1}
\setlist[inlinelist]{label=(\roman*)}

\DeclareMathOperator*{\argmax}{arg\,max}
\DeclareMathOperator*{\maxmin}{max/min}

\newcommand{\certset}{\mathcal{C}}

\theoremstyle{plain}
\newtheorem{theorem}{Theorem}
\newtheorem{proposition}[theorem]{Proposition}
\theoremstyle{definition}
\newtheorem{assumption}{Assumption}
\newtheorem{definition}{Definition}

%% file: appendix_clean.tex
\section{Additional Problem Formulation Details}
\label{sec:problem:add}

\paragraph{Deviation function $D$} {For the case where the inputs $y$, $y_0$ to $D$ are real-valued scalars (which covers binary classification and regression), while Assumption~\ref{ass:monotone} was stated as a sufficient condition for tractable optimization with GAMs, it is also an intuitively reasonable requirement: the deviation should increase the farther $y$ is from $y_0$ in either direction. In addition, symmetry may be desirable, i.e. $D(y, y_0) = D(y_0, y)$, to not favor one of the two models over the other. Both Assumptions~\ref{ass:monotone} and \ref{ass:difference} as well as symmetry are satisfied by monotonically increasing functions $D(|y - y_0|)$ of the absolute difference, for example powers $|y - y_0|^p$ for $p > 0$. 

For the case where $y, y_0 \in \mathbb{R}^M$ as in multi-class classification, it may be advantageous for $D(y, y_0)$ to decompose into a sum over output dimensions: $D(y, y_0) = \sum_{k=1}^M D_k(y_k, y_{0k})$, where $y_k$, $y_{0k}$ are the components of $y$, $y_0$. For example, the $p$th power of the $\ell_p$ distance $\lVert y - y_0 \rVert_p^p = \sum_{k=1}^M |y_k - y_{0k}|^p$ is separable in this manner.}

\paragraph{Models that abstain} The formulation in Section~\ref{sec:problem} can also accommodate models that abstain from predicting (and possibly defer to a human expert or other fallback system). If $f(x) = \emptyset$, representing abstention, then we may set $D(\emptyset, y_0) = d$ for any $y_0 \in \mathcal{Y}$, where $d > 0$ is an intermediate value less than the maximum value that $D$ can take \citep{bartlett2008classification}. The value $d$ might also be less than a ``typically bad'' value for $D$, to reward the model for abstaining when it is uncertain.

\section{Additional Details on Deviation Maximization for Specific Model Classes}

\subsection{Trees}
\label{sec:model:tree_add}

\paragraph{Rule lists} A rule list as defined by \citet{yang2017scalable,angelino2018learning} is a nested sequence of IF-THEN-ELSE statements, where the IF condition is a conjunctive rule and the THEN consequent is an output value. If each rule involves a single feature (i.e., the conjunctions are of degree $1$), such rule lists are one-sided trees in the sense 
of Section~\ref{sec:model:tree}. The number of leaves in the equivalent tree is equal to the number of rules in the list (including the last default rule).

\paragraph{Intersection of $\certset$ with a Cartesian product} If $\certset = \prod_{j=1}^d \certset_j$ is also a Cartesian product, then determining whether the intersection is non-empty amounts to checking whether all of the coordinate-wise intersections with $\certset_j$, $j=1,\dots,d$, are non-empty. If $\certset$ is not a Cartesian product but is a union of $\ell_\infty$ balls (which are Cartesian products), then the intersection is non-empty if the intersection with any one ball is non-empty. 

\paragraph{Relationship between Proposition~\ref{prop:tree} and \cite[Theorem 1]{chen2019robustness}} {In the case of a $K = 2$-tree ensemble, \cite[Theorem 1]{chen2019robustness} bounds the complexity of exact robustness verification as $\min\{O(n^2), O((4n)^d\} = O(n^2)$, where $n$ is the maximum number of leaves in a tree and we assume that the feature dimension $d \geq 2$. In Proposition~\ref{prop:tree}, we account for the possibly different numbers of leaves $L$ and $L_0$ in the two trees $f$ and $f_0$, and we exactly enumerate the edges, $|\mathcal{E}| \leq L_0 L \leq n^2$.}

\paragraph{Additive reference model} For the case where $f$ is a decision tree and $f_0$ is a generalized additive model, if the deviation function is symmetric, $D(y, y_0) = D(y_0, y)$, then this case is covered in Section~\ref{sec:model:linear_additive}.

\subsection{Linear and Additive Models}
\label{sec:model:linear_additive_add}

\paragraph{Categorical features} A function $f_j(x_j)$ of a categorical feature $x_j$ can be represented in two ways, depending on whether $f$ is a GAM or a GLM. In the GAM case, we may use the native representation in which $x_j$ takes values in a finite set $\mathcal{X}_j$ of categories. In the GLM case, $x_j$ is one-hot encoded into multiple binary-valued features $x_{jk}$, one for each category $k$. Then any function $f_j$ can be represented as a linear function, 
\[
f_j(x_j) = \sum_{k=1}^{\lvert\mathcal{X}_j\rvert} w_{jk} x_{jk},
\]
where $w_{jk}$ is the value of $f_j$ for category $k$. 

\paragraph{Implication of Assumption~\ref{ass:monotone}} The second condition implies that the deviation increases or stays the same as $y$ moves away from $y_0$ in either direction.

\begin{proof}[Proof of Proposition~\ref{prop:additive}]
Let $x \in \mathcal{S}$ and $S(x) = \sum_{j=1}^d f_j(x_j)$. Under Assumption~\ref{ass:monotone}.1, if $S(x) = g(y_0)$, then 
\[
D(f(x), f_0(x)) = D(g^{-1}(g(y_0)), y_0) = D(y_0, y_0) = 0. 
\]
As $S(x)$ increases from $g(y_0)$, $f(x)$ also increases because $g^{-1}$ is an increasing function, and $D(f(x), y_0)$ increases or stays the same due to Assumption~\ref{ass:monotone}.2. Similarly, as $S(x)$ decreases from $g(y_0)$, $f(x)$ decreases, and $D(f(x), y_0)$ again increases or stays the same. It follows that to maximize $D(f(x), y_0)$, it suffices to separately maximize and minimize $S(x)$, compute the resulting values of $D(f(x), y_0)$, and take the larger of the two. This yields the result.
\end{proof}

\paragraph{Implication of Assumption~\ref{ass:difference} and identity link function $g$} These two assumptions imply that the deviation is measured on the difference between $f$ and $f_0$ in the space in which they are additive. For example, if $f$ and $f_0$ are logistic regression models predicting the probability of belonging to one of the classes, the difference is taken in the log-odds (logit) domain. It is left to future work to determine other assumptions under which problem~\eqref{eqn:maxDev} is tractable when $f$ and $f_0$ are both additive.

\paragraph{Cartesian product $\certset$ implies Cartesian product $\mathcal{S}$} In the cases of constant and additive $f_0$, $\mathcal{S} = \certset$. In the decision tree case, since each leaf is a Cartesian product $\mathcal{L}_{0m} = \prod_{j=1}^d \mathcal{R}_{mj}$, the intersections $\mathcal{S} = \mathcal{L}_{0m} \cap \certset$ are also Cartesian products $\prod_{j=1}^d \mathcal{S}_{j}$ where $\mathcal{S}_{j} = \mathcal{R}_{mj} \cap \certset_j$. 

\paragraph{One-dimensional optimization complexities $C_j$} For discrete-valued $x_j$, $C_j$ is proportional to the number of allowed values $\lvert \mathcal{S}_j \rvert$. For continuous $x_j$, it is common to use spline functions or tree ensembles as $f_j$ in constructing GAMs. In the former case, $C_j$ is proportional to the number of knots. In the latter, the tree ensemble can be converted to a piecewise constant function and $C_j$ is then proportional to the number of pieces. Lastly in the case where $f_j(x_j) = w_j x_j$ is linear and $\mathcal{S}_j = [\underline{X}_j, \overline{X}_j]$ is an interval, $C_j = O(1)$ because it suffices to evaluate the two endpoints.

\paragraph{Convex $\mathcal{S}$} If $\certset$ is a convex set, then in the cases of constant and additive $f_0$, $\mathcal{S} = \certset$ is also convex. In the case of tree-structured $f_0$, $\mathcal{S} = \mathcal{L}_{0m} \cap \certset$ and each leaf $\mathcal{L}_{0m}$ can be represented as a convex set, with interval constraints on continuous features and set membership constraints on categorical features. The latter can be represented as $x_{jk} = 0$ constraints on the one-hot encoding (see ``Categorical features'' paragraph above) for non-allowed categories $k$. Hence $\mathcal{S}$ is also convex. 

As a specific example, suppose that $\mathcal{S}$ is the product of independent constraints on each categorical feature and an $\ell_p$ norm constraint on the continuous features jointly. The maximization over each categorical feature has complexity $C_j = \lvert \mathcal{S}_j \rvert$ as noted above, while the maximization of $w^T x$ over continuous features lying in an $\ell_p$ ball has closed-form solutions for the common cases $p = 1, 2, \infty$. 

\paragraph{Relaxations of the Cartesian product assumption} {If the certification set $\certset$ is not a Cartesian product, then one way to still bound the maximum deviation is to find the smallest Cartesian product $\overline{\certset}$ that contains $\certset$ and maximize deviation over $\overline{\certset}$. As long as it is relatively easy to optimize linear functions over $\certset$, then constructing such a Cartesian product is similarly easy. Another conceivable relaxation of the Cartesian product assumption is a Cartesian product of low-dimensional sets, not just one-dimensional.}

\subsection{Tree Ensembles}
\label{app:ensembles:algo}

The full algorithm for clique search from Section \ref{sec:model:ensemble} is presented in Algorithm \ref{alg:tree-ensembles:enum}. It uses $Z$ as a node compatibility vector to keep track of valid leaves and $B$ a set of trees/partites not yet covered by the maximum clique. The algorithm starts with and empty clique $S$ and anytime bounds as $0$. It starts the search with the smallest tree to limit the search space. This is typically $f_0$. Each leaf is added to the intermediate clique $S$ in turn (Line 6). A stronger primal bound can be achieved if the traversal is ordered in a meaningful way. In particular, starting with nodes with the highest heuristic function value $H(S)$ aids the algorithm to focus on better areas of the search space.

If the size of the clique is $K+1$ the primal bound is updated. Otherwise, the dual bound is computed. If the node is promising, the algorithm recurses to the next level. When the search is terminated at any step, the maximum deviation is bounded by $(D_{lb}, D_{ub})$.

\begin{algorithm}
\caption{Max clique search for maximum deviation}\label{alg:tree-ensembles:enum}
\begin{algorithmic}[1]
\REQUIRE $M$ adjacency matrix, $H$ heuristic function
\STATE $Z[i]=1\forall i\in V$, $B=\{1,2,\hdots,K+1\}$, $S=\emptyset$ \COMMENT{All nodes valid, all trees uncovered}
\STATE Q = Enumerate($Z, B, S$)
\STATE \textbf{Initialize:}$D_{lb} = 0$, $D_{ub} = 0$ \COMMENT{Anytime bounds}
\STATE \textbf{function} Enumerate($Z, B, S$):
\STATE $t = \argmax_b \{|Z_b| \bigm| b\in B\}$ \COMMENT{Uncovered tree with fewest valid nodes}
\FOR{$i$ in $Z_t$}
\STATE $Z[i] = 0$ \COMMENT{Mark node as incompatible}
\STATE $S = S\cup \{i\}$ \COMMENT{Add to candidate clique}
\IF{$|S|=K+1$}  
    \STATE $D_{lb} = \max \left(D_{lb}, D(S)\right)$ \COMMENT{Update primal bound}
    \STATE $Q = Q \cup S$ \COMMENT{Add to set of max cliques}
    \STATE $S = S\setminus \{i\}$ \COMMENT{Backtrack}
\ELSE
    \STATE $Z_{t+1} = Z_t \land M(i)$ \COMMENT{Update valid nodes}
    \STATE $B = B\setminus \{t\}$ \COMMENT{Update uncovered trees}
    \STATE $D_{ub} = \max \left(D_{ub}, H(S)\right)$ \COMMENT{Update dual bound}
    \IF{$D_{ub} > D_{lb}$}  
        \STATE Enumerate($Z_{t+1}, B, S$) \COMMENT{Recurse to next level}
    \ENDIF
    \STATE $S = S \setminus \{i\}$ \COMMENT{Backtrack}
    \STATE $B = B \cup \{t\}$
\ENDIF
\ENDFOR
\end{algorithmic}
\end{algorithm}

\paragraph{Rule ensembles} Similar to the tree ensembles considered in Section~\ref{sec:model:ensemble}, a rule ensemble is a linear combination of conjunctive rules, where the antecedent is a conjunction of conditions on individual features, and the consequent takes a real value if the antecedent is true and zero otherwise. They are produced by algorithms such as SLIPPER \citep{cohen1999}, that of \citet{rueckert2006}, RuleFit \citep{friedman2008}, ENDER \citep{dembczynski2010} and have also been referred to as generalized linear rule models \citep{wei2019generalized}. A rule ensemble can be converted into a tree ensemble by converting each conjunctive rule into an IF-THEN-ELSE rule list, which is a one-sided tree (see Appendix~\ref{sec:model:tree_add}). Specifically, the conditions in the conjunction are taken in any order, each condition is negated to become an IF condition, and the THEN consequents are all output values of zero. The final ELSE consequent, which is reached if all the IF conditions are false (and hence the original rule holds), returns the output value of the original rule. The number of leaves in the resulting tree equals the number of conditions in the conjunction plus one.

\subsection{Piecewise Lipschitz Functions}

\begin{proof}[Proof of Proposition \ref{prop:lip}]
Consider two inputs $x$ and $y$ then,
\begin{align*}
    |h(x)-h(y)| &= |(h_0-h_1)(x)-(h_0-h_1)(y)|=|h_0(x)-h_0(y)+h_1(y)-h_1(x)|\\
    &\le |h_0(x)-h_0(y)| + |h_1(x)-h_1(y)|\le c_0\cdot\ell(x,y)^{\beta_0}+c_1\cdot\ell(x,y)^{\beta_1}\\
    &\le c\cdot\ell(x,y)^{\beta}
\end{align*}
where, $c=2\max(c_0,c_1)$ and $\beta=\min(\beta_0,\beta_1)$.
\end{proof}

\noindent\textbf{Other choices for $D(.,.)$:} The results assumed $D(.,.)$ to be the identity function, where $\Delta=D(f_0,f)$. This choice of function clearly satisfies Assumptions~\ref{ass:monotone} and \ref{ass:difference}. Again consistent with these assumptions we look at some other choices for $D(.,.)$. If $D(.,.)$ were an affine function with a positive scaling such as $D(y_0,y)=\alpha (y_0-y)+b$ where $\alpha > 0$, then our result in equation \ref{eq:bint} would be unchanged as only the Lipschitz constant of $\Delta$ would change, but not its (underestimated) order. If the function were a polynomial or exponential however, no such guarantees can be made and we would be back to case 1.

\subsection{Other Model Classes: Neural Networks and Post Hoc Explanations}
\label{sec:model:post_hoc}

For model classes beyond the ones discussed in Section~\ref{sec:model}, it appears to be a greater challenge to obtain 
reasonably tractable algorithms that guarantee exact computation of or bounds on the maximum deviation. Here we outline some future directions for neural networks and post hoc explanations.

Robustness verification for neural networks has attracted a great deal of attention and made considerable progress, with exact approaches including satisfiability modulo theory \cite{katz2019marabou} and mixed integer programming \cite{tjeng2019evaluating,anderson2020strong}, and incomplete methods that compute bounds using bound propagation \cite{singh2018fast,weng2018towards}, linear programming and duality \cite{ehlers2017formal,wong2018provable,dvijotham2018dual}, and semidefinite programming \cite{raghunathan2018semidefinite,dathathri2020enabling}. However, all of these methods consider a single model, effectively comparing it to a constant. {Robustness verification is thus essentially a single-model case of our problem \eqref{eqn:maxDev} in which $f_0$ is a constant (and with an appropriate choice of the deviation function $D(f, f_0)$). While we may expect that solutions to a two-model verification problem would leverage existing robustness verification methods, developing such solutions remains for future work.} Furthermore, evaluation of robustness verification methods has largely been limited to local neighborhoods around input points (with typical radii $\epsilon \leq 0.1$ in terms of normalized feature values). 
This limitation may also need to be addressed to enable evaluation of maximum deviation in the way envisioned in this paper.

It is also natural to ask whether post hoc explanations for the model can help. One way in which this could occur is if the post hoc explanation approximates the model $f$ by a simpler model $\hat{f}$ and if the deviation function $D$ satisfies the triangle inequality $D(f(x), f_0(x)) \leq D(f(x), \hat{f}(x)) + D(\hat{f}(x), f_0(x))$. Then the maximum deviation in \eqref{eqn:maxDev} would be bounded as 
\begin{equation}\label{eqn:triangleIneq}
    \max_{x\in\certset} D(f(x), f_0(x)) \leq \max_{x\in\certset} D(f(x), \hat{f}(x)) + \max_{x\in\certset} D(\hat{f}(x), f_0(x)).
\end{equation}
While we may choose $\hat{f}$ to be interpretable so that the rightmost maximization is tractable, the middle maximization asks for a \emph{uniform} bound on the deviation between $f$ and $\hat{f}$, i.e, the fidelity of $\hat{f}$. We are not aware of a post hoc explanation method that provides such a guarantee. Indeed, in general, the middle maximization might not be any easier than the left-hand one that we set out to bound.

A (practical) possibility may be to perform quantile regression \citep{qr} for a large enough quantile to learn $\hat{f}$, as opposed to minimizing expected error as is typically done. This may be an interesting direction to explore in the future as quantile regression algorithms are available for varied model classes including linear models, tree ensembles \citep{qr-rf} and even neural networks \citep{qr-nn}. More investigation is needed into whether quantile regression methods can provide approximate guarantees on the middle term in \eqref{eqn:triangleIneq}.

{Assuming that uniform proxies in the above sense can be constructed, then for certain modalities or applications it may be possible to train highly accurate proxies. For instance for tabular data, Random Forests or boosted trees might very well replicate the behavior of a neural network, in which case the machinery introduced in Section~\ref{sec:model:ensemble} could be used. Even for other modalities such as text and images, interpretable models such as Neural Additive Models (NAMs) \cite{agarwal2021neural} and continued fraction networks (CoFrNets) \cite{puri2021cofrnet} may prove to be sufficient in some cases.

Finally, there are recent architectures such as Lipschitz neural networks \cite{meunier2022dynamical} which are adversarially robust and hence valuable in practice. Our analysis presented in Section 4.4 for piecewise Lipschitz models would be applicable here, where the simple regret of standard bandit algorithms for a given number of queries could be reduced to \eqref{eq:bint} as opposed to \eqref{eq:bbb}.
}

\section{Further Discussion}
\label{sec:discuss}

{
\paragraph{Worst-case approach} The formulation of \eqref{eqn:maxDev} as the worst case over a certification set represents a deliberate choice to depend as little as possible on a probability distribution or a dataset sampled from one. As stated in Section~\ref{sec:problem}, Certification Set paragraph, $\mathcal{C}$ can depend at most on (an expanded version of) the support set of a distribution. The reason for this choice is because safety is an out-of-distribution notion: harmful outputs often arise precisely because they were not anticipated in the data. The trade-off inherent in this choice is that the maximum deviation may be more conservative than needed. The high maximum deviation values in e.g.~Figure~\ref{fig:Adult_all_add} may reflect this. Given definition \eqref{eqn:maxDev} as a starting point in this paper, future work could consider variations that depend more on a distribution and are thus less conservative, but may also offer a weaker safety guarantee.
}

\paragraph{Choice of reference model} The proposed definition of maximum deviation~\eqref{eqn:maxDev} depends on the choice of reference model $f_0$. Different choices will lead to different deviation values and, perhaps more importantly, different combinations of features that maximize the deviation. We have discussed possible choices in Section~\ref{sec:problem}, and the results in Section~\ref{sec:model} indicate that, as with the assessed model $f$, interpretable forms for $f_0$ can ease the computation of maximum deviation. Beyond these guidelines, it is up to ML practitioners and domain experts to decide on appropriate reference models for their application (and there may be benefit to considering more than one). We mention an additional concern with the reference model in the Ethics Discussion.

For some real applications it may be difficult to come up with a globally interpretable reference model. But specific to particular scenarios it may be possible. For instance, it might be difficult to provide general rules for how to drive a car, but in specific scenarios such as there being an obstacle in front, one can suggest that you stop or turn, which is a simple rule. So our machinery could potentially be applied at a local level where the reference model is interpretable in that locality. This might help in ``spot checking" a deployed model and estimating its safety by computing these maximum deviations in scrupulously selected (challenging) scenarios.

\paragraph{Impossible inputs in certification set} Mathematically simple sets such as Cartesian products and $\ell_p$ balls permit simpler algorithms for optimizing functions over them. Accordingly, these sets have been the focus of not only the present work but also the related literature on ML verification and adversarial robustness. However, they may not serve to exclude inputs that are physically or logically impossible from the certification set $\certset$, and thus, the resulting maximum deviation values may be too large and conservative. Here it is important to distinguish between impossible inputs and those that are merely implausible (i.e., with low probability). Techniques for capturing implausibility have been proposed for contrastive/counterfactual explanations \citep{CEM,CEM-MAF}, whereas we expect the set of impossible inputs to be smaller and more constrained. As a simple example from the Adult Income dataset, if we agree that a wife/husband is defined to be of female/male gender (regardless of the gender of the spouse), then the cross combinations male-wife and female-husband cannot occur. Future work can consider the representation and handling of such constraints.

\paragraph{White-box vs.~grey-box models} In this paper, we have assumed full ``white-box'' access to both $f$ and $f_0$, namely complete knowledge of their structure and parameters. Interesting questions may arise when this assumption is relaxed to different ``grey-box'' possibilities. For example, 
one could further investigate the third case in Section~\ref{sec:model:piecewise_Lipschitz}, where one of $f, f_0$ is black-box and the other is white-box interpretable. There may exist assumptions that we have not identified that would improve the query complexity compared to generic black-box optimization.

\paragraph{Other interpretability-safety relationships} This paper has focused on one relationship between the interpretability of a model and the safety of its outputs. It has not addressed other ways in which interpretability/explainability can affect the \emph{risk} of a model (in the plain English sense, not the expectation of a loss function). For example, in regulated industries such as consumer finance, not providing explanations or providing inadequate ones can lead to legal, financial, and reputational risks. On the other hand, providing explanations is associated with its own risks \citep{Weller2019}. These include the leakage of personal information or model information (intellectual property), an increase in appeals of decisions for the decision-making entity, and strategic manipulation of attributes (i.e.~``gaming'') by individuals to gain more favorable outcomes.

\paragraph{Applicability to RL settings} In RL, if one views the actions as labels and state representation as features, one can build a tree, albeit likely a deep/wide one, to represent exactly the RL policy, where the probability distribution over the actions can be viewed as the class distribution in a normal supervised setting. Rolling up the states, creating leaves with multiple states, and simply averaging the probabilities for each action would yield smaller trees that approximate the policy. Our work lays a foundation where in principle we can also compare $f$ and $f_0$ that are policies using such tree representations. This may be related to a popular global explainability method \citep{viper} that samples policies and builds trees to explain them.

\paragraph{Ethics} The safety of machine learning systems has been called out by the European Commission's regulatory framework \citep{altai2020}. The commission states seven key dimensions to be evaluated and audited by a cross-disciplinary team:
\begin{inlinelist}
    \item human agency and oversight,
    \item technical robustness and safety,
    \item privacy and data governance,
    \item transparency,
    \item diversity, non-discrimination and fairness,
    \item environmental and societal well-being, and
    \item accountability.
\end{inlinelist}
The second of these dimensions is safety. However, \citet{sloane2021silicon} argue that algorithmic audits are ill-defined as the underlying definitions are vague. The proposed work helps fill that ill-definedness using a quantitative approach. One may argue against this particular choice of quantification, but it does start the community down the path toward being more concrete in its definitions.

As with many other technologies, the proposed approach may be misused. For example, the reference model may be chosen in a  way that hides the safety concerns of the model being evaluated. Transparent documentation and reporting with provenance guarantees can help avoid this kind of purposeful deceit \citep{arnold2019factsheets}.

\input{experiments_add_clean}

%% file: experiments_add_clean.tex
\clearpage
\section{Experiment Details and Additional Results}
\label{sec:expt_add}

\subsection{General Experiment Details}

\paragraph{Data processing} We use the training set of each dataset to train models and the test set as the basis for evaluating maximum deviation, specifically as the set of centers for the $\ell_\infty$ balls in \eqref{eqn:unionBalls}. Continuous features are standardized and categorical features are one-hot encoded. The $\ell_\infty$ norm is computed on the resulting normalized feature values.

\paragraph{Models} We use scikit-learn \citep{scikit-learn} to train decision tree (DT), logistic regression (LR), and Random Forest (RF) models. The corresponding complexity parameters are the number of leaves for DT (parameter \texttt{max\_leaf\_nodes}), the amount of $\ell_1$ regularization for LR (inverse $\ell_1$ penalty $C$), and the number of estimators/trees for RF (\texttt{n\_estimators}). For additive models, we use Explainable Boosting Machines (EBM) from the InterpretML package \citep{nori2019interpretml} with zero interaction terms (so that the models are indeed additive). 
Smoothness is controlled by the \texttt{max\_bins} parameter, the number of discretization bins for continuous features.

\paragraph{Deviation maximization} In all cases, when the certification set radius $r = 0$, maximum deviation can be computed simply by evaluating the models on the test set. For the case where $r > 0$, $f$ is a DT or RF, and $f_0$ is a DT, Algorithm \ref{alg:tree-ensembles:enum} is used (on a bipartite graph if $f$ is a DT). 
The cases where $f$ is LR or EBM fall under the generalized additive case of Section~\ref{sec:model:linear_additive}. Given that $f_0$ is a DT, we use \eqref{eqn:maxDevAdditiveTree}, \eqref{eqn:maxAdditive} to determine the maximum deviation. For $r < \infty$ when $\certset$ is a union of $\ell_\infty$ balls, we maximize separately over each intersection between a ball and a leaf of $f_0$ and then take the maximum over the intersections. 

\paragraph{Computation} All experiments were run on CPU nodes with 64GB memory. For decision trees and tree ensembles run times of Algorithm \ref{alg:tree-ensembles:enum} were limited to $2$ hours, after which the best available bounds were used.

\subsection{Home Mortgage Disclosure Act Dataset}
\label{sec:expt_add:HMDA}

\paragraph{Data source and pre-processing} The data is made available by the US Consumer Finance Protection Bureau (CFPB) under the Home Mortgage Disclosure Act (HMDA). We use the national snapshot from year 2018\footnote{\url{https://ffiec.cfpb.gov/data-publication/snapshot-national-loan-level-dataset/2018}} of ``loan/application records,'' which contain information on mortgage applications and their outcomes. According to their website, the CFPB has modified the data to protect applicant and borrower privacy.

We processed the 2018 loan/application records as summarized below. These steps were informed by \citet{gill2020hmda} but not identical to theirs:
\begin{itemize}
    \item Restrict to complete, submitted applications with `action\_taken' $\leq 3$ (loan originated, application approved but not accepted by applicant, or application denied).
    \item Create a binary-valued target variable representing approval by binarizing `action\_taken' (originated or approved $\to 1$, denied $\to 0$).
    \item Restrict to purchases of principal residences (the most consequential in terms of people's lives, i.e., not refinances or for investment) and single-family homes.
    \item Restrict to loans that are not ``special'' in any way: conventional loans, first mortgages, not manufactured homes, no non-amortizing features, etc.
    \item Drop columns that are not applicable for site-built single-family homes.
    \item Drop columns that are not applicable or recorded until the approval/origination decision is made. For example, loan costs (points and fees) are not recorded until a loan is originated, the type of entity that purchases a loan does not apply unless the loan is originated, etc. This is in keeping with the mortgage approval scenario that we consider.
    \item Drop all demographic columns to reflect laws that forbid lending decisions from explicitly depending on applicant demographics.
    \item Drop geographical columns that have too many unique values (e.g., census tract).
    \item Drop rows that have null or ``exempt'' values in key features such as loan-to-value ratio, debt-to-income ratio, property value, income.
    \item Take a 10\% sample of the records remaining after the above processing, to make experimentation less time- and resource-consuming.
    \item Split the subsampled dataset 80\%--20\% into training and test sets.
\end{itemize}

The dataset resulting from the above processing is imbalanced, with nearly 93\% of mortgage applications approved. Thus in training models, we balance the classes by weighting, either using the \texttt{class\_weight=`balanced'} option in scikit-learn or defining sample weights for the same purpose. On the test set, we evaluate balanced accuracy instead of accuracy.

\begin{figure}[t]
    \centering
    \includegraphics[width=\textwidth]{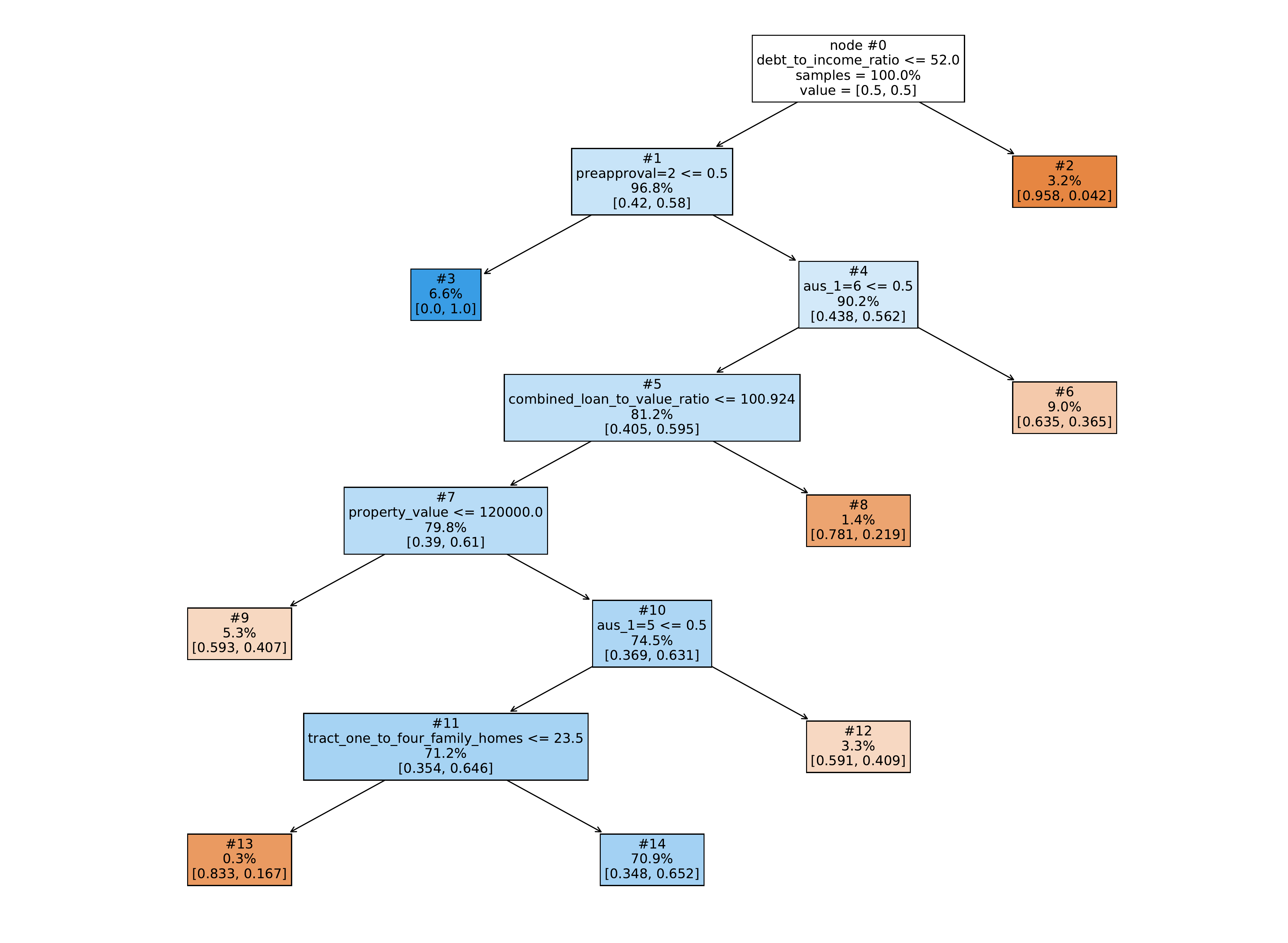}
    \caption{Decision tree reference model with 8 leaves for the HMDA dataset.}
    \label{fig:HMDA_DT}
\end{figure}

\paragraph{Reference model} Figure~\ref{fig:HMDA_DT} depicts the $8$-leaf DT reference model used in the experiments on the HMDA dataset. This DT has a test set balanced accuracy of $70.9\%$ and area under the receiver operating characteristic (AUC) of $0.750$. The top-level split is based on debt-to-income ratio, a measure often used in lending. Other common mortgage measures such as loan-to-value ratio, property value, and whether a preapproval was requested (value 2 means no) also appear. {`aus\_1'=5} and {`aus\_1'=6} refer to the automated underwriting system used to evaluate the application, with values 5 and 6 denoting ``other'' and ``not applicable''.

\begin{table}[ht]
    \small
    \centering
    \begin{tabular}{rrrrr}
    \toprule
    $C$&nonzeros&$\ell_1$ norm&bal. acc.&AUC\\
    \midrule
    1e-4&2&0.4&0.605&0.663\\
    3e-4&7&1.3&0.633&0.695\\
    1e-3&17&4.6&0.663&0.736\\
    3e-3&26&7.7&0.669&0.744\\
    1e-2&42&12.7&0.674&0.750\\
    3e-2&68&16.6&0.675&0.751\\
    1e-1&85&20.3&0.675&0.752\\
    3e-1&96&22.2&0.676&0.752\\
    1e+0&99&22.9&0.675&0.752\\
    3e+0&100&23.2&0.675&0.752\\
    \bottomrule
    \end{tabular}
    \caption{Number of nonzero coefficients, $\ell_1$ norm of coefficients, test set balanced accuracy, and area under the receiver operating characteristic (AUC) for logistic regression models on the HMDA dataset as a function of inverse $\ell_1$ penalty $C$.}
    \label{tab:HMDA_LR_stats}
\end{table}

\begin{table}[ht]
    \small
    \centering
    \begin{tabular}{rrr}
    \toprule
    max\_bins&bal. acc.&AUC\\
    \midrule
    4&0.669&0.743\\
    8&0.695&0.772\\
    16&0.711&0.785\\
    32&0.720&0.798\\
    64&0.723&0.799\\
    128&0.722&0.800\\
    256&0.722&0.800\\
    512&0.723&0.800\\
    1024&0.723&0.800\\
    \bottomrule
    \end{tabular}
    \caption{Test set balanced accuracy and AUC for Explainable Boosting Machines on the HMDA dataset as a function of \texttt{max\_bins} parameter.}
    \label{tab:HMDA_GAM_stats}
\end{table}

\paragraph{LR and GAM models} In Tables~\ref{tab:HMDA_LR_stats} and \ref{tab:HMDA_GAM_stats}, we show the values of inverse $\ell_1$ penalty $C$ and \texttt{max\_bins} that were used for LR and GAM respectively, as well as statistics of the resulting classifiers. Test set balanced accuracy and AUC increase and reach a plateau. For LR, we take the $\ell_1$ norm of the coefficients to be the main measure of smoothness as it depends on both the number of nonzero coefficients as well as their magnitudes, which both affect the extreme values attained in \eqref{eqn:maxAdditive}. Since the LR balanced accuracy and AUC are not higher than those of the reference DT, we focus less on LR in what follows. 

For EBM, based on Table~\ref{tab:HMDA_GAM_stats}, we select \texttt{max\_bins} $= 32$ as a representative model with balanced accuracy and AUC nearly equal to the maximum attainable values. 
Plots for this EBM model are shown in Figure \ref{fig:HMDA_GAM_functions}. First note that whether a preapproval was requested (value 1 means yes) is quite predictive of final approval. The shape functions for the four continuous features debt-to-income ratio, loan-to-value ratio, property value, and income are mostly monotonic and agree with domain knowledge. The log-odds of mortgage approval decrease as debt-to-income ratio and loan-to-value ratio increase, with abrupt drops around $50\%$ for debt-to-income ratio and at $80\%$ and $100\%$ for loan-to-value ratio. For property value and income, after a minimum value is reached, the shape function increases rapidly and then stays more or less constant for high property values and incomes.

\begin{figure}[t]
    \centering
    \includegraphics[width=0.495\textwidth]{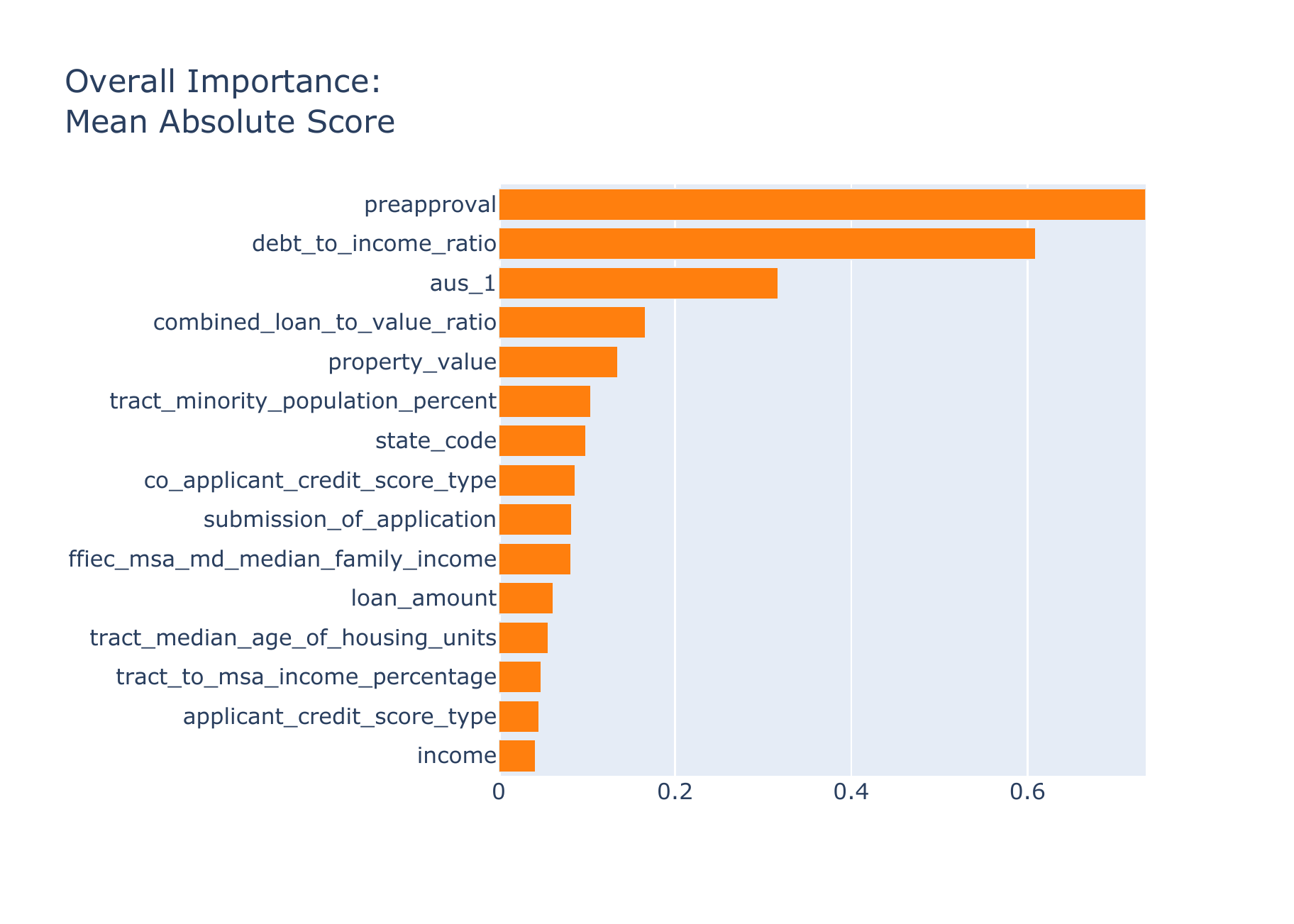}
    \includegraphics[width=0.495\textwidth]{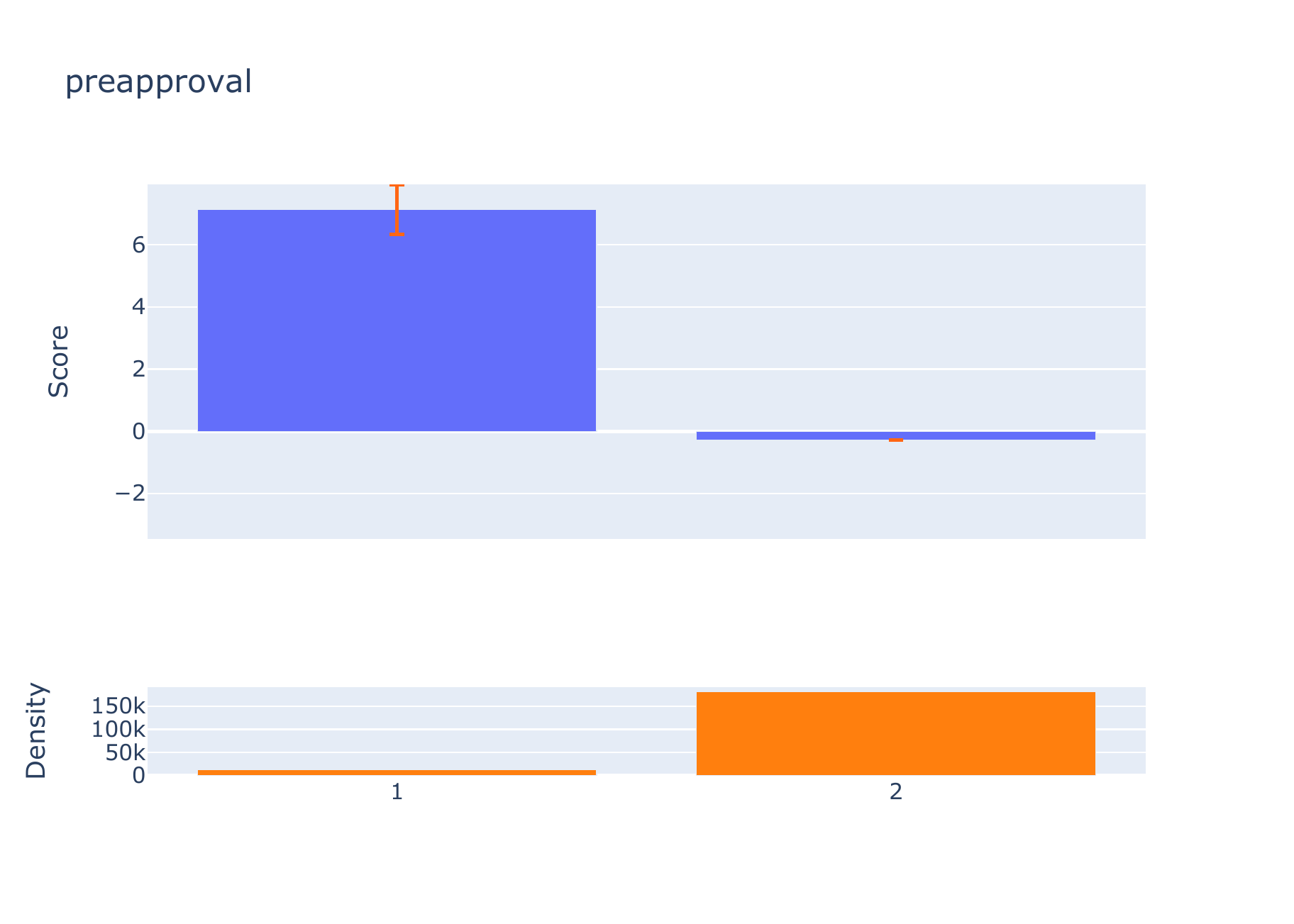}
    \includegraphics[width=0.495\textwidth]{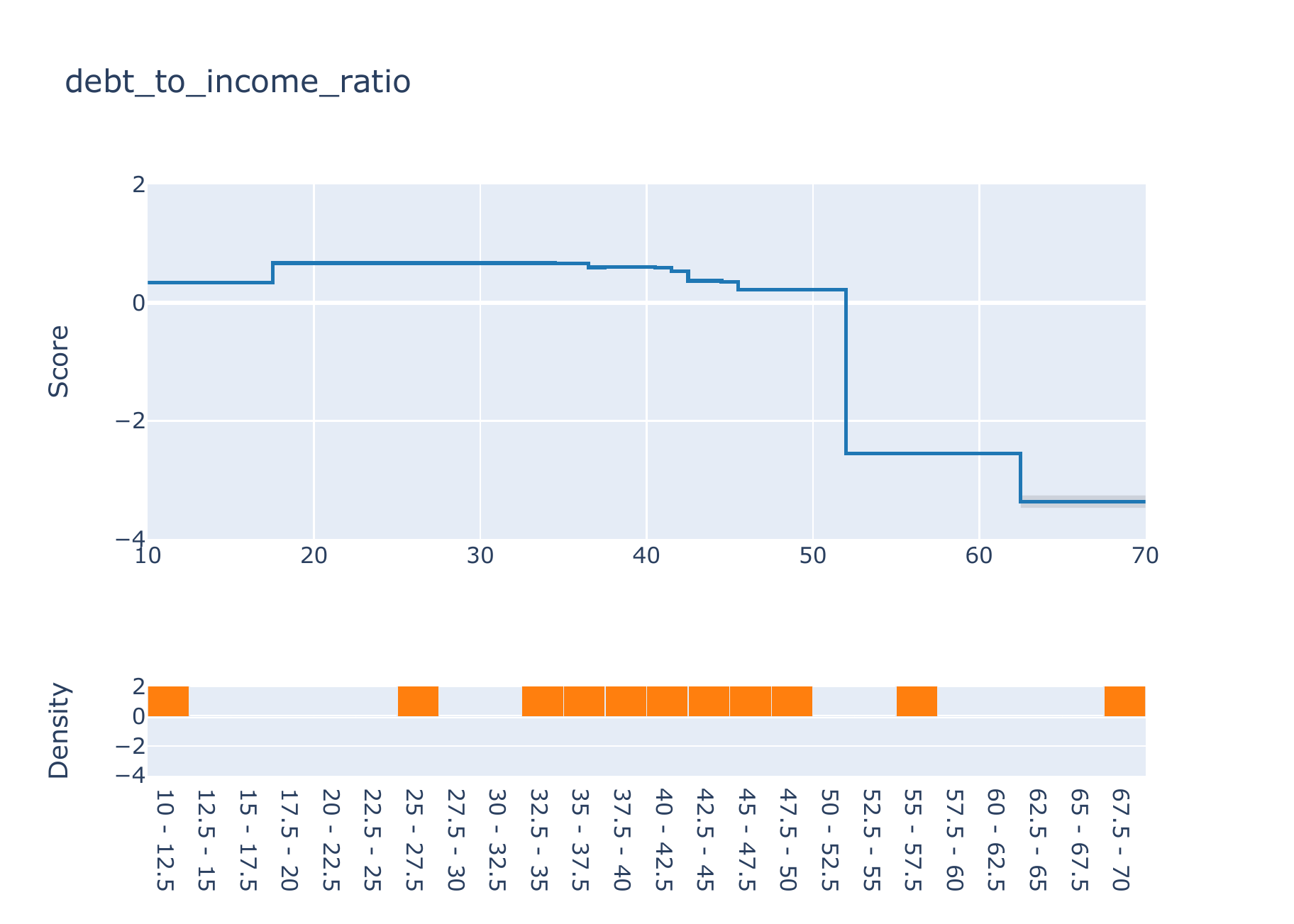}
    \includegraphics[width=0.495\textwidth]{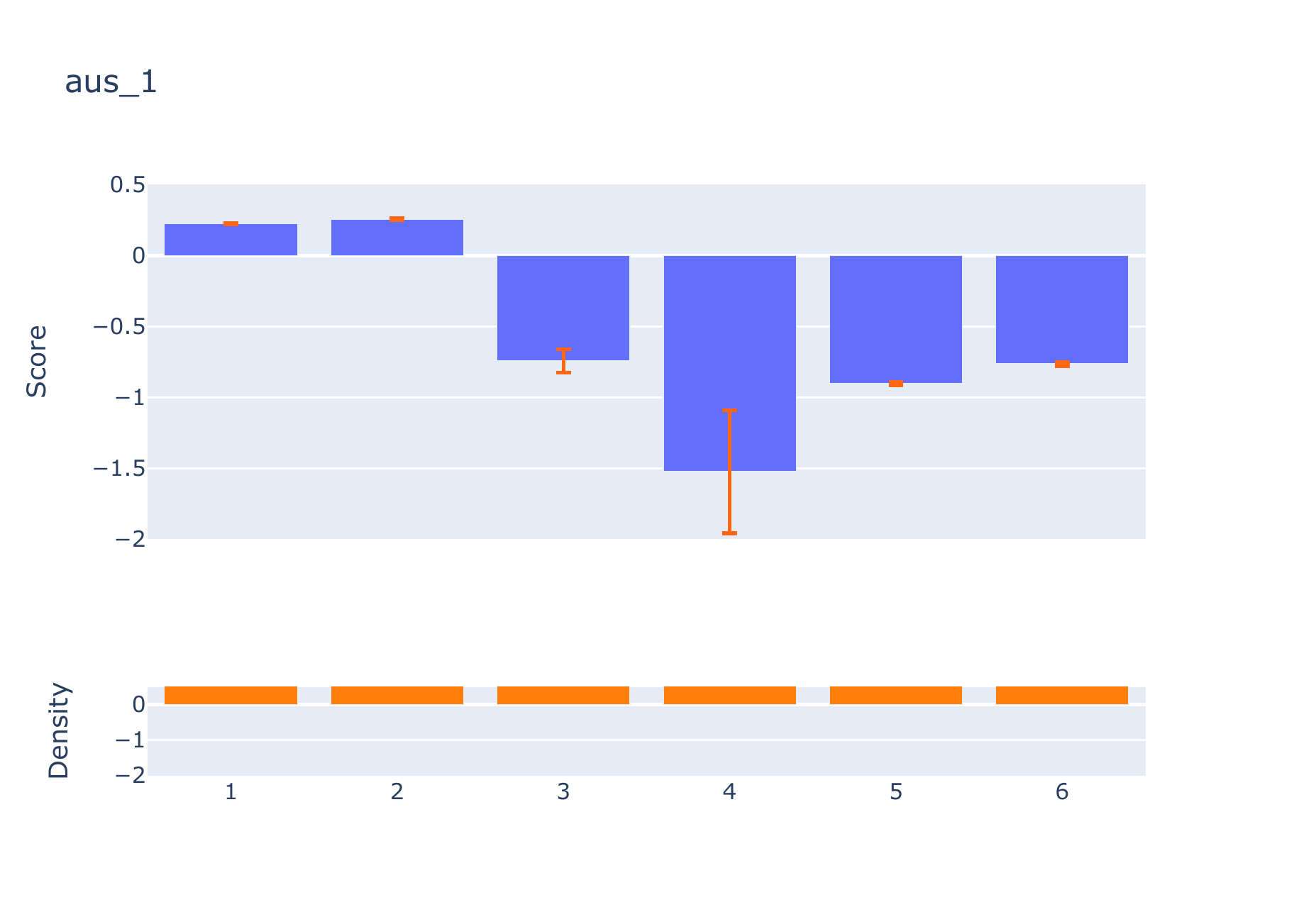}
    \includegraphics[width=0.495\textwidth]{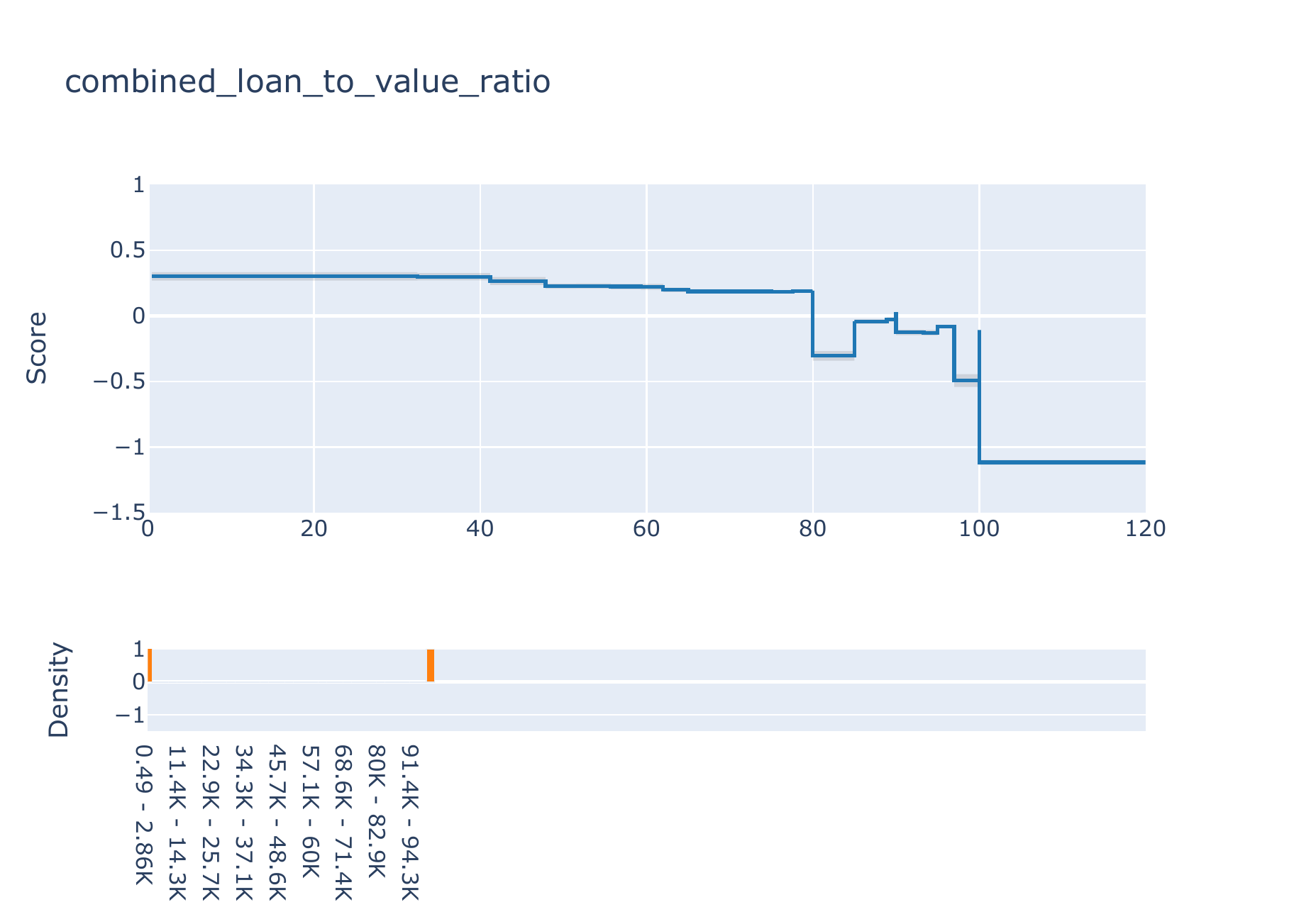}
    \includegraphics[width=0.495\textwidth]{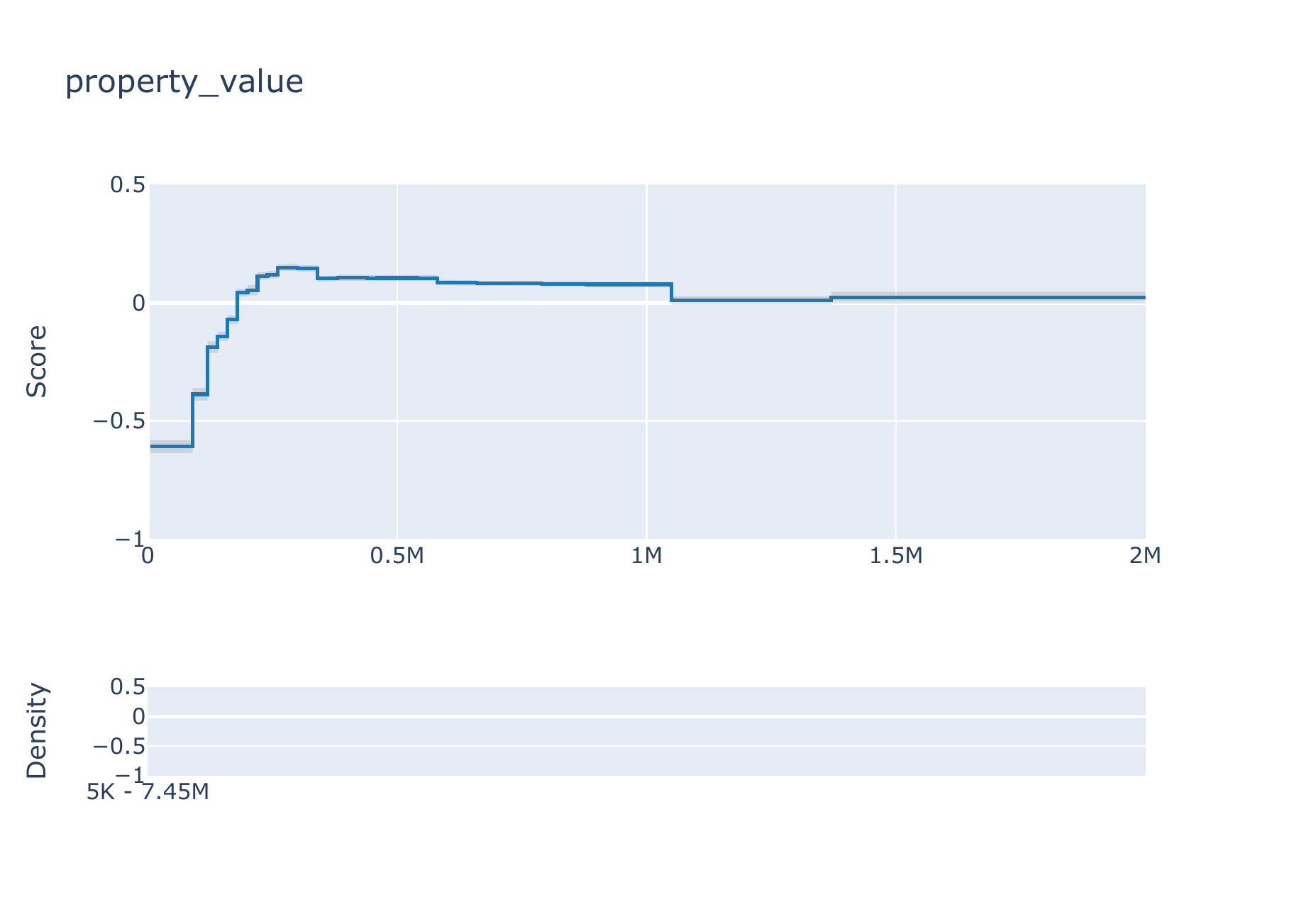}
    \includegraphics[width=0.495\textwidth]{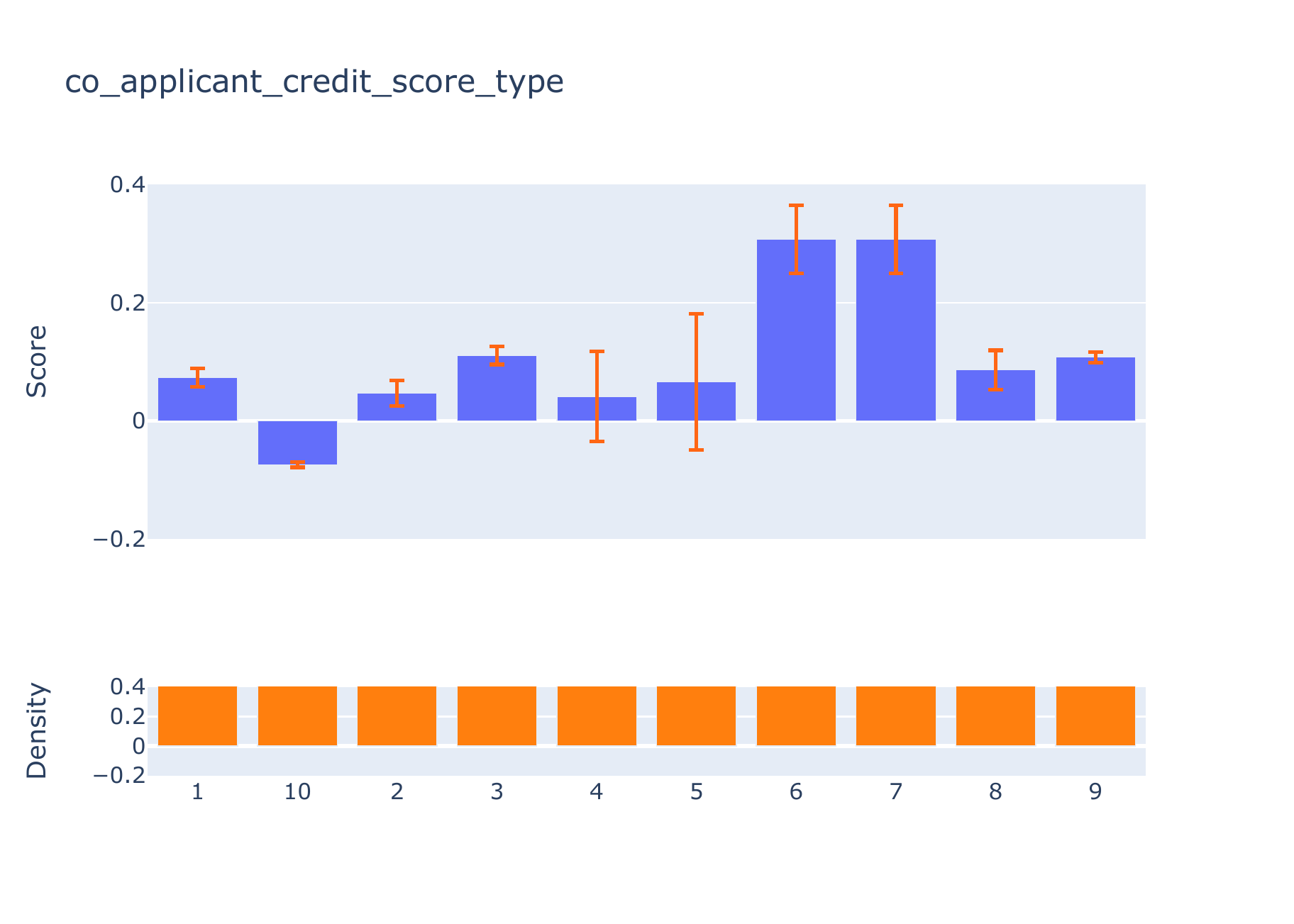}
    \includegraphics[width=0.495\textwidth]{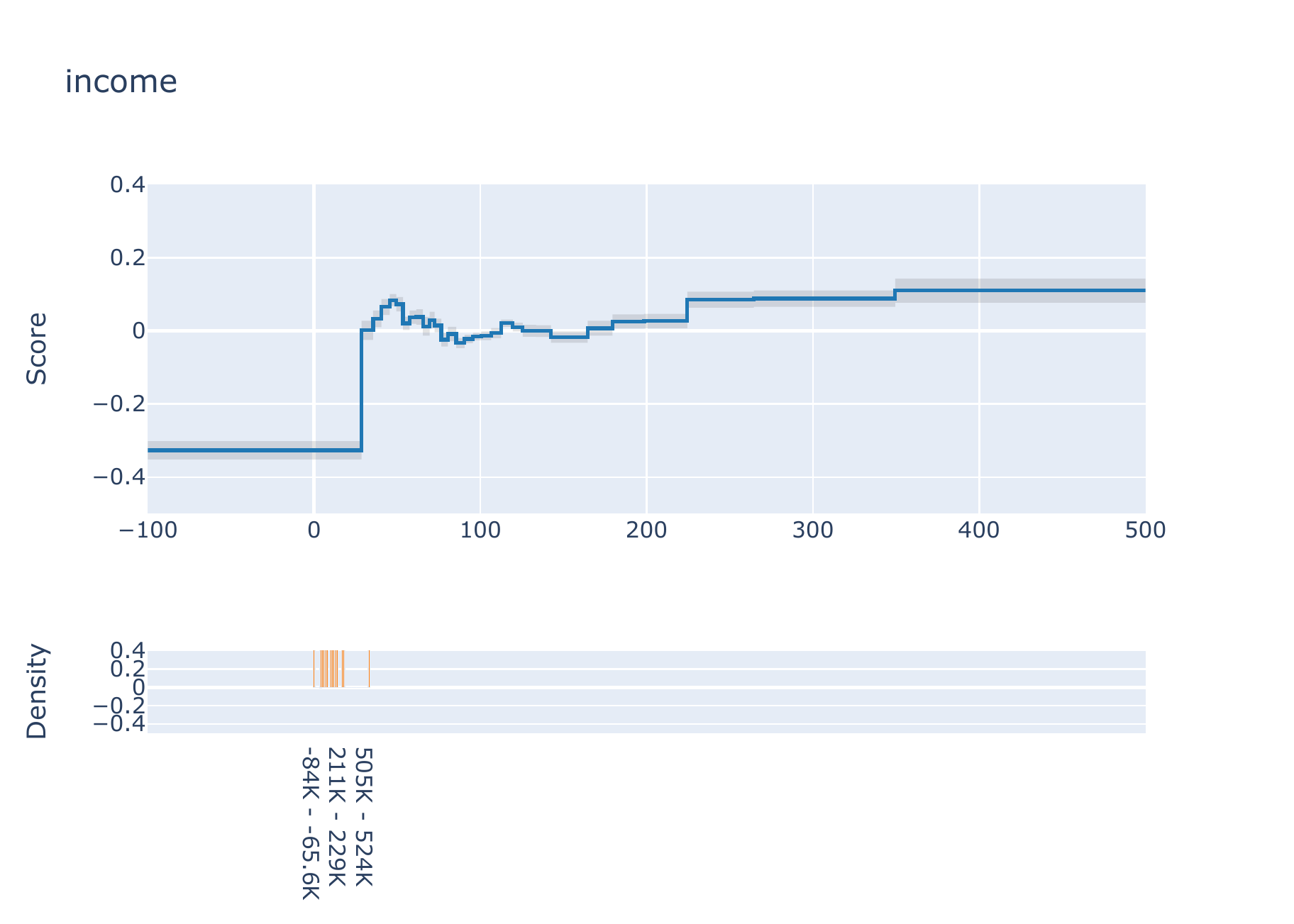}
    \caption{Feature importances and selected univariate functions $f_j$ for the Explainable Boosting Machine with \texttt{max\_bins} $= 32$ on the HMDA dataset.}
    \label{fig:HMDA_GAM_functions}
\end{figure}

\clearpage

\begin{table*}[t]
    \small
    \centering
    \begin{tabular}{lrrrrrr}
    \toprule
    $r$&preapproval&appl\_credit&loan\_to\_value&co\_appl\_credit&intro\_rate\_period&debt\_to\_income\\
    &&\_score\_type&(\%)&\_score\_type&&(\%)\\
    \midrule
    0.0&1&2&74.35&10&122.0&55.0\\
    0.1&1&3&[80. 80.]&1&[28.4 43.6]&[53.9 56.1]\\
    0.2&1&1&[29.33 32.42]&1&[57.  60.5]&[52.8 57.2]\\
    0.4&1&7&[ 4.34 32.42]&7&[329.7 360. ]&[52.  53.4]\\
    0.6&1&7&[ 0.49 32.42]&7&[314.5 360. ]&[52.  55.6]\\
    0.8&1&7&[ 0.49 32.42]&7&[299.3 300. ]&[52.  57.7]\\
    0.999&1&7&[ 0.49 32.42]&7&[120.5 163. ]&[52.  53.9]\\
    1.001&1&6&[ 0.49 32.42]&6&[120.5 159.9]&[52.  62.5]\\
    1.2&1&6&[ 0.49 32.42]&6&[120.5 151. ]&[56.9 62.5]\\
    1.4&1&6&[ 0.49 32.42]&6&[120.5 163. ]&[52.  57.3]\\
    1.6&1&6&[ 0.49 32.42]&6&[120.5 163. ]&[52.  60.5]\\
    1.8&1&6&[ 0.49 32.42]&6&[120.5 163. ]&[52.  52.7]\\
    2.0&1&6&[ 0.49 32.42]&6&[120.5 163. ]&[52.  62.5]\\
    $\infty$&1&6&[ 0.49 32.42]&6&[120.5 163. ]&[52.  62.5]\\
    \bottomrule
    \end{tabular}
    \caption{Feature values that result in most positive difference in predicted probabilities between an Explainable Boosting Machine $f$ (\texttt{max\_bins} $= 32$) and an $8$-leaf decision tree reference model $f_0$ on the HMDA dataset. The 6 features that contribute most  are shown as a function of certification set radius $r$. For radius $r > 0$, the maximizing values of continuous features form an interval because the corresponding functions $f_j$ are piecewise constant.}
    \label{tab:HMDA_GAM_max}
\end{table*}

\paragraph{Feature combinations that maximize deviation} Table~\ref{tab:HMDA_GAM_max} shows feature values that yield the most positive difference between the predicted probabilities of the EBM $f$ with \texttt{max\_bins}$= 32$ and the $8$-leaf DT $f_0$. This table corresponds to the second ``Conflict between $f$, $f_0$'' example in Section~\ref{sec:expt}. The 6 features that contribute most to the deviation are shown. These contributions are determined using \eqref{eqn:maxAdditiveSep}; since the maximum deviation occurs in one of the $\ell_\infty$ ball-leaf intersections and this intersection is a Cartesian product, the feature-wise decomposition in \eqref{eqn:maxAdditiveSep} applies. The contribution of feature $j$ is then $\max_{x_j \in \mathcal{S}_j} f_j(x_j)$. We take an average of the contributions over $r$ to give a single ranking of features for all $r$. The same method is used to determine feature contributions and choose the top 6 features for Table~\ref{tab:HMDA_GAM_min}.

As mentioned in Section~\ref{sec:expt}, all solutions in Table~\ref{tab:HMDA_GAM_max} have `preapproval'=1 and debt-to-income ratios $>52\%$ that place them in leaf 2 in Figure~\ref{fig:HMDA_DT}. The latter results in a low predicted probability of approval from $f_0$ while the former makes a large positive contribution to the probability from $f$ (see Figure~\ref{fig:HMDA_GAM_functions}). The values of the other features also make increasingly larger positive contributions to $f$ as $r$ increases. Loan-to-value ratio decreases, while co-applicant credit score type moves from type 10 (no co-applicant, hence weaker application) to increasingly favorable score types (1, 7, 6, see Figure~\ref{fig:HMDA_GAM_functions}); applicant credit score is similar. This behavior is similar to the movement toward extreme points seen in Table~\ref{tab:HMDA_GAM_min}.

\begin{figure}[ht]
    \centering
    \begin{subfigure}[b]{0.33\textwidth}
        \includegraphics[width=\textwidth]{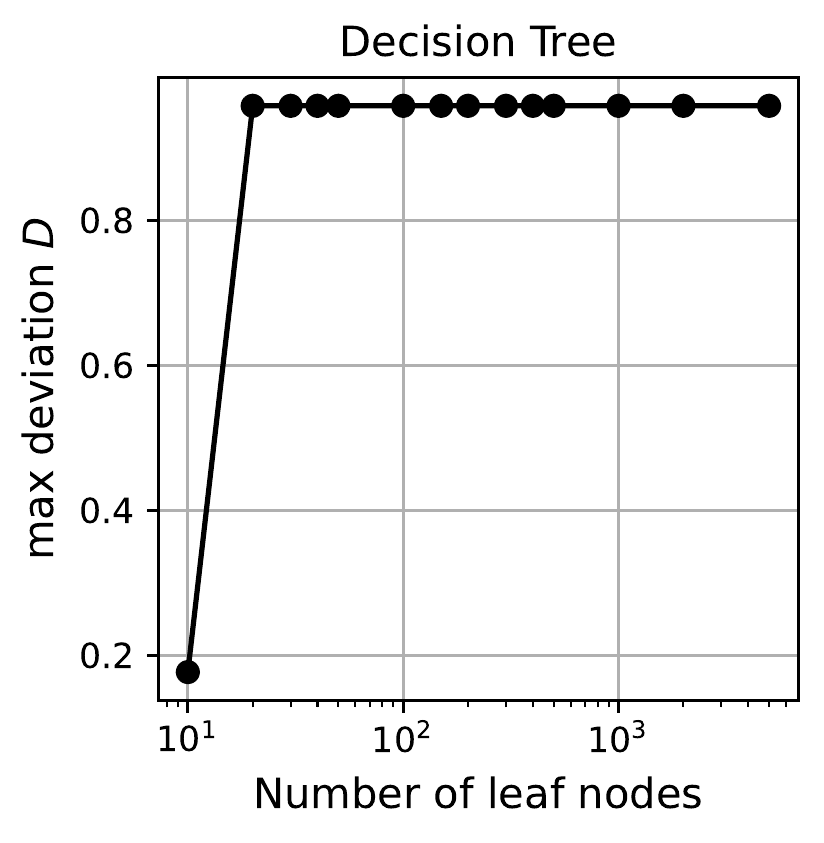}
        \caption{}
        \label{fig:HMDA_DT_leaves}
    \end{subfigure}
    \begin{subfigure}[b]{0.33\textwidth}
        \includegraphics[width=\textwidth]{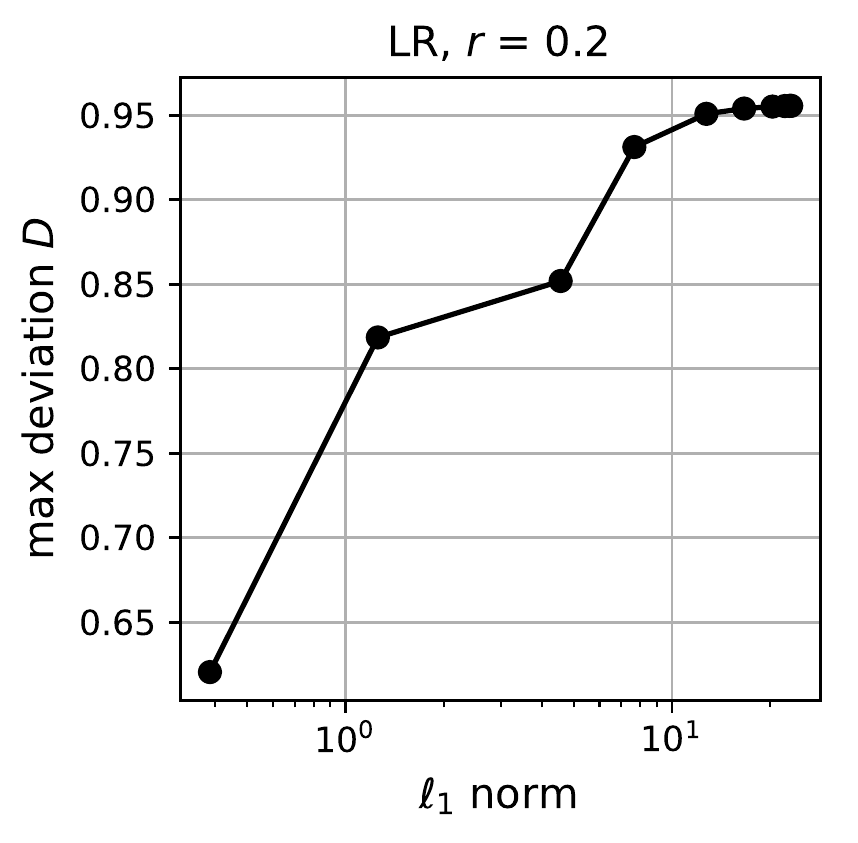}
        \caption{}
        \label{fig:HMDA_LR_C}
    \end{subfigure}
    \begin{subfigure}[b]{0.33\textwidth}
        \includegraphics[width=\textwidth]{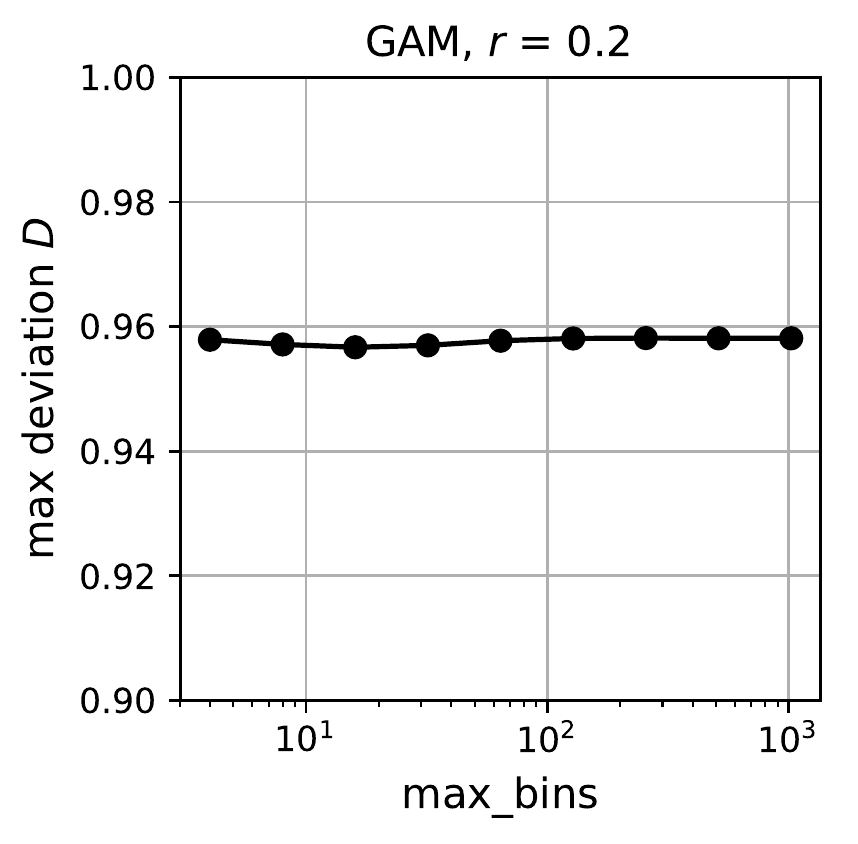}
        \caption{}
        \label{fig:HMDA_GAM_maxBins}
    \end{subfigure}
    \begin{subfigure}[b]{0.5\textwidth}
        \includegraphics[width=\textwidth]{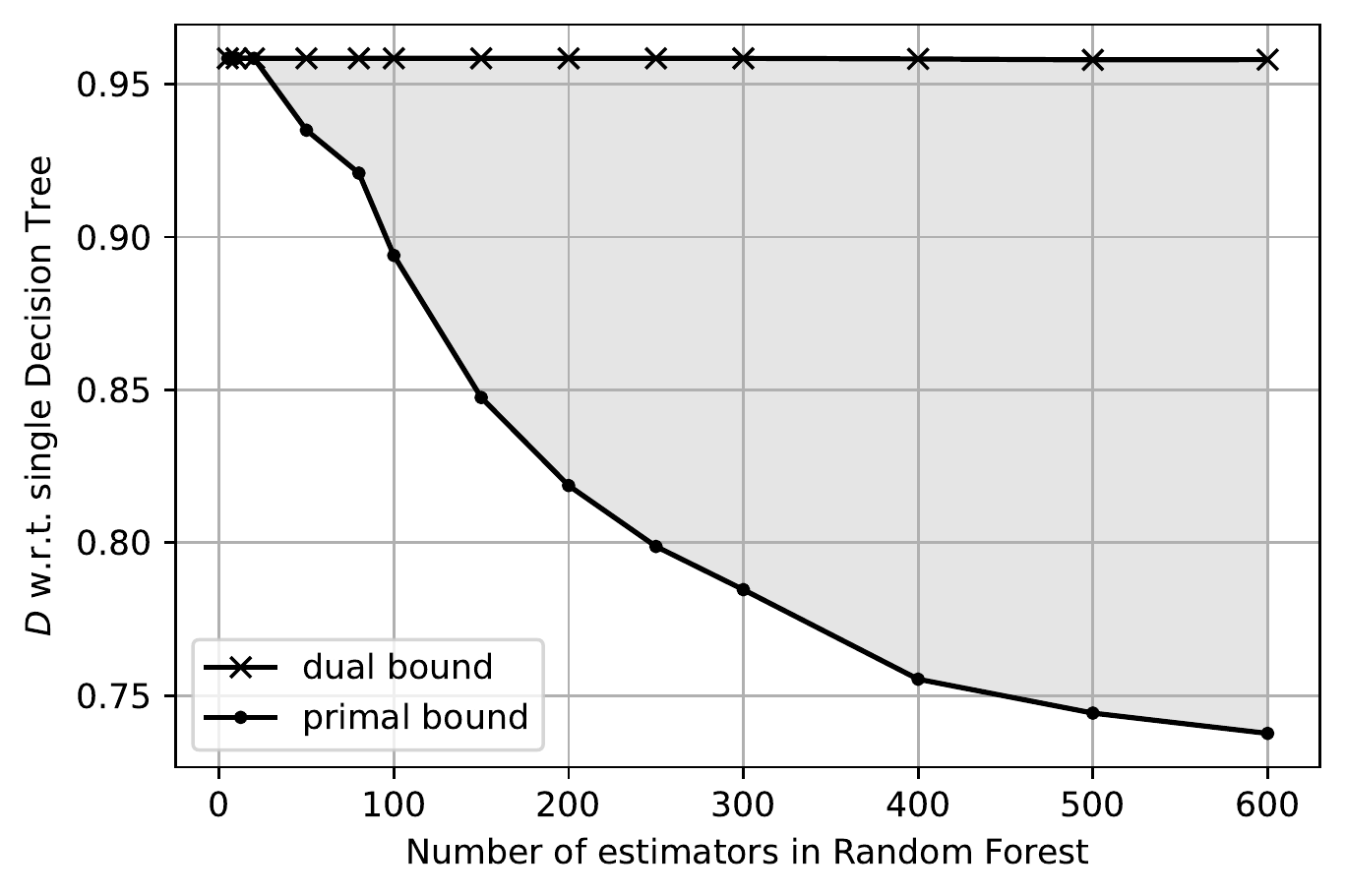}
        \caption{}
        \label{fig:HMDA_RF_trees}
    \end{subfigure}
    \caption{Maximum deviation $D$ on the HMDA dataset as a function of model complexity for (a) DT (number of leaves), (b) LR ($\ell_1$ norm), (c) GAM (\texttt{max\_bins}), and (d) RF (number of estimators).
    }
    \label{fig:HMDA_all}
\end{figure}

\paragraph{Dependence on model complexity}
Figure~\ref{fig:HMDA_all} shows the dependence of maximum deviation on the complexity of model $f$, quantified by the number of leaves for DTs, coefficient $\ell_1$ norm for LR, \texttt{max\_bins} for EBM, and the number of estimators (trees) for RF. 
The DT and RF cases demonstrate the methods in Section \ref{sec:model:ensemble}, specifically Algorithm~\ref{alg:tree-ensembles:enum}, where in the RF case, the algorithm may only provide bounds after a time limit of two hours. The plots show that maximum deviation may or may not increase with model complexity. 
In Figure~\ref{fig:HMDA_DT_leaves}, the deviation is small for a $10$-leaf DT and increases rapidly.
Figures~\ref{fig:HMDA_LR_C} and \ref{fig:HMDA_GAM_maxBins} 
indicate that maximum deviation is sensitive to the $\ell_1$ norm of LR models but not to the \texttt{max\_bins} parameter of EBMs. The latter may increase the resolution of the EBM shape functions but not their dynamic range.

\begin{figure}[ht]
    \centering
    \begin{subfigure}[b]{0.33\textwidth}
        \includegraphics[width=\textwidth]{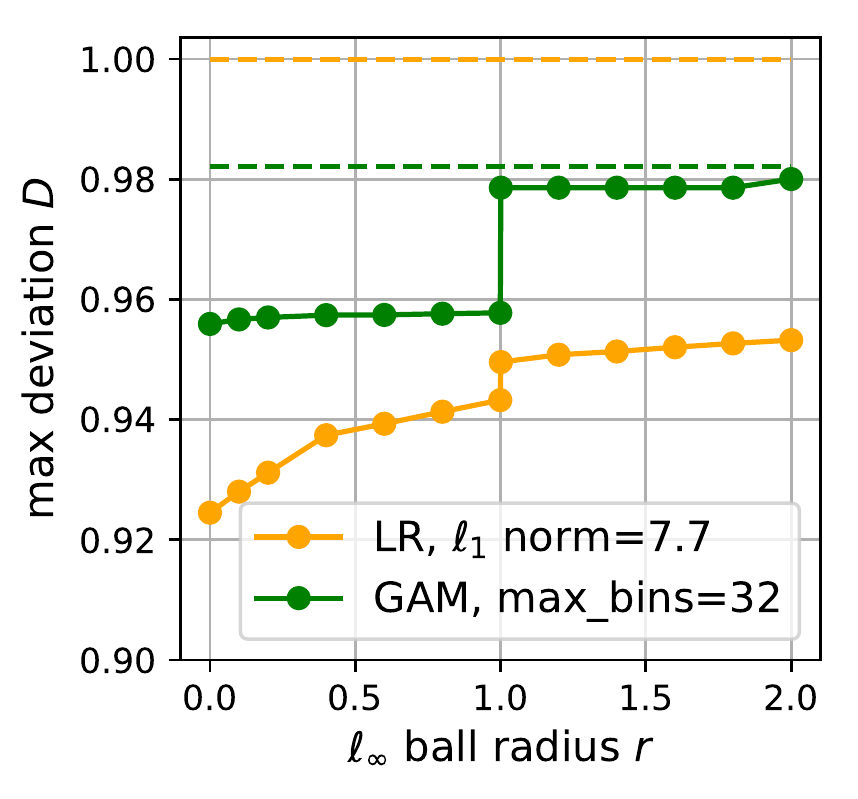}
        \caption{}
        \label{fig:HMDA_LR_GAM_r}
    \end{subfigure}\\
    \begin{subfigure}[b]{0.495\textwidth}
        \includegraphics[width=\textwidth]{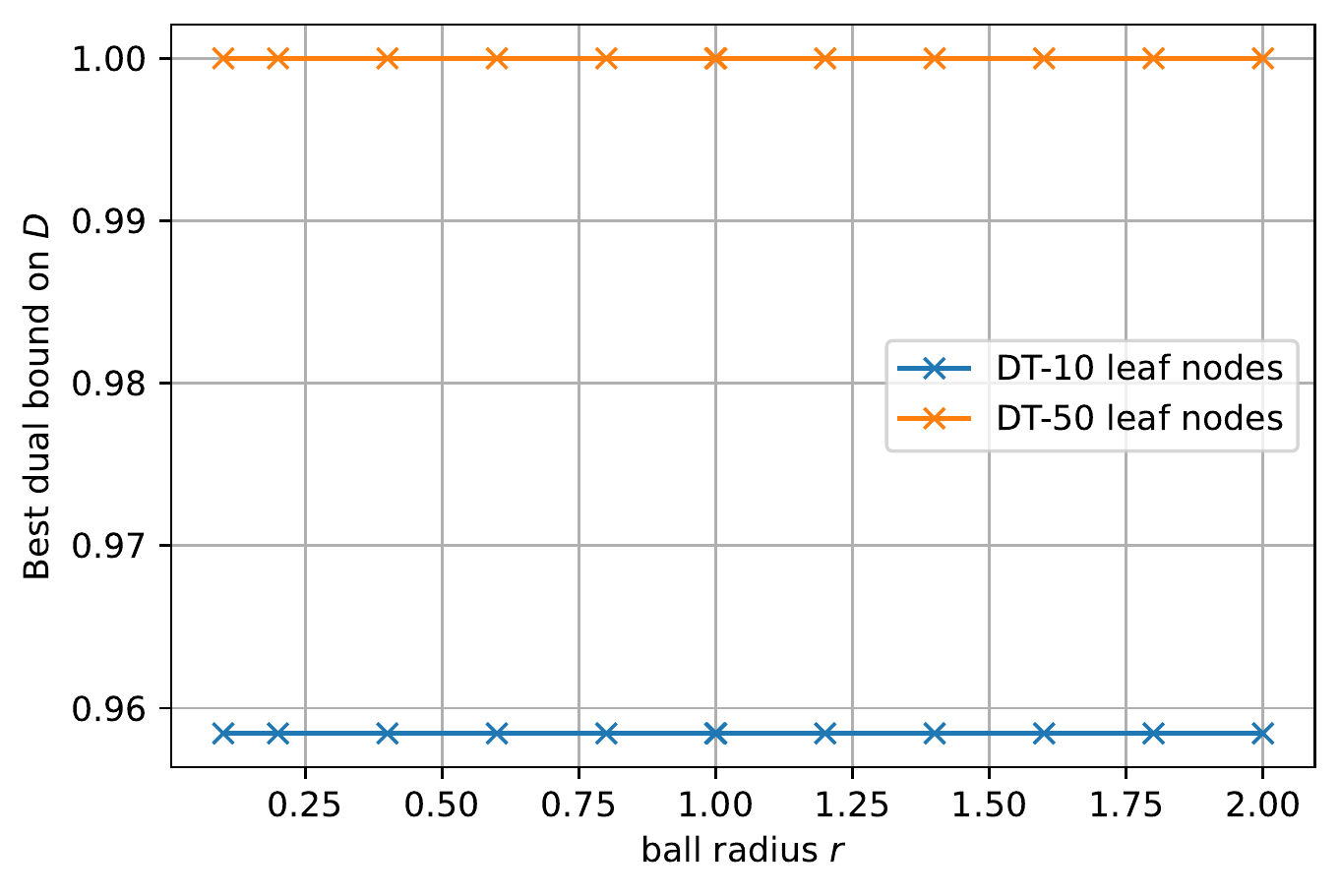}
        \caption{}
        \label{fig:HMDA_DT_r}
    \end{subfigure}
    \begin{subfigure}[b]{0.495\textwidth}
        \includegraphics[width=\textwidth]{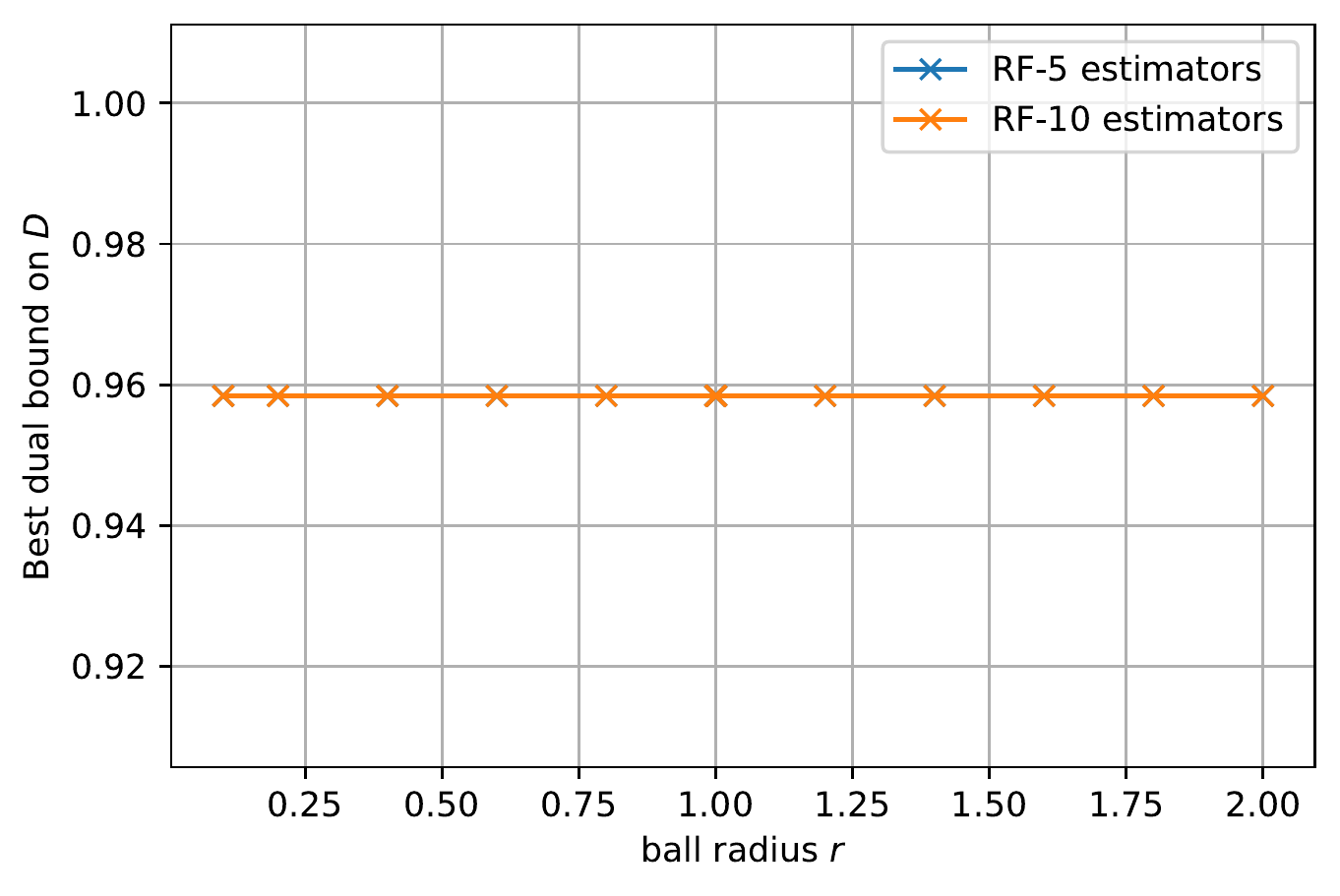}
        \caption{}
        \label{fig:HMDA_RF_r}
    \end{subfigure}
    \caption{Maximum deviation $D$ on the HMDA dataset as a function of certification set radius $r$ for (a) LR and GAM, (b) DT (10 and 50 leaves), (c) RF (5 and 10 estimators). Dashed lines in (a) indicate the $r \to \infty$ asymptote of the curve of the same color.}
    \label{fig:HMDA_r}
\end{figure}

\paragraph{Dependence on certification set size}
Figure~\ref{fig:HMDA_r} shows the dependence of maximum deviation on the certification set radius $r$. For LR and GAM, the maximum deviation is greater for $r > 0$ than for $r = 0$, showing that evaluation on a finite test set may not be sufficient and infinite certification sets (with $r > 0$) should be considered, especially to account for unexpected, out-of-distribution deviations. There are jumps at $r = 1$ because this is the radius that permits values of categorical features of test set points (the ball centers in \eqref{eqn:unionBalls}) to change to any other value. In the case of GAM, this jump is sufficient for the deviation to equal that for $r = \infty$ ($\certset = \mathcal{X}$, dashed line in figure). On the other hand, the deviation for DT and RF remains constant as a function of $r$.

\begin{figure}[ht]
    \centering
    \includegraphics[width=0.495\textwidth]{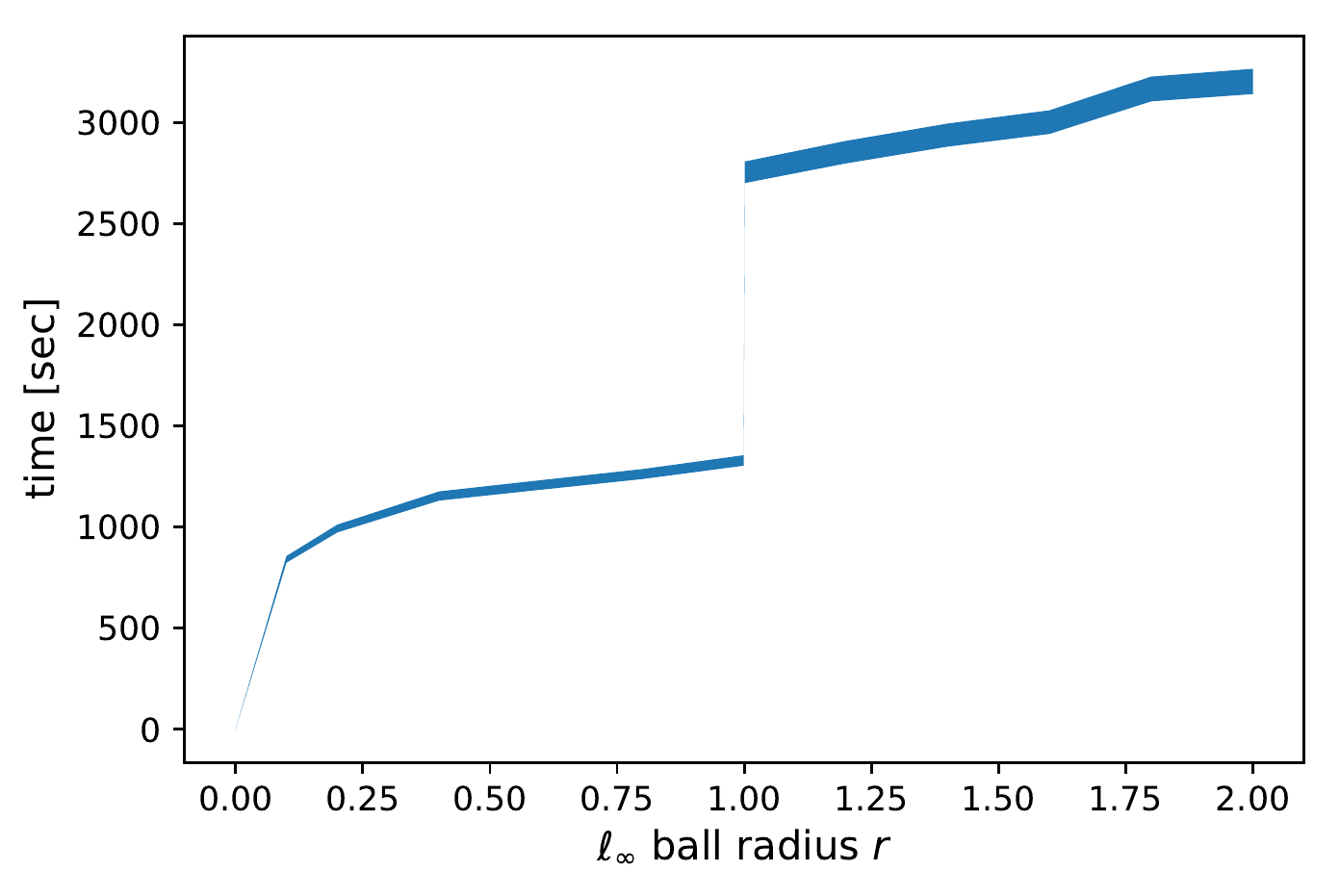}
    \includegraphics[width=0.495\textwidth]{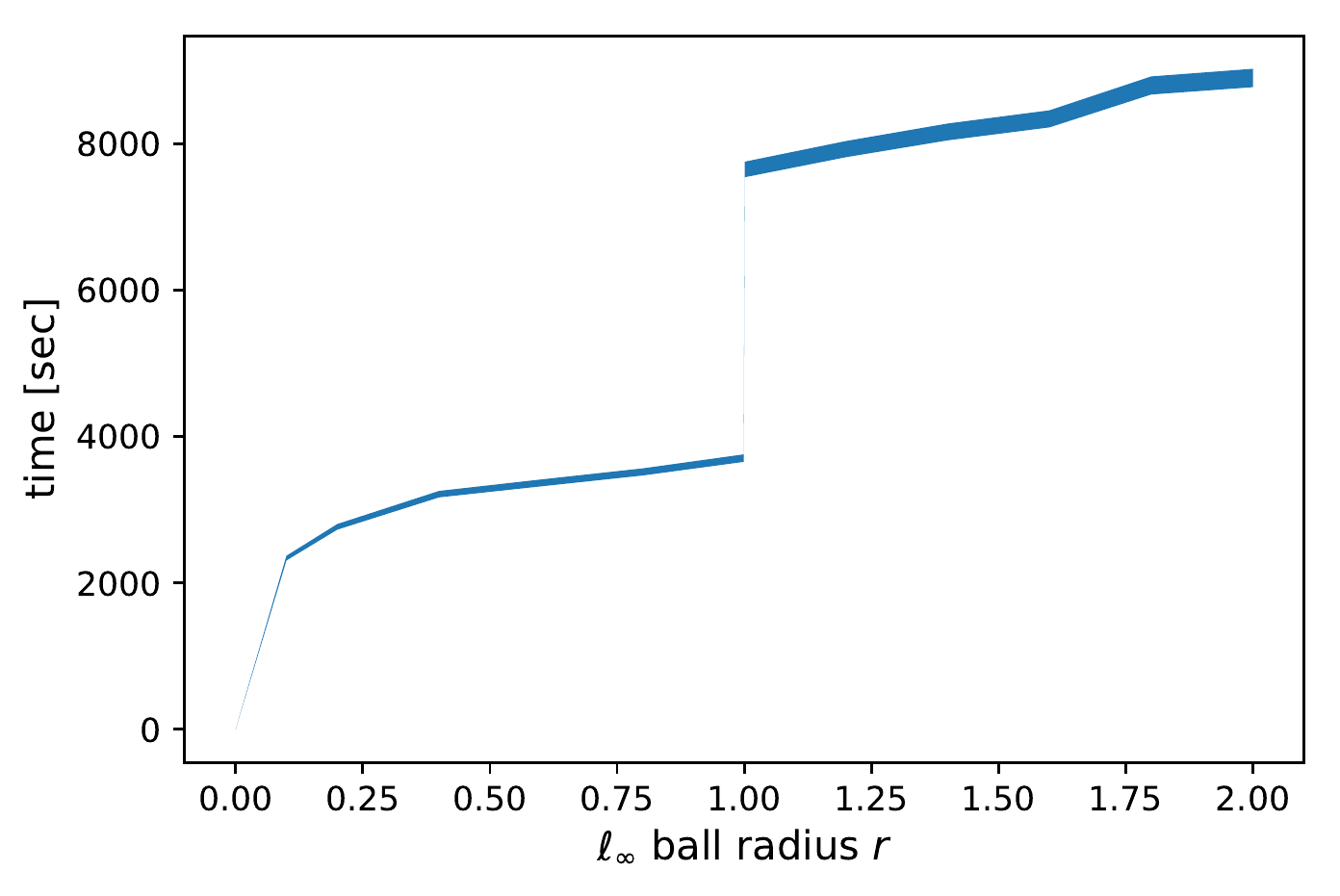}
    \caption{Time to compute maximum deviation for logistic regression models (left) and Explainable Boosting Machines (right) on the HMDA dataset as a function of certification set size (radius $r$). The filled-in region shows the min-max variation with model complexity ($\ell_1$ norm for LR, \texttt{max\_bins} for EBM).}
    \label{fig:HMDA_time}
\end{figure}

\paragraph{Running time} Figure~\ref{fig:HMDA_time} shows the time required to compute the maximum deviation for LR and GAM on the HMDA dataset. These times were obtained using a single $2.0$ GHz core of a server with $64$ GB of memory (only a small fraction of which was used) running Ubuntu 16.04 (64-bit).  
The times increase with the $\ell_\infty$ ball radius $r$ because of the increasing number of ball-leaf intersections that become non-empty and hence need to be evaluated. The time for $r = 0$ is minimal because this case requires only model evaluation over the finite test set, as mentioned. The jumps at $r = 1$ are due again to the ability of categorical features to change values, leading to an increase in ball-leaf intersections. The filled-in regions show that there was little variation due to different $\ell_1$ norms for LR or \texttt{max\_bins} for GAM. This was most likely because of a vectorized implementation, which operates on all LR coefficients or all GAM bins at once (i.e., without a for loop).

\clearpage

\subsection{Adult Income dataset}
\label{sec:expt_add:Adult}

We use the given partition of the Adult Income dataset into training and test sets.

\begin{figure}[t]
    \centering
    \includegraphics[width=\textwidth]{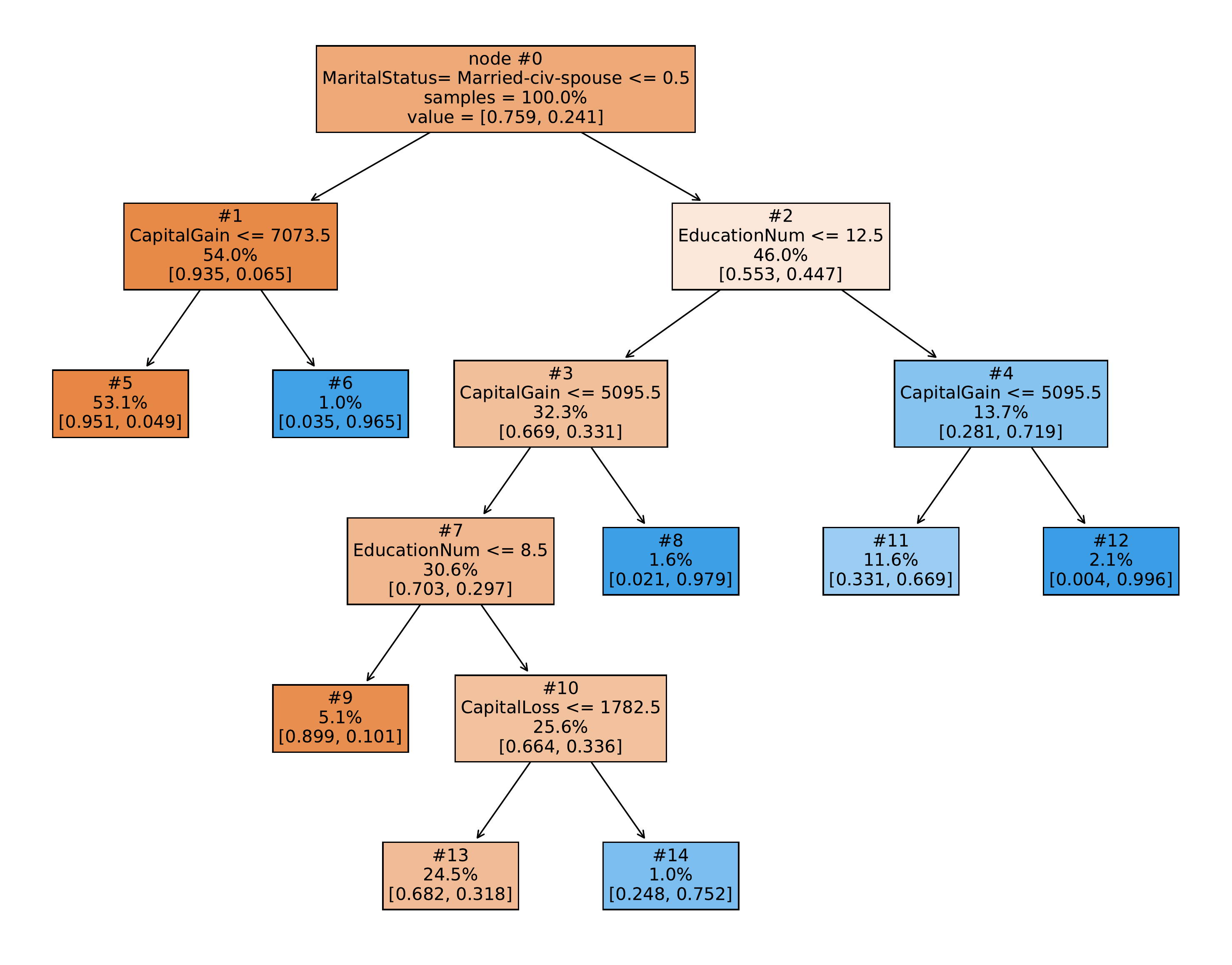}
    \caption{Decision tree reference model with 8 leaves for the Adult Income dataset.}
    \label{fig:Adult_DT}
\end{figure}

\paragraph{Reference model} Figure~\ref{fig:Adult_DT} depicts the $8$-leaf DT reference model used in the experiments on the Adult Income dataset. This DT has $85.0\%$ accuracy on the test set. The root node separates individuals based on whether the marital status is \textit{Married-civ-spouse}. The remaining splits divide the population into those with high and low education, high and low capital gains, and high and low capital losses. In particular, having high capital gains or losses is a good predictor of high income ($> \$50000$).

\begin{table}[ht]
    \small
    \centering
    \begin{tabular}{rrrrr}
    \toprule
    $C$&nonzeros&$\ell_1$ norm&accuracy&AUC\\
    \midrule
    3e-4&1&0.2&0.764&0.715\\
    1e-3&6&2.6&0.825&0.885\\
    3e-3&7&4.7&0.840&0.895\\
    1e-2&16&7.4&0.848&0.900\\
    3e-2&30&12.9&0.852&0.905\\
    1e-1&38&17.3&0.853&0.905\\
    3e-1&62&25.3&0.853&0.905\\
    1e+0&83&40.7&0.852&0.905\\
    3e+0&92&54.6&0.852&0.905\\
    1e+1&101&65.1&0.852&0.904\\
    3e+1&105&73.5&0.852&0.904\\
    1e+2&107&77.6&0.852&0.904\\
    3e+2&107&79.1&0.852&0.904\\
    \bottomrule
    \end{tabular}
    \caption{Number of nonzero coefficients, $\ell_1$ norm of coefficients, test set accuracy, and AUC for logistic regression models on the Adult Income dataset as a function of inverse $\ell_1$ penalty $C$.}
    \label{tab:Adult_LR_stats}
\end{table}

\begin{table}[ht]
    \small
    \centering
    \begin{tabular}{rrr}
    \toprule
    \texttt{max\_bins}&accuracy&AUC\\
    \midrule
    4&0.858&0.910\\
    8&0.862&0.915\\
    16&0.865&0.920\\
    32&0.870&0.924\\
    64&0.871&0.925\\
    128&0.871&0.925\\
    256&0.872&0.925\\
    512&0.871&0.925\\
    1024&0.871&0.925\\
    \bottomrule
    \end{tabular}
    \caption{Test set accuracy and AUC for Explainable Boosting Machines on the Adult Income dataset as a function of \texttt{max\_bins} parameter.}
    \label{tab:Adult_GAM_stats}
\end{table}

\paragraph{LR and GAM models} Tables~\ref{tab:Adult_LR_stats} and \ref{tab:Adult_GAM_stats} show the values of $C$ and \texttt{max\_bins} used for LR and GAM respectively together with statistics of the resulting classifiers. Based in part on Tables~\ref{tab:Adult_LR_stats} and \ref{tab:Adult_GAM_stats}, we select $C = 0.01$ and \texttt{max\_bins} $= 8$ as representative models that remain simple and have accuracies and AUCs not far from the maximum attainable. Plots for these two models are shown in Figures~\ref{fig:Adult_LR_coef} and \ref{fig:Adult_GAM_functions}.

\begin{figure}[ht]
    \centering
    \includegraphics[width=0.8\textwidth]{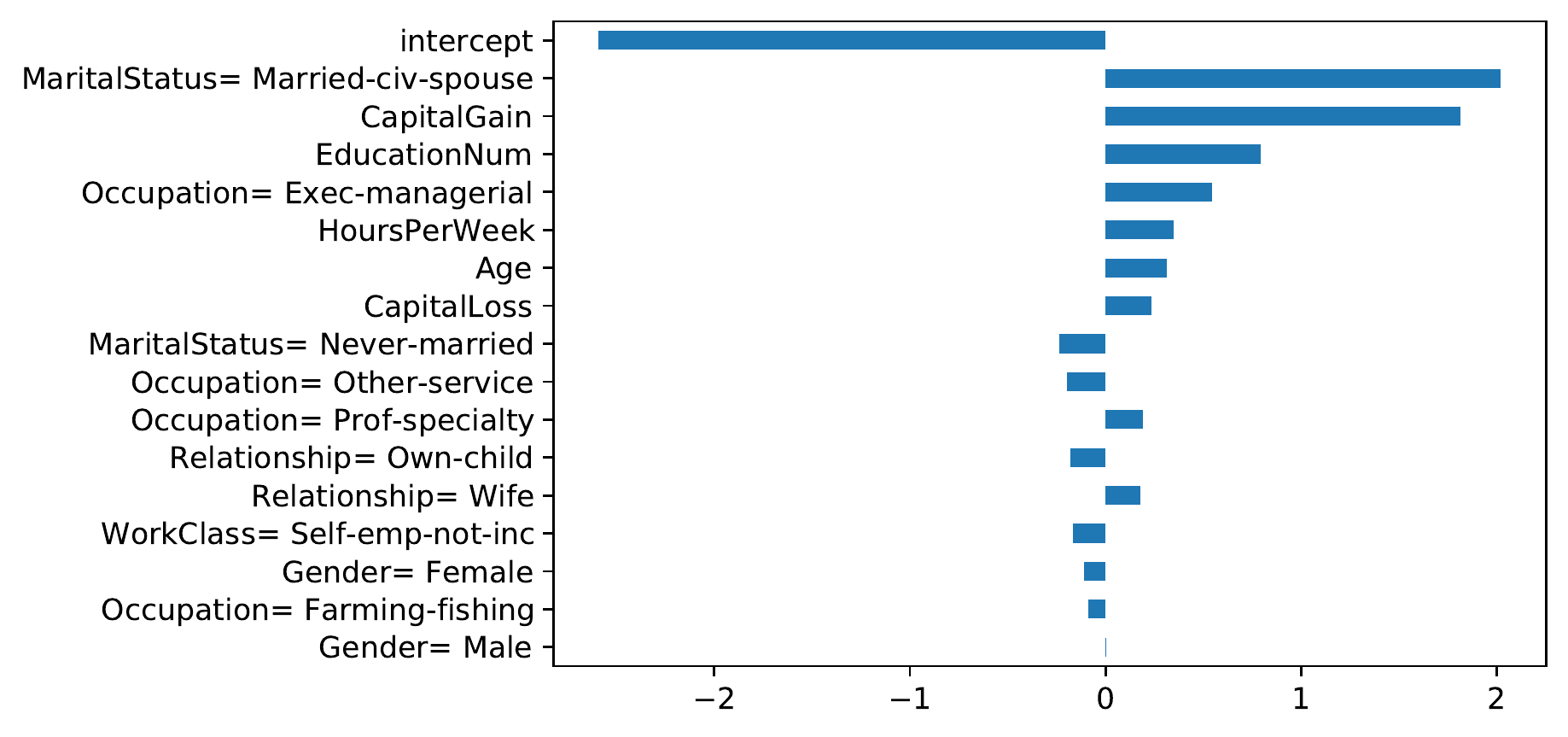}
    \caption{Coefficient values of the logistic regression model with $C = 0.01$ (16 nonzeros) for the Adult Income dataset.}
    \label{fig:Adult_LR_coef}
\end{figure}

\begin{figure}[t]
    \centering
    \includegraphics[width=0.495\textwidth]{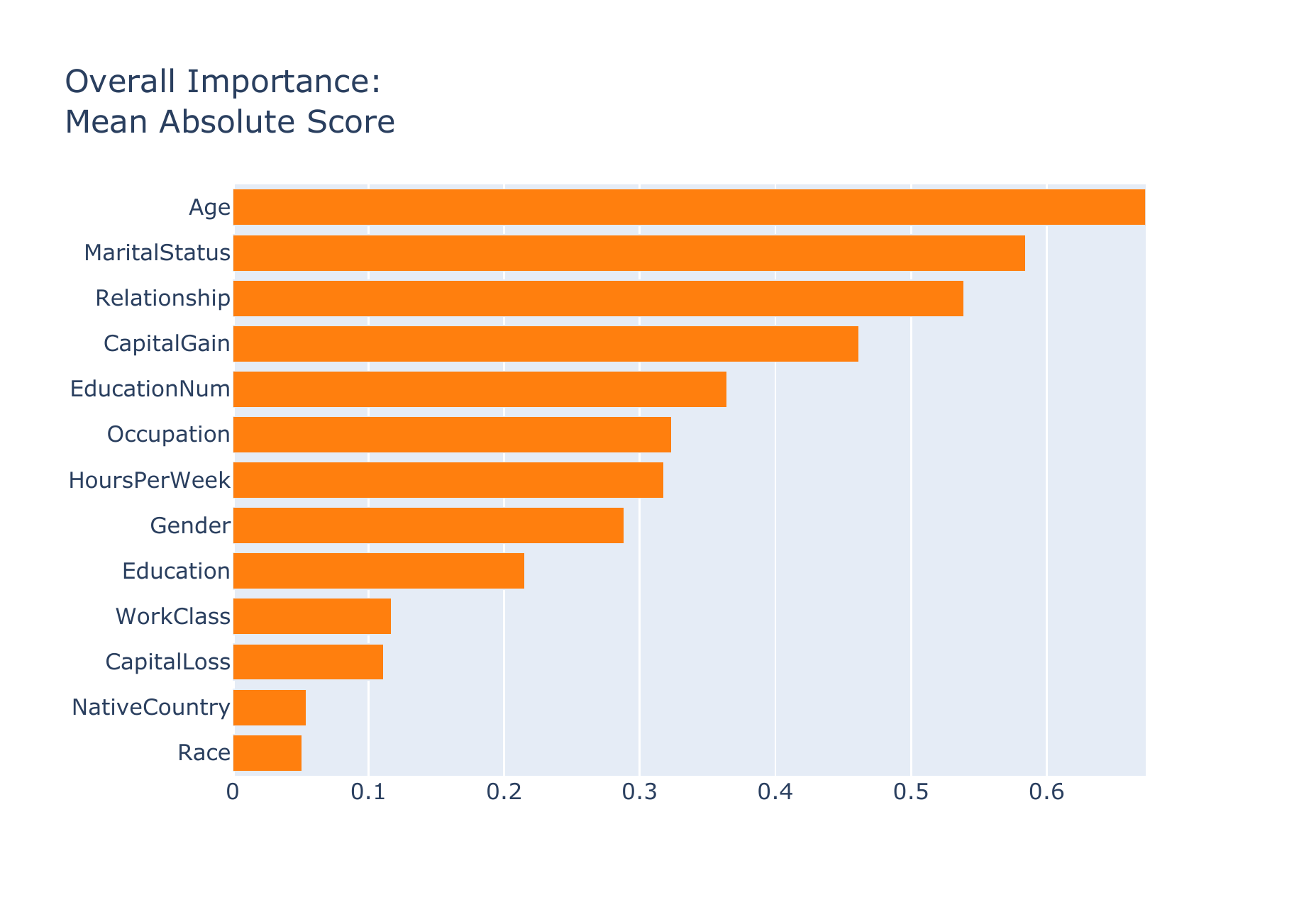}
    \includegraphics[width=0.495\textwidth]{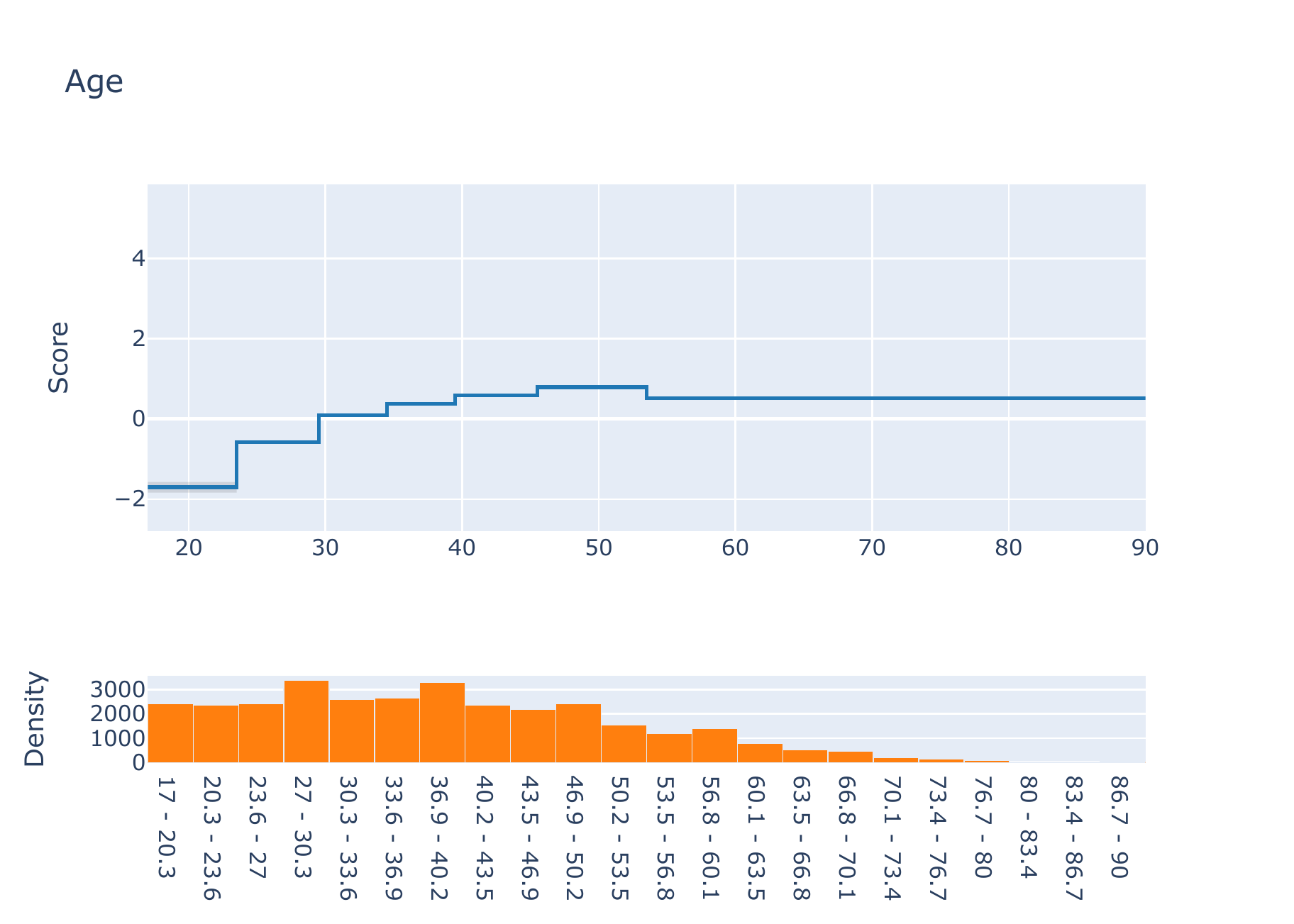}
    \includegraphics[width=0.495\textwidth]{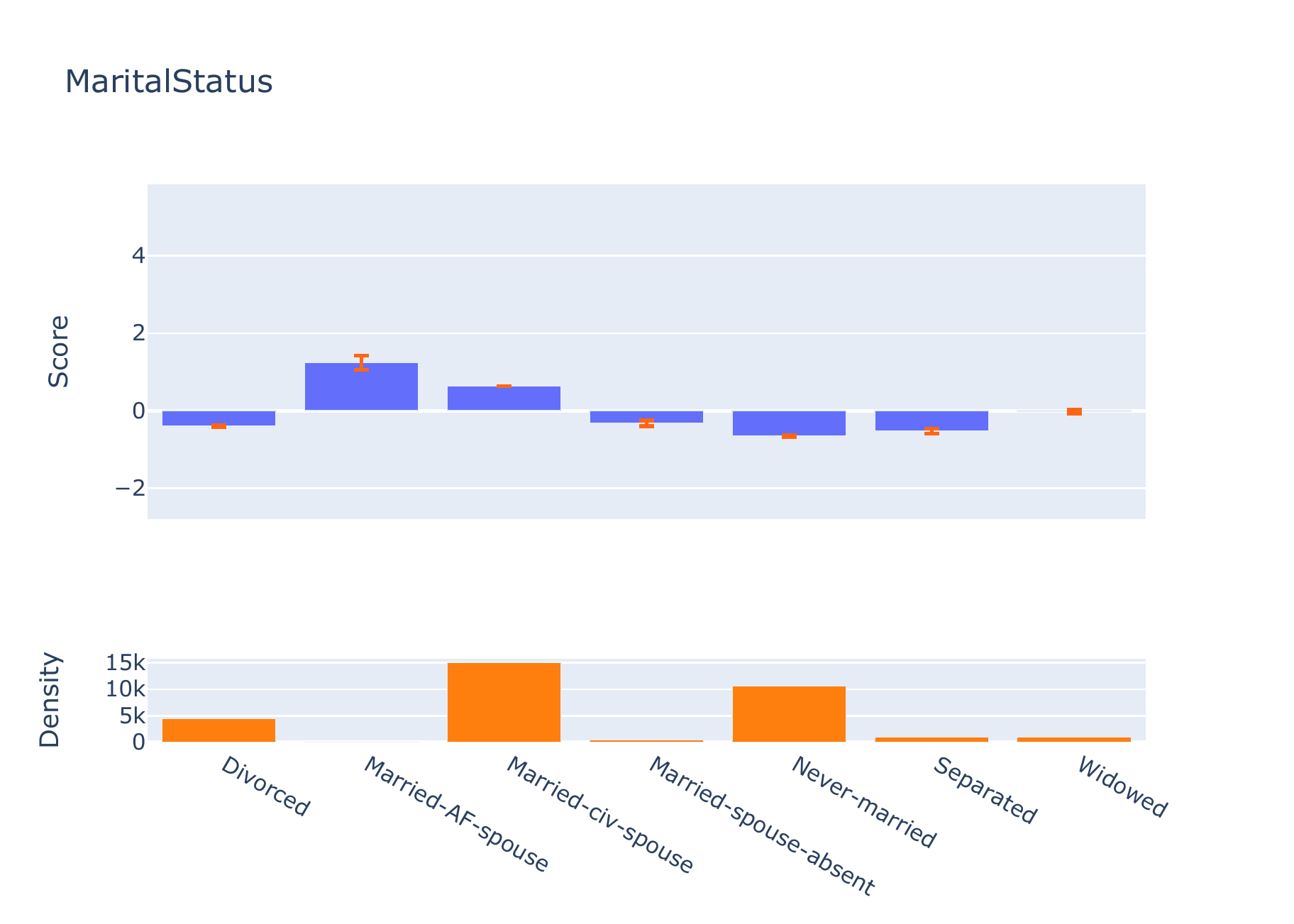}
    \includegraphics[width=0.495\textwidth]{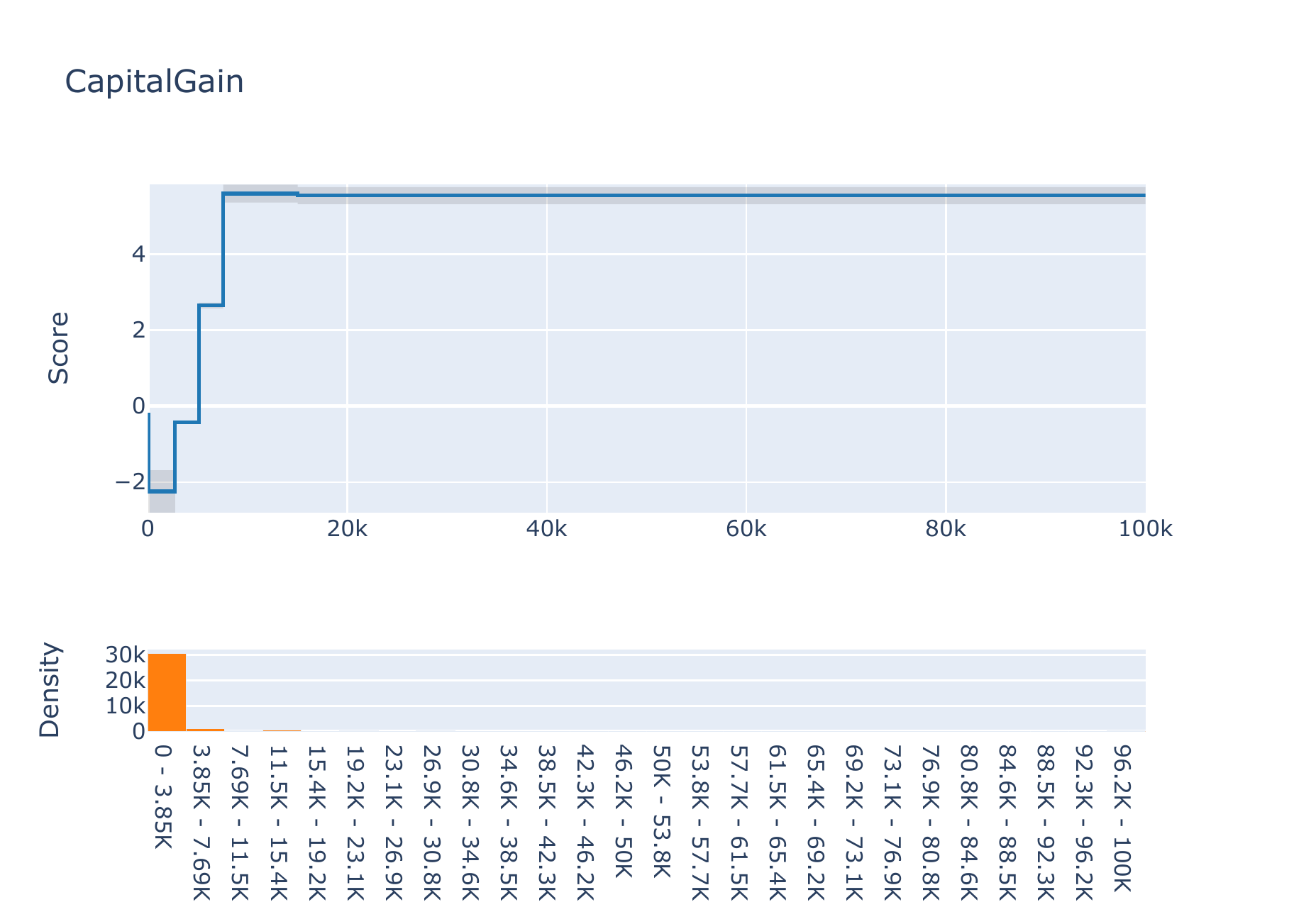}
    \includegraphics[width=0.495\textwidth]{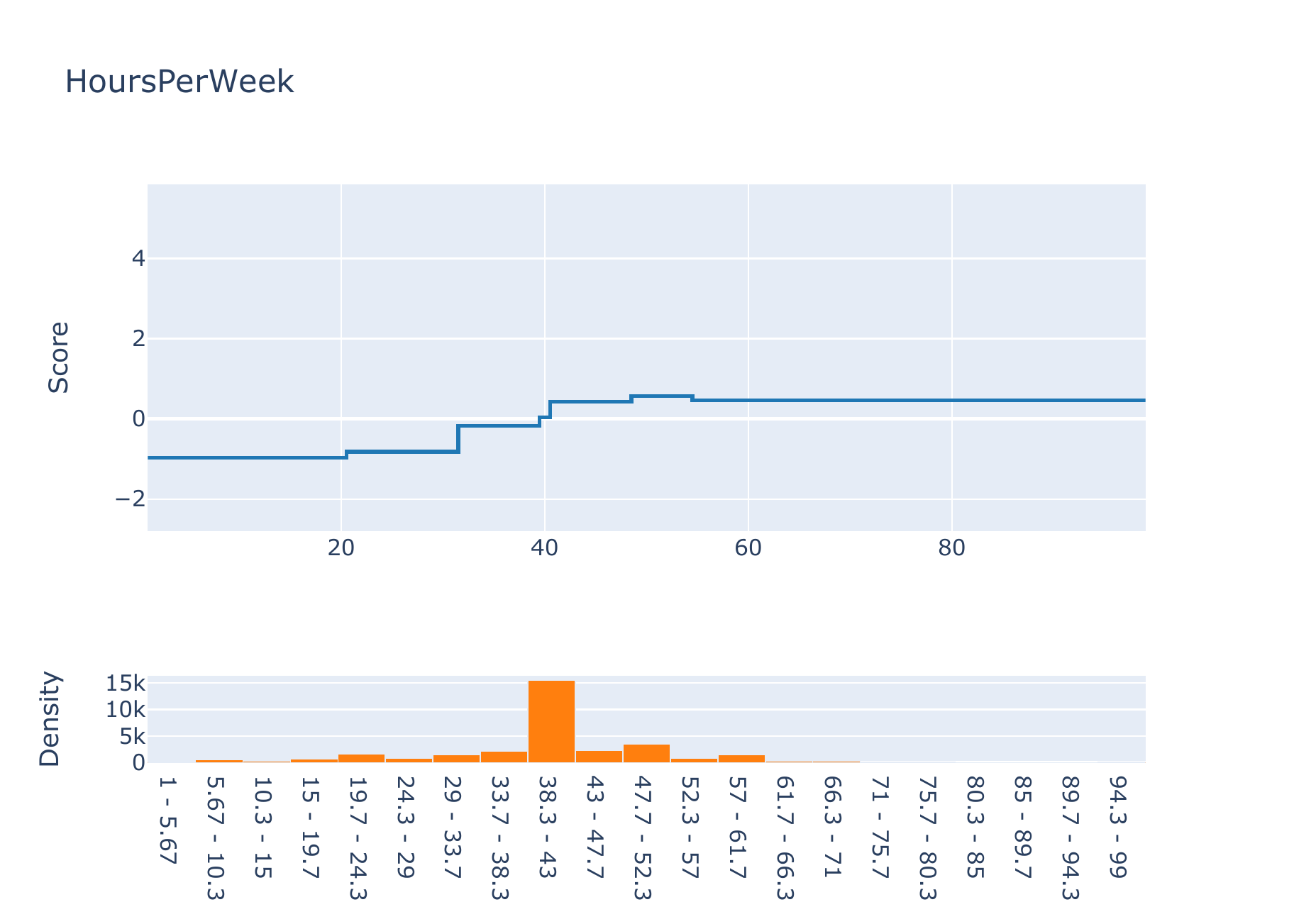}
    \includegraphics[width=0.495\textwidth]{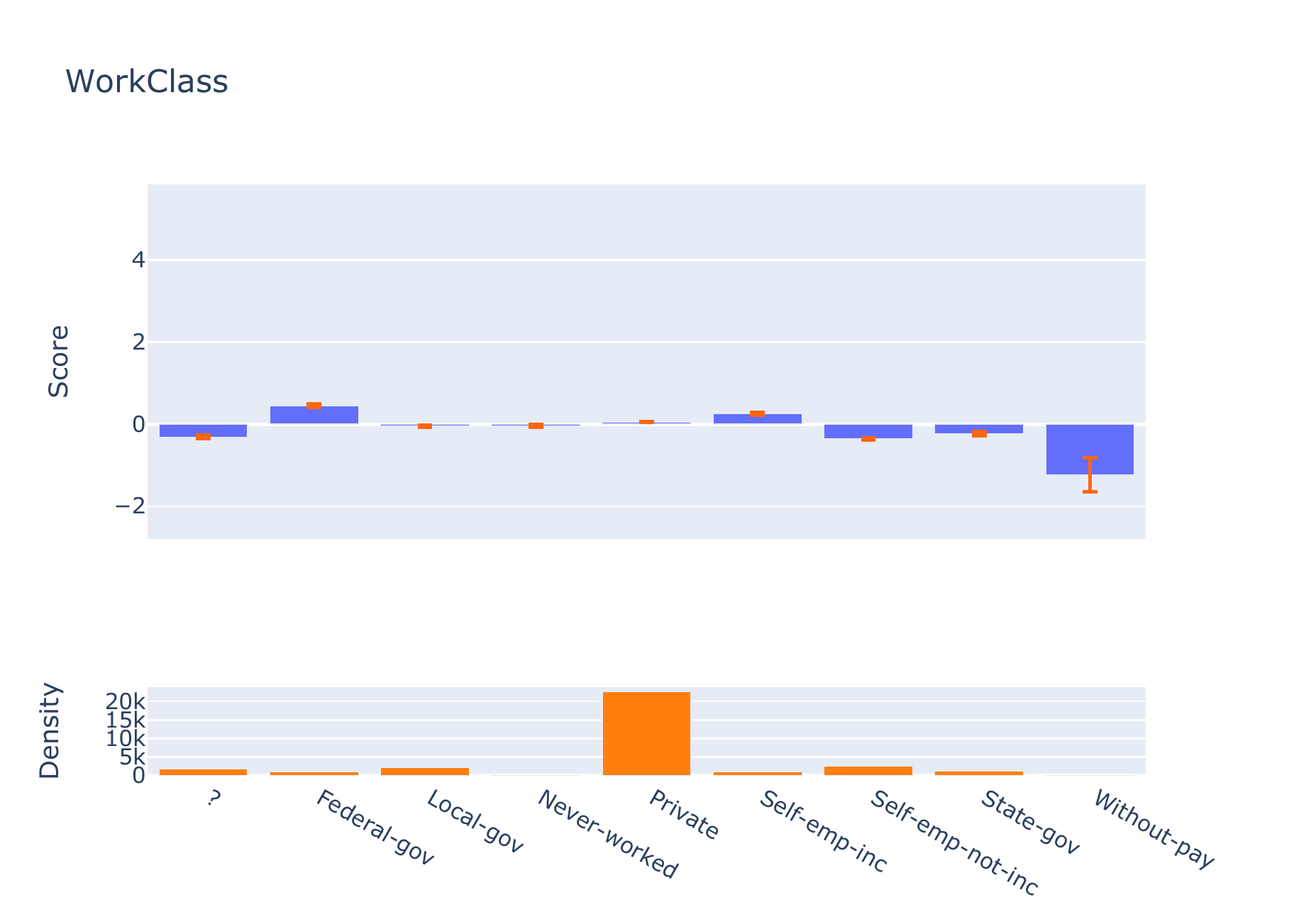}
    \includegraphics[width=0.495\textwidth]{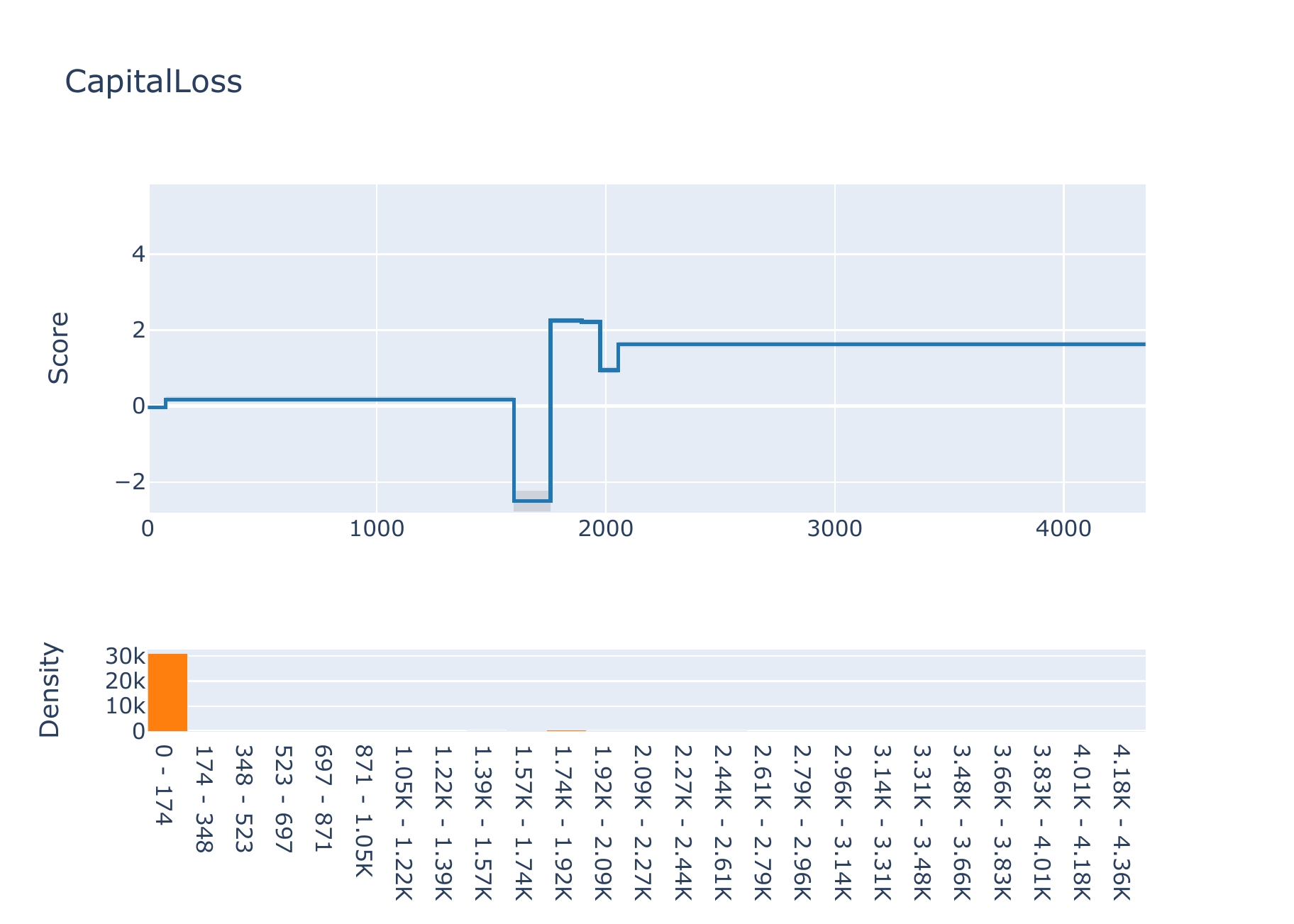}
    \includegraphics[width=0.495\textwidth]{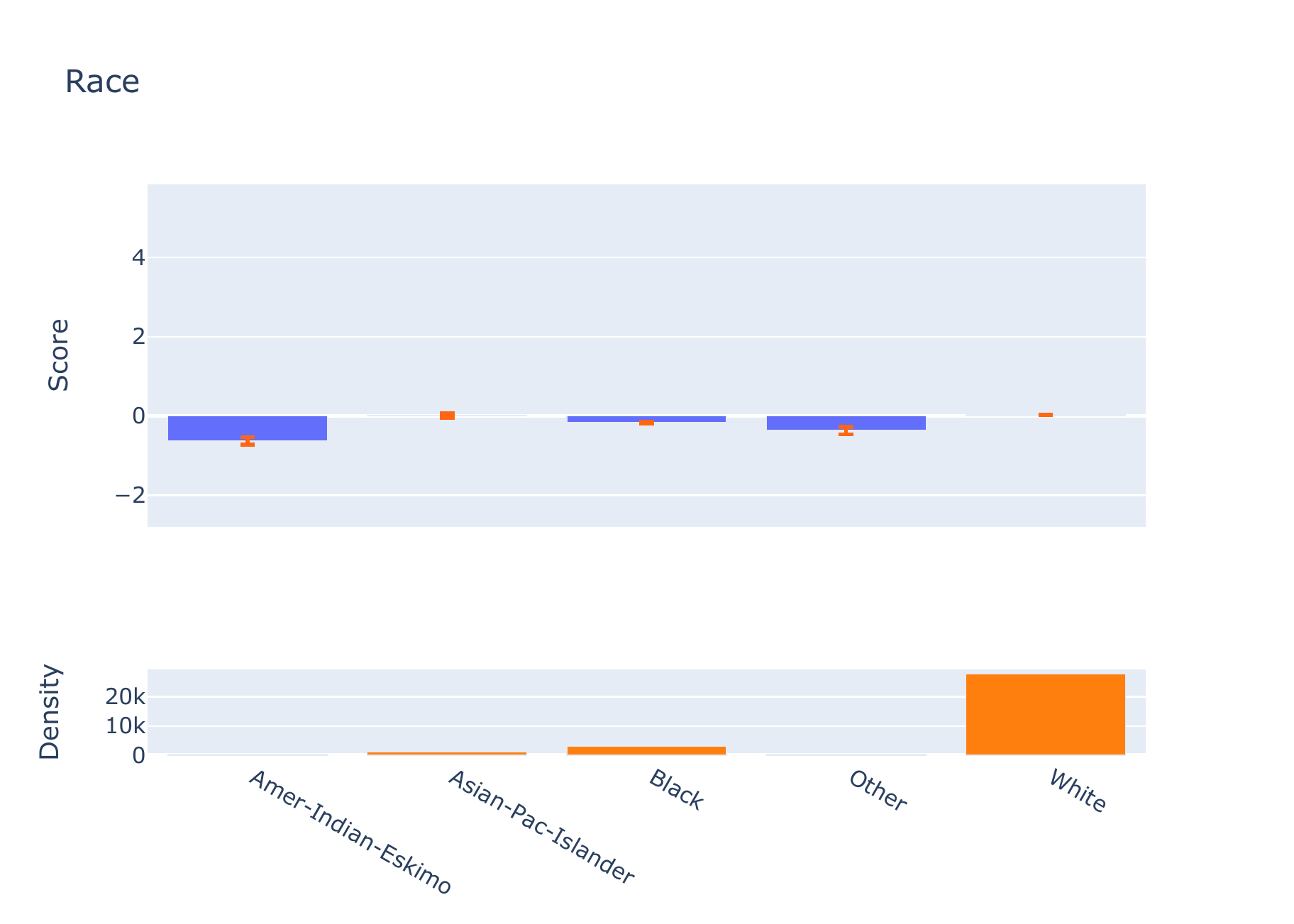}
    \caption{Feature importances and selected univariate functions $f_j$ for the Explainable Boosting Machine with \texttt{max\_bins} $= 8$ on the Adult Income dataset.}
    \label{fig:Adult_GAM_functions}
\end{figure}

\clearpage

\begin{figure}[ht]
    \centering
    \begin{subfigure}[b]{0.33\textwidth}
        \includegraphics[width=\textwidth]{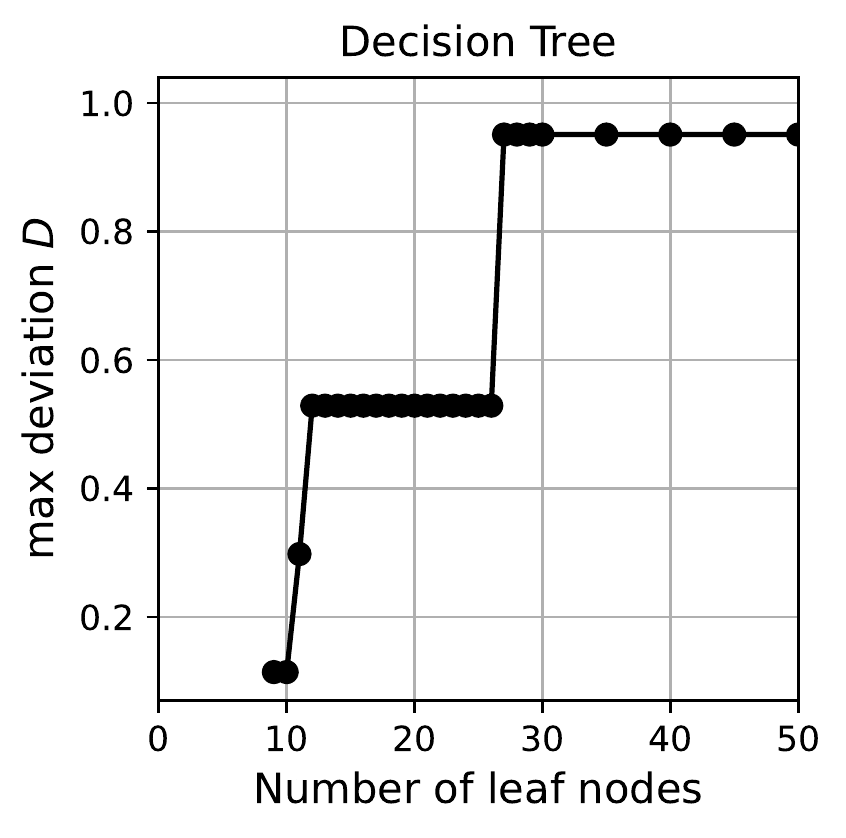}
        \caption{}
        \label{fig:Adult_DT_leaves_linear}
    \end{subfigure}%
    \begin{subfigure}[b]{0.33\textwidth}
        \includegraphics[width=\textwidth]{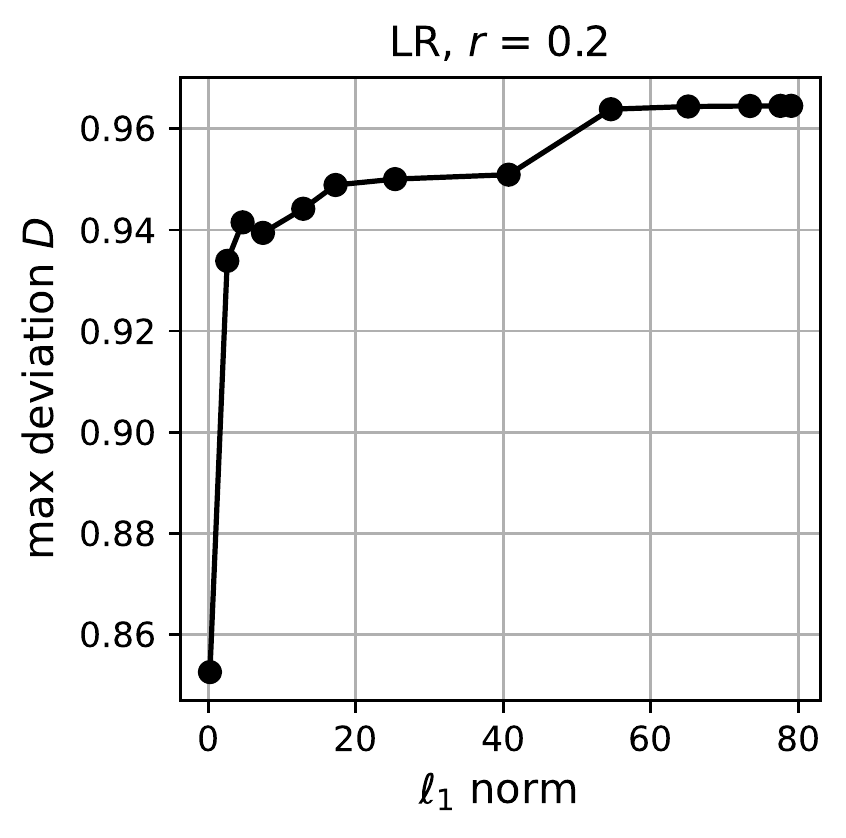}
        \caption{}
        \label{fig:Adult_LR_C}
    \end{subfigure}%
    \begin{subfigure}[b]{0.33\textwidth}
        \includegraphics[width=\textwidth]{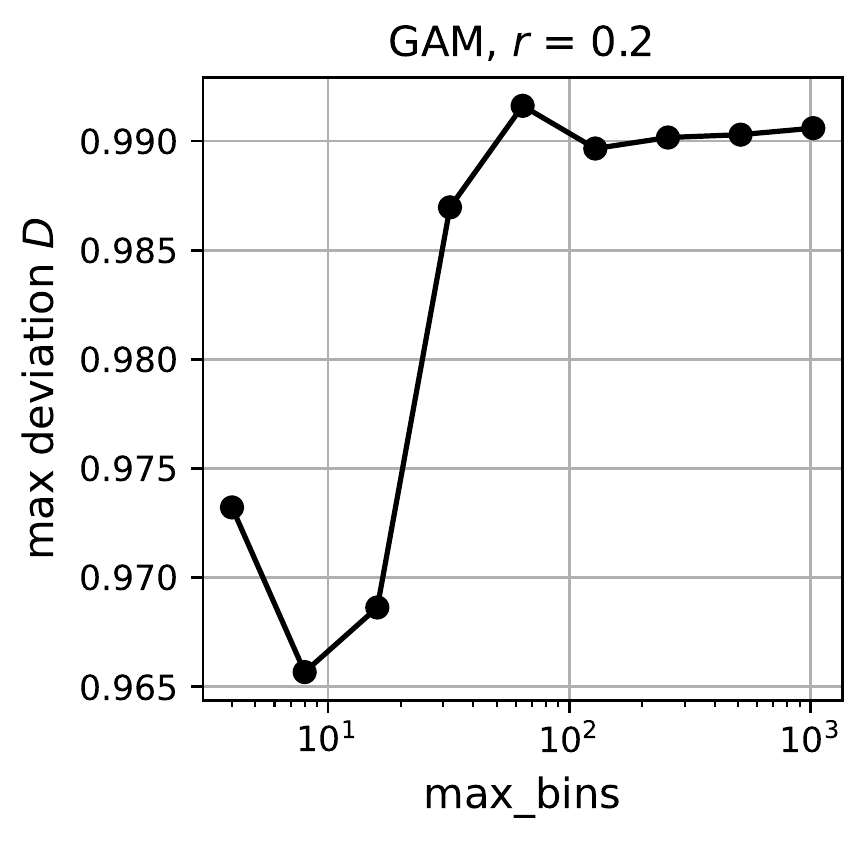}
        \caption{}
        \label{fig:Adult_GAM_maxBins}
    \end{subfigure}
    \caption{Maximum deviation $D$ on the Adult Income dataset as a function of model complexity for (a) DT (number of leaves), (b) LR ($\ell_1$ norm), and (c) GAM (\texttt{max\_bins}). 
    }
    \label{fig:Adult_all_add}
\end{figure}

\paragraph{Dependence on model complexity} 
Figure~\ref{fig:Adult_all_add} shows the dependence on model complexity for DT (number of leaves), LR (coefficient $\ell_1$ norm), and EBM (\texttt{max\_bins}). In Figure~\ref{fig:Adult_DT_leaves_linear}, 
the maximum deviation is $0.114$ for trees with $9$ and $10$ leaves, and remains moderate up to $26$ leaves, which is different than in Figure~\ref{fig:HMDA_DT_leaves}. Similar to Figures~\ref{fig:HMDA_LR_C} and \ref{fig:HMDA_GAM_maxBins}, $\ell_1$ norm has a larger effect on maximum deviation than \texttt{max\_bins} (note the vertical scale
in Figure~\ref{fig:Adult_GAM_maxBins}).

\paragraph{Dependence on certification set size}
Figure~\ref{fig:Adult_DT_LR_GAM_r} shows the dependence on the certification set radius $r$ for DT, LR, and GAM. The patterns are similar to those in Figure~\ref{fig:HMDA_r}: the deviations for LR and GAM increase from $r = 0$ and have jumps at $r = 1$, while the deviation for DT remains constant. One difference is that the LR curve in Figure~\ref{fig:Adult_DT_LR_GAM_r} meets its $r \to \infty$ asymptote (dashed line in figure), similar to GAM.

Figure \ref{fig:adult:d_vs_Q} shows the upper bound on the maximum deviation as a function of the certification set size for two RF models. As the test set is large in this case, the deviations observed even for small values of $r$ are high and grow to reach the value of the full feature space quickly.

\begin{figure}[htb]
    \centering
    \begin{subfigure}[b]{0.3\textwidth}
        \includegraphics[width=\textwidth]{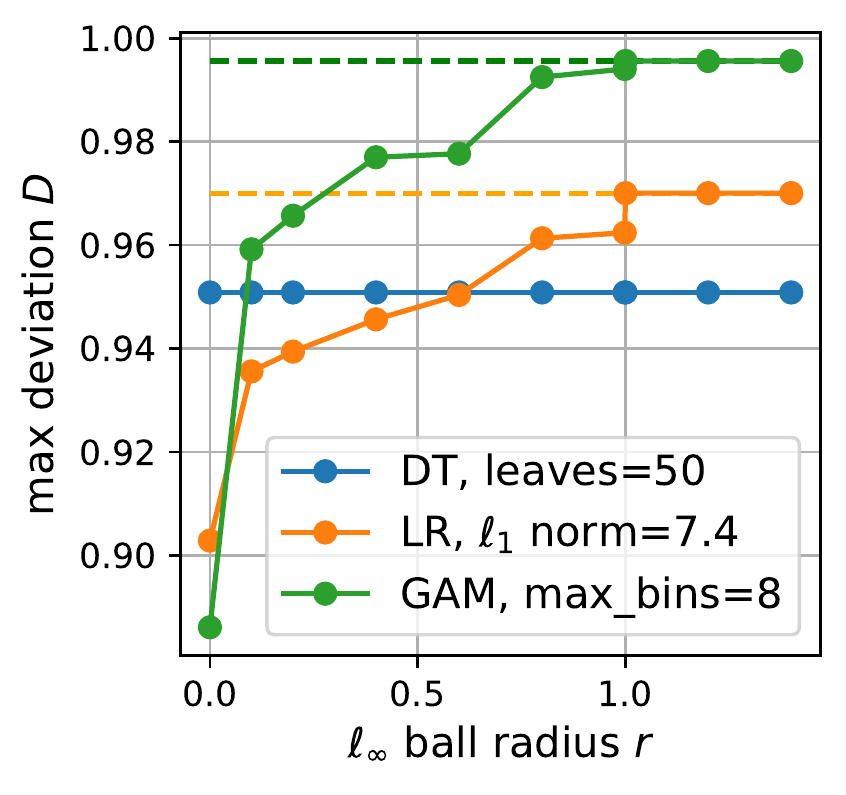}
        \caption{}
        \label{fig:Adult_DT_LR_GAM_r}
    \end{subfigure}
    \begin{subfigure}[b]{0.4\textwidth}
        \includegraphics[width=\textwidth]{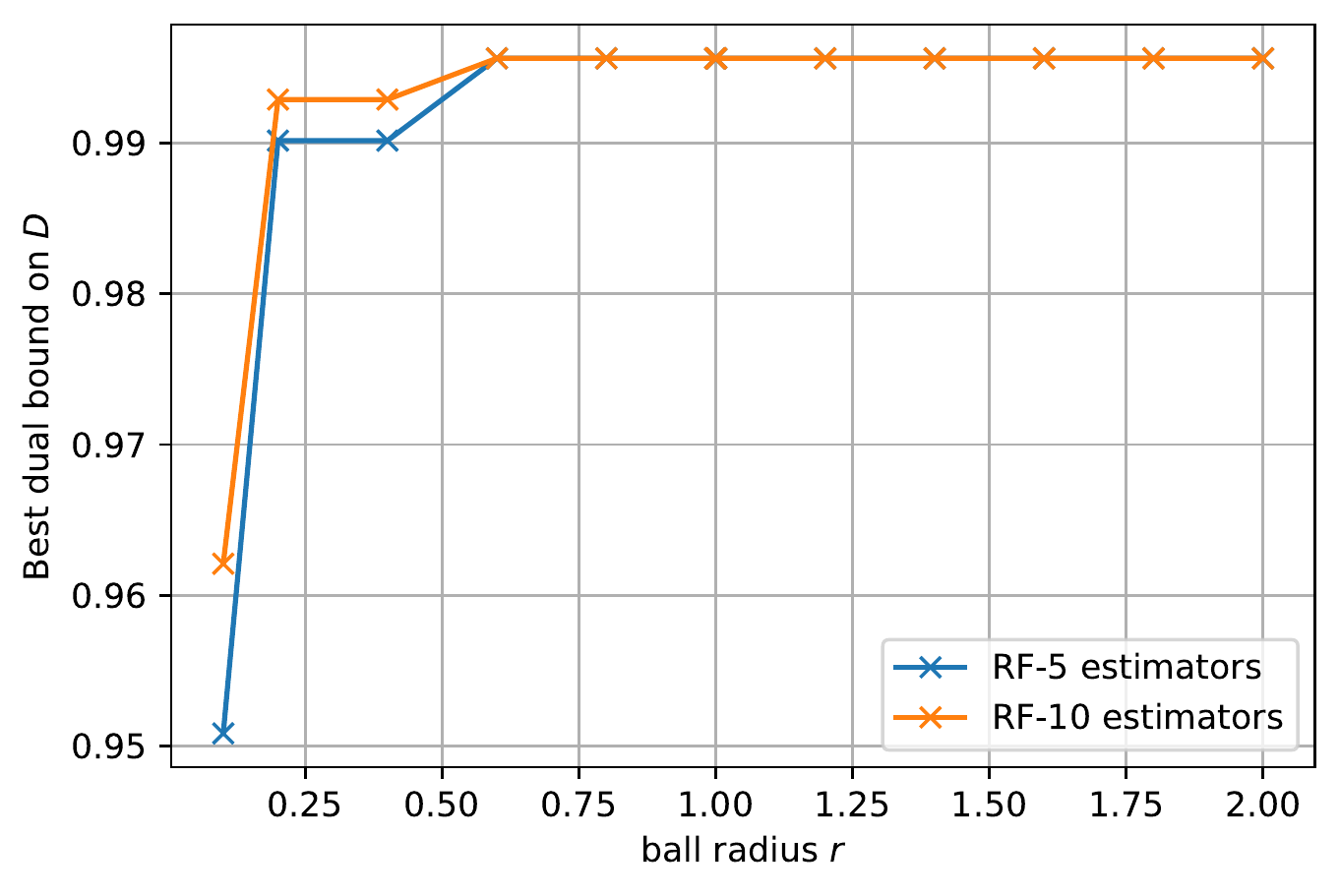}
        \caption{}
        \label{fig:adult:d_vs_Q}
    \end{subfigure}
    \caption{(Left) Maximum deviation $D$ on the Adult Income dataset as a function of certification set radius $r$ for DT, LR, and GAM. (Right) Upper bound on maximum deviation of $f$, a Random Forest, trained on the Adult Income dataset.}
\end{figure}

{
\paragraph{Relationships with accuracy and robust accuracy} In Figure~\ref{fig:Adult_acc}, we show maximum deviation as a function of test set accuracy for the DT, LR, and GAM models shown in  Figure~\ref{fig:Adult_all_add} (the LR and GAM models are listed in Tables~\ref{tab:Adult_LR_stats} and \ref{tab:Adult_GAM_stats}). Broadly, the plots show two regimes: one where accuracy increases and maximum deviation increases moderately or not at all, and one where accuracy stalls while maximum deviation increases. The latter is less desirable as it suggests increasing safety risks without a gain in accuracy. The last branch of the DT curve actually decreases in accuracy, indicating overfitting, while maximum deviation is high.

\begin{figure}[t]
    \centering
    \includegraphics[width=0.29\textwidth]{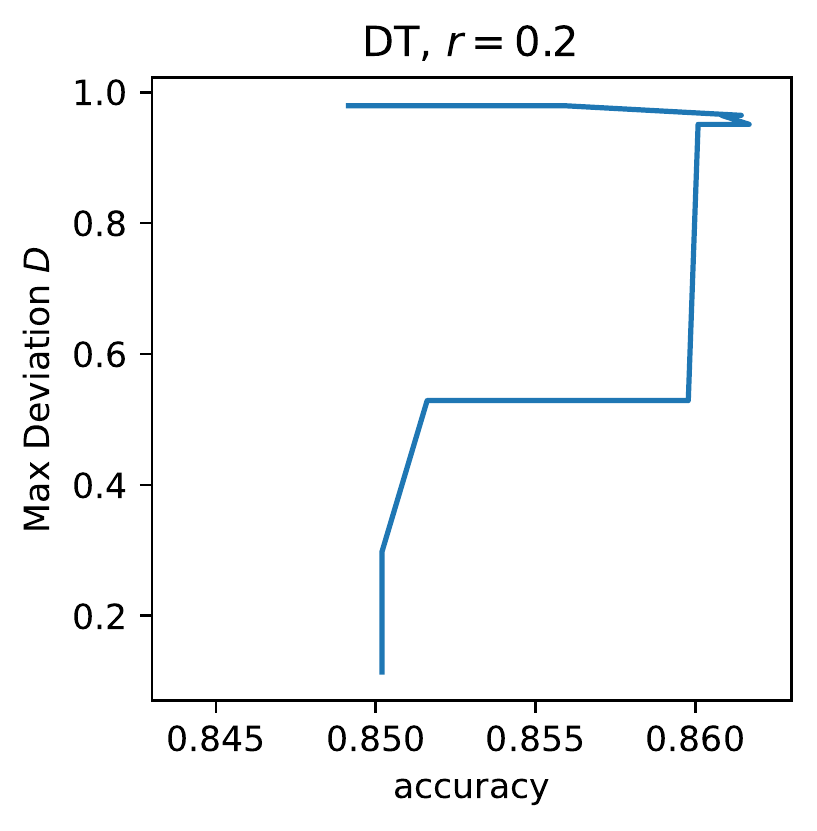}
    \includegraphics[width=0.30\textwidth]{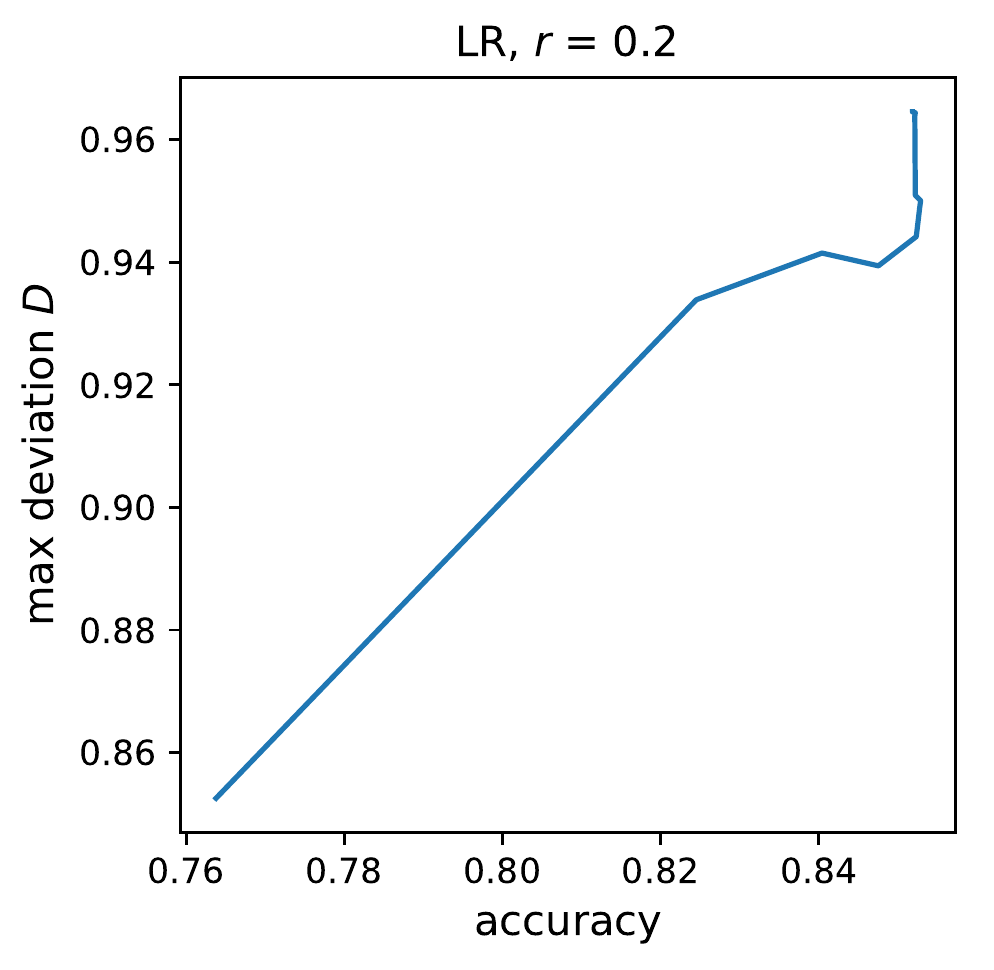}
    \includegraphics[width=0.30\textwidth]{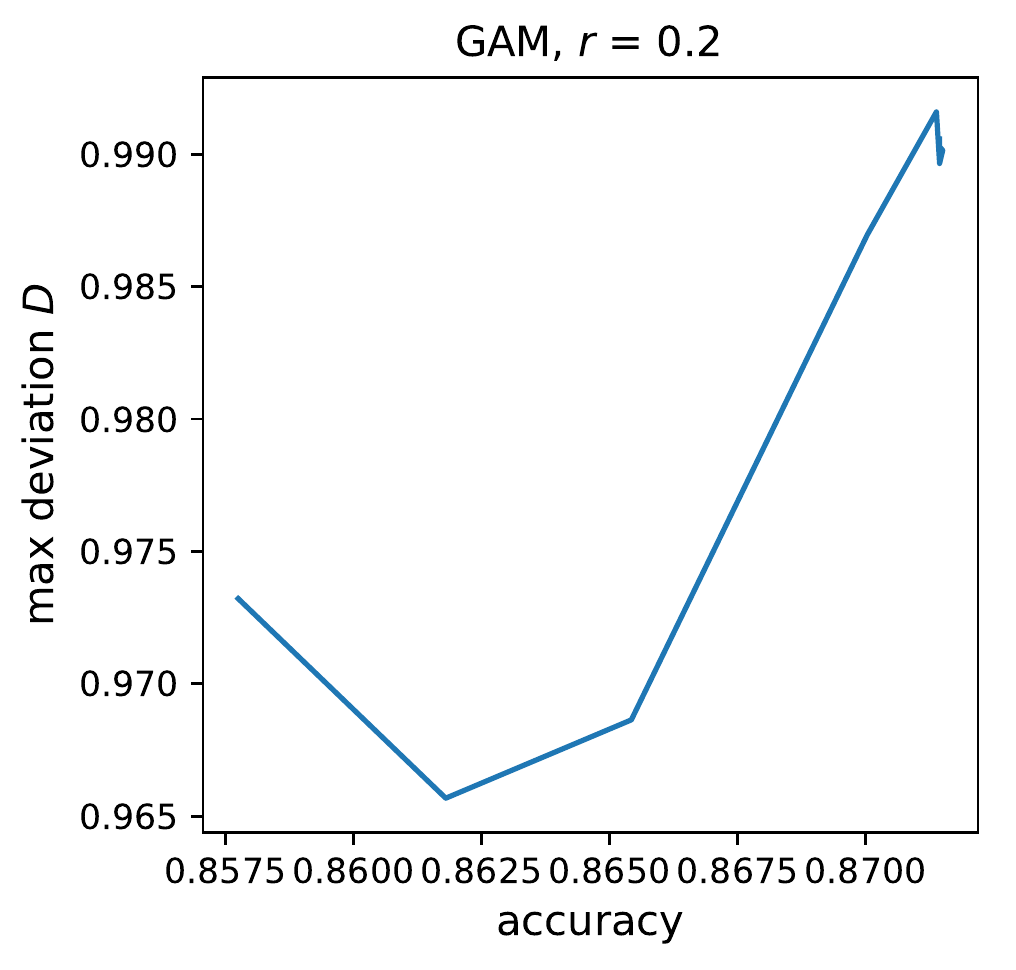}
    \caption{Maximum deviation $D$ (at certification set radius $r = 0.2$) vs.~test set accuracy on the Adult Income dataset.}
    \label{fig:Adult_acc}
\end{figure}

We also consider the relationship of maximum deviation to \emph{robust accuracy}. Following \citet{wong2018provable}, robust loss for a pair $(x, y)$ is defined as the worst-case loss over an $\ell_\infty$ ball centered at $x$, 
\begin{equation}\label{eqn:robustLoss}
\max_{\lVert\Delta\rVert_{\infty} \leq \epsilon} L(f(x + \Delta), y),
\end{equation}
and robust accuracy is therefore $1$ minus the average robust $0$-$1$ loss over a dataset. While \citet{wong2018provable} focus on bounding robust loss for feedforward neural networks with ReLU activations, we find that the results in Section~\ref{sec:model:linear_additive} apply to computing robust loss \eqref{eqn:robustLoss} exactly for LR and GAM models. Specifically, for $0$-$1$ loss and $\ell_\infty$ balls, the separable optimization \eqref{eqn:maxAdditiveSep} applies, and the worst case is obtained by minimizing $f$ when the label $y$ is positive and maximizing $f$ when $y$ is negative.

The resulting robust accuracy values for DT, LR and GAM are plotted in Figure~\ref{fig:Adult_robustAcc} in a similar fashion as Figure~\ref{fig:Adult_acc}. Here we set $\epsilon = 0.1$ and $r = 0.1$ as well in computing maximum deviation. The DT plot shows maximum deviation increasing with model complexity while robust accuracy is stable up to a point. In the subsequent regime, when there are a large number of leaves, model robustness reduces while deviation remains high.
The LR plot begins similarly to the one in Figure~\ref{fig:Adult_acc} in that robust accuracy increases along with maximum deviation, but then it stalls and decreases for maximum deviation above $0.94$. In the GAM plot, robust accuracy actually decreases with the \texttt{max\_bins} parameter, i.e., the curve goes from right to left.

\begin{figure}[t]
    \centering
    \includegraphics[width=0.29\textwidth]{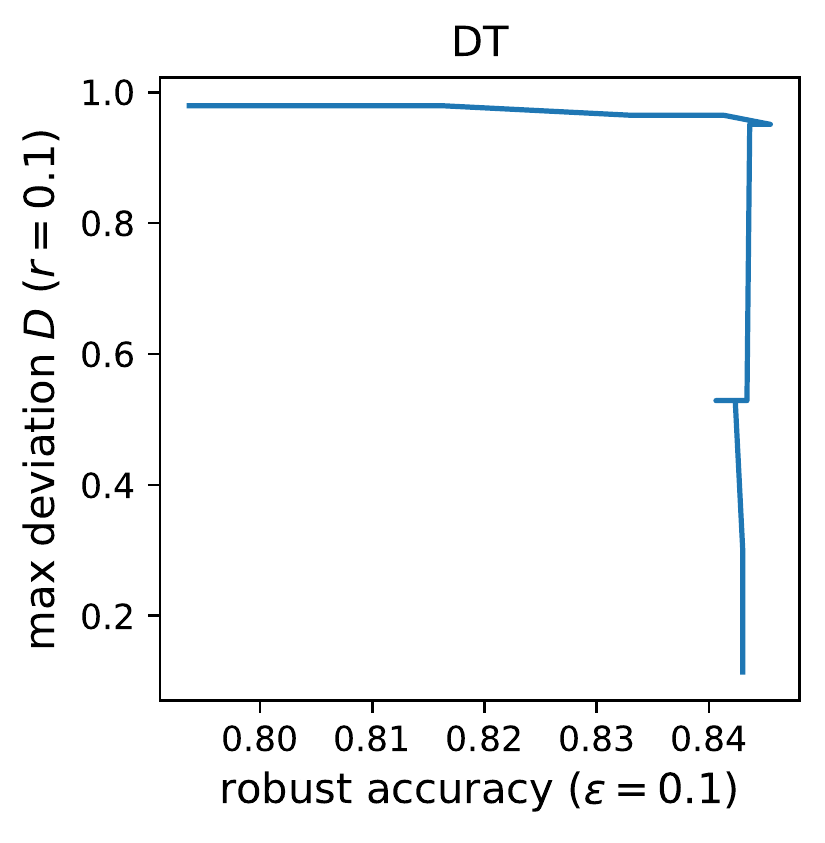}
    \includegraphics[width=0.30\textwidth]{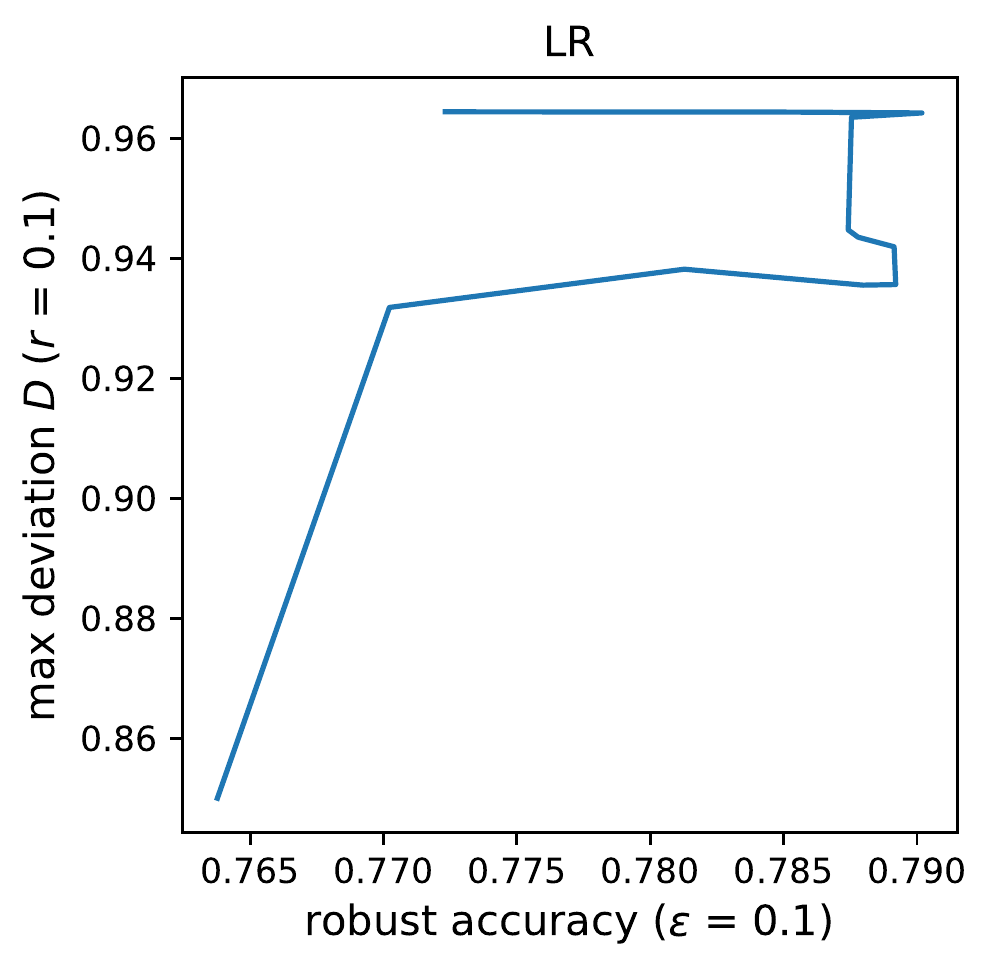}
    \includegraphics[width=0.30\textwidth]{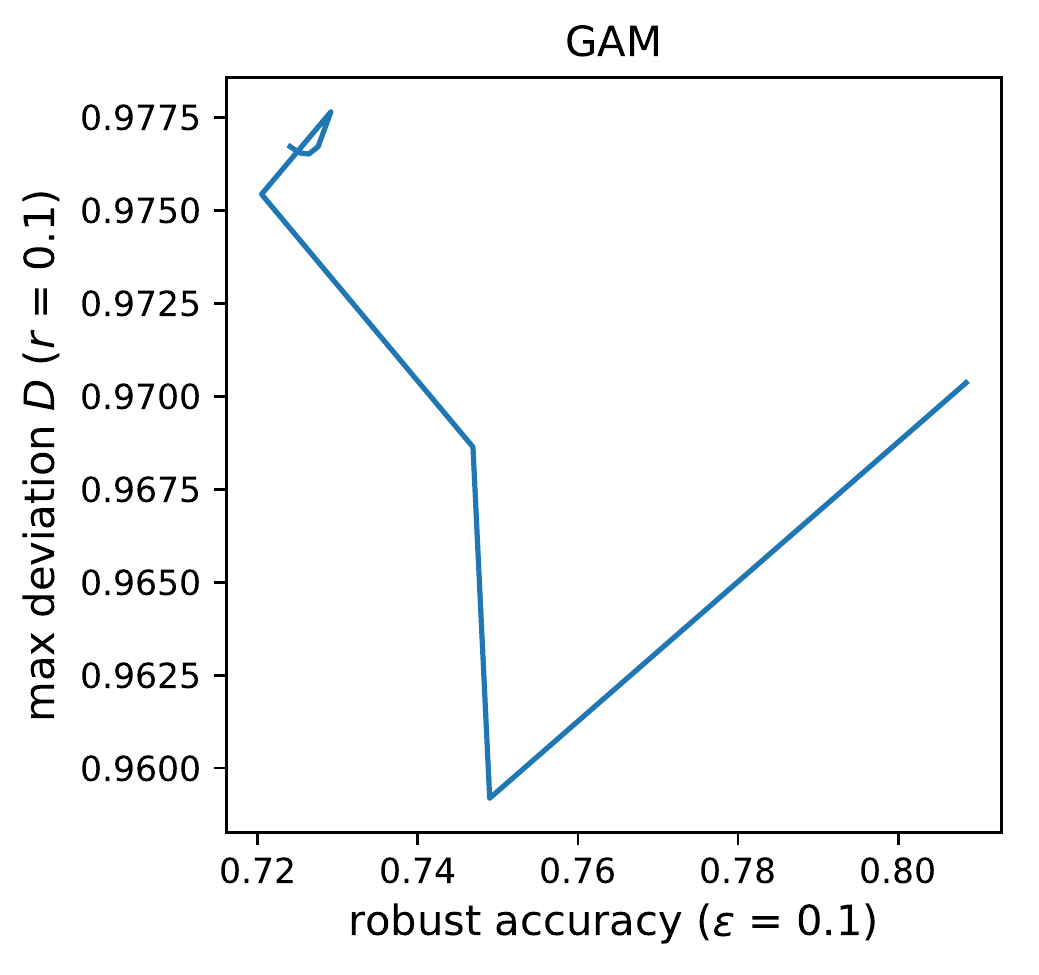}
    \caption{Maximum deviation $D$ vs.~robust accuracy ($\epsilon = 0.1$, $r = 0.1$) on the Adult Income dataset.}
    \label{fig:Adult_robustAcc}
\end{figure}}

\paragraph{Breakdown by leaves of $f_0$} In Figures~\ref{fig:Adult_leaves_LR_r}--\ref{fig:Adult_leaves_GAM_maxBins}, we plot the maximum log-odds achieved by model $f$ ($\max$ on RHS of \eqref{eqn:maxDevAdditiveTree}), the minimum log-odds achieved by $f$, and the reference model log-odds $g(y_{0m})$ over each leaf of the decision tree reference model in Figure~\ref{fig:Adult_DT}. Plots are on the log-odds scale to show the deviations more clearly, including those that would be compressed by the nonlinear logistic function $g^{-1}(z) = 1 / (1 + e^{-z})$. Figures~\ref{fig:Adult_leaves_LR_r} and \ref{fig:Adult_leaves_GAM_r} show the dependence on the certification set radius $r$ while Figures~\ref{fig:Adult_leaves_LR_C} and \ref{fig:Adult_leaves_GAM_maxBins} show dependence on the smoothness parameters for LR and GAM. These figures provide a more granular picture corresponding to the summary in 
Figure~\ref{fig:Adult_all_add} and support the trends seen there. In Figures~\ref{fig:Adult_leaves_LR_r} and \ref{fig:Adult_leaves_GAM_r}, there are jumps at $r = 1$ because this is the smallest radius that permits the values of categorical features of test set points (the ball centers in \eqref{eqn:unionBalls}) to change to any other value. (Recall that categorical features are one-hot encoded into binary-valued features.) In Figure~\ref{fig:Adult_leaves_GAM_r}, the GAM achieves the limiting deviations corresponding to $r = \infty$ ($\certset = \mathcal{X}$, dashed lines) no later than $r = 1.2$. In Figure~\ref{fig:Adult_leaves_LR_r}, the LR model achieves the lower limit on log-odds as soon as $r > 1$ but the upper limit is not achieved for most leaves.

\begin{figure}[ht]
    \centering
    \includegraphics[width=\textwidth]{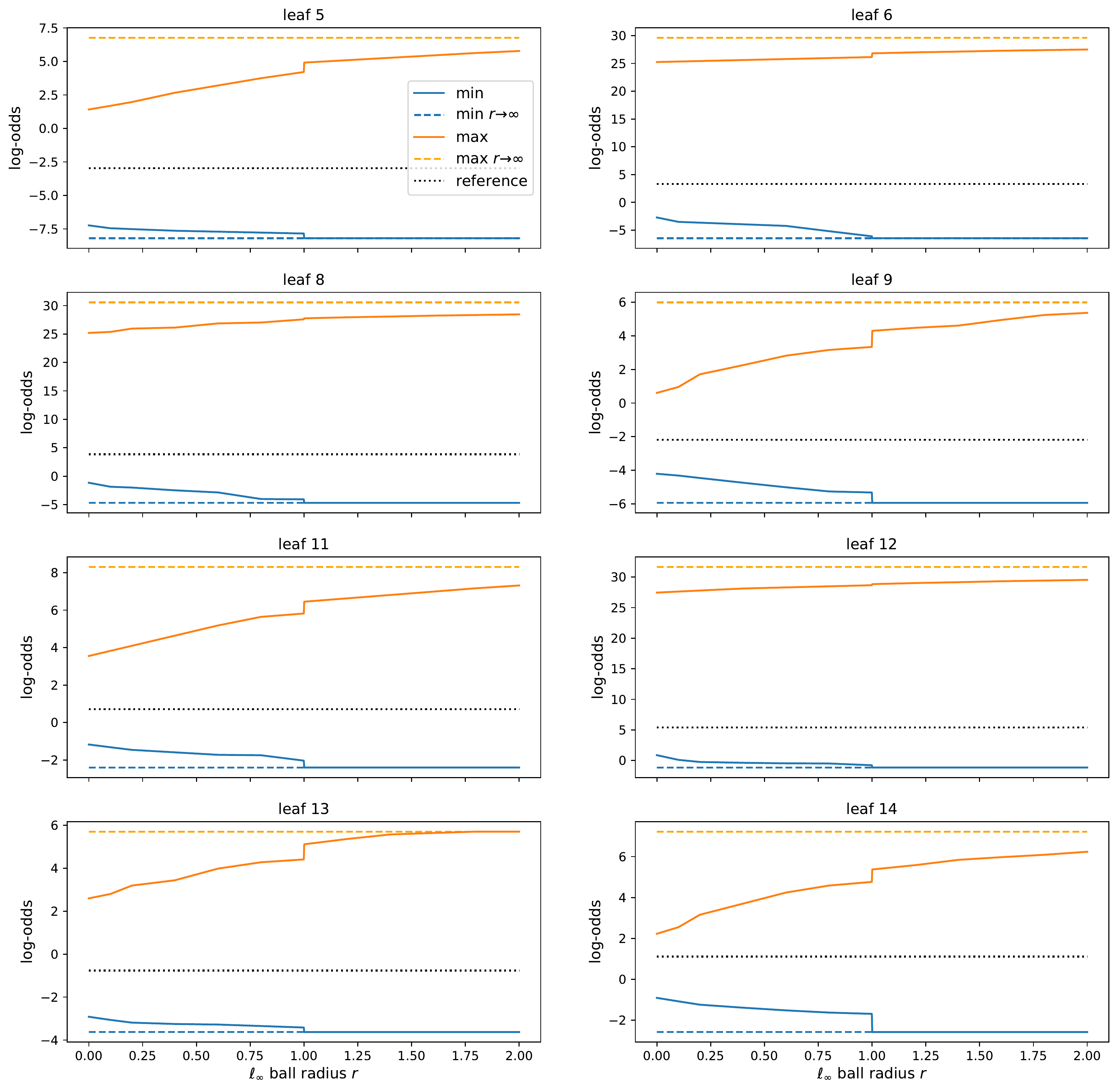}
    \caption{Minimum and maximum predicted log-odds for a logistic regression model with inverse $\ell_1$ penalty $C = 0.01$, as a function of certification set size (radius $r$) and broken down by leaves of the decision tree reference model.}
    \label{fig:Adult_leaves_LR_r}
\end{figure}

\begin{figure}[t]
    \centering
    \includegraphics[width=\textwidth]{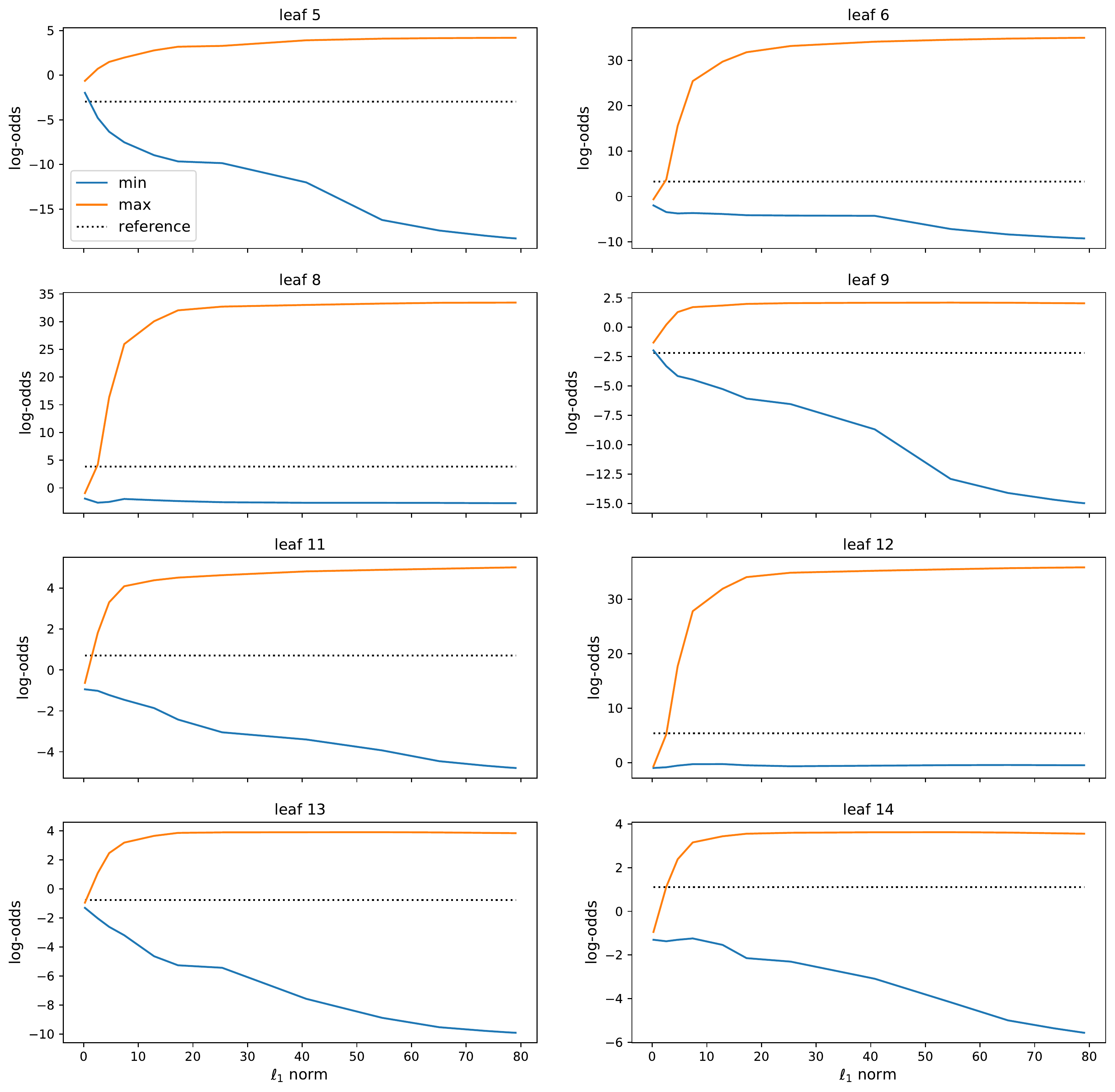}
    \caption{Minimum and maximum predicted log-odds for logistic regression models with different $\ell_1$ penalties $C$, broken down by leaves of the decision tree reference model. The certification set $\ell_\infty$ ball radius is $r = 0.2$.}
    \label{fig:Adult_leaves_LR_C}
\end{figure}

\begin{figure}[t]
    \centering
    \includegraphics[width=\textwidth]{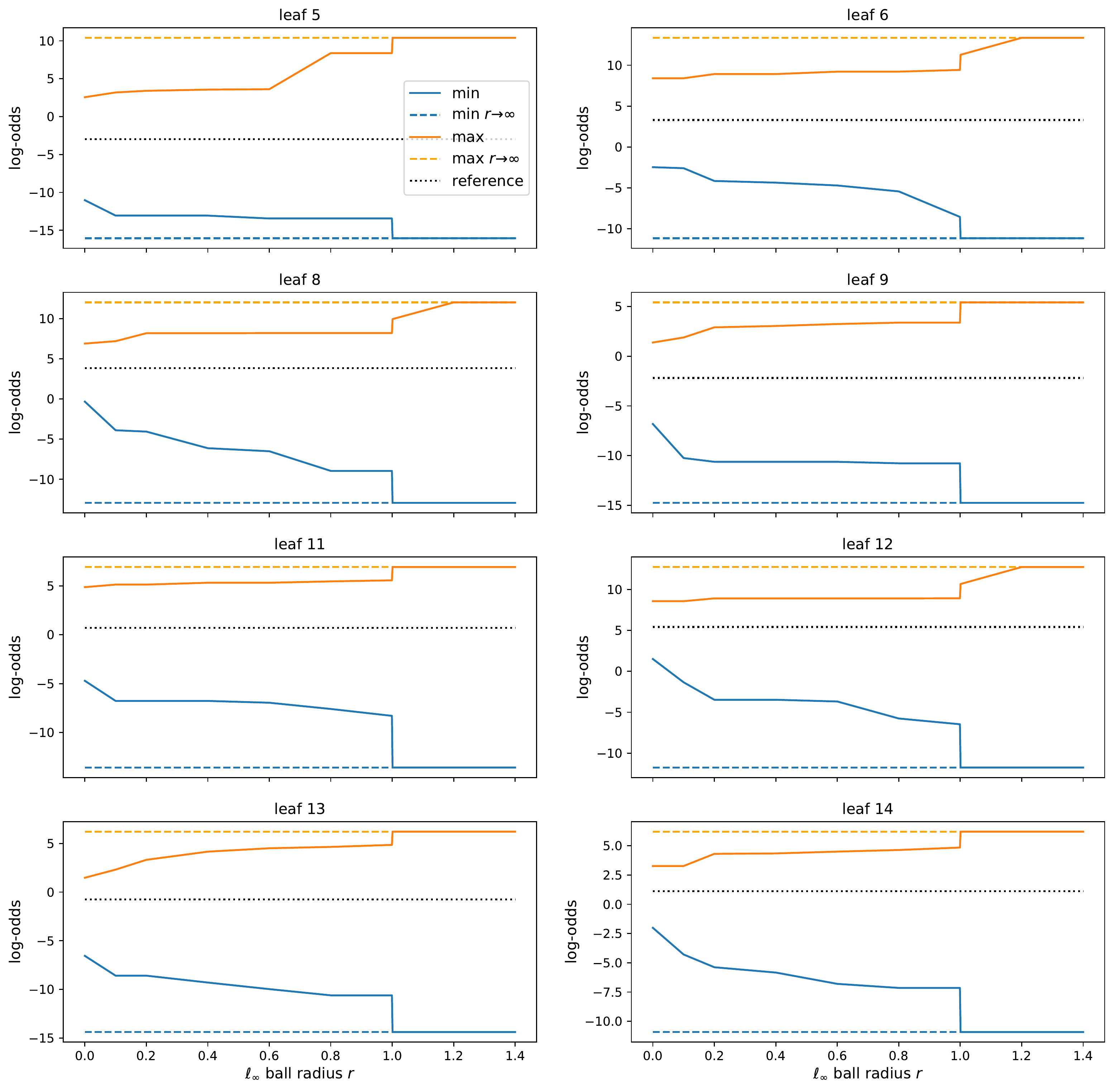}
    \caption{Minimum and maximum predicted log-odds for an Explainable Boosting Machine with \texttt{max\_bins} $= 8$, as a function of certification set size (radius $r$) and broken down by leaves of the decision tree reference model.}
    \label{fig:Adult_leaves_GAM_r}
\end{figure}

\begin{figure}[t]
    \centering
    \includegraphics[width=\textwidth]{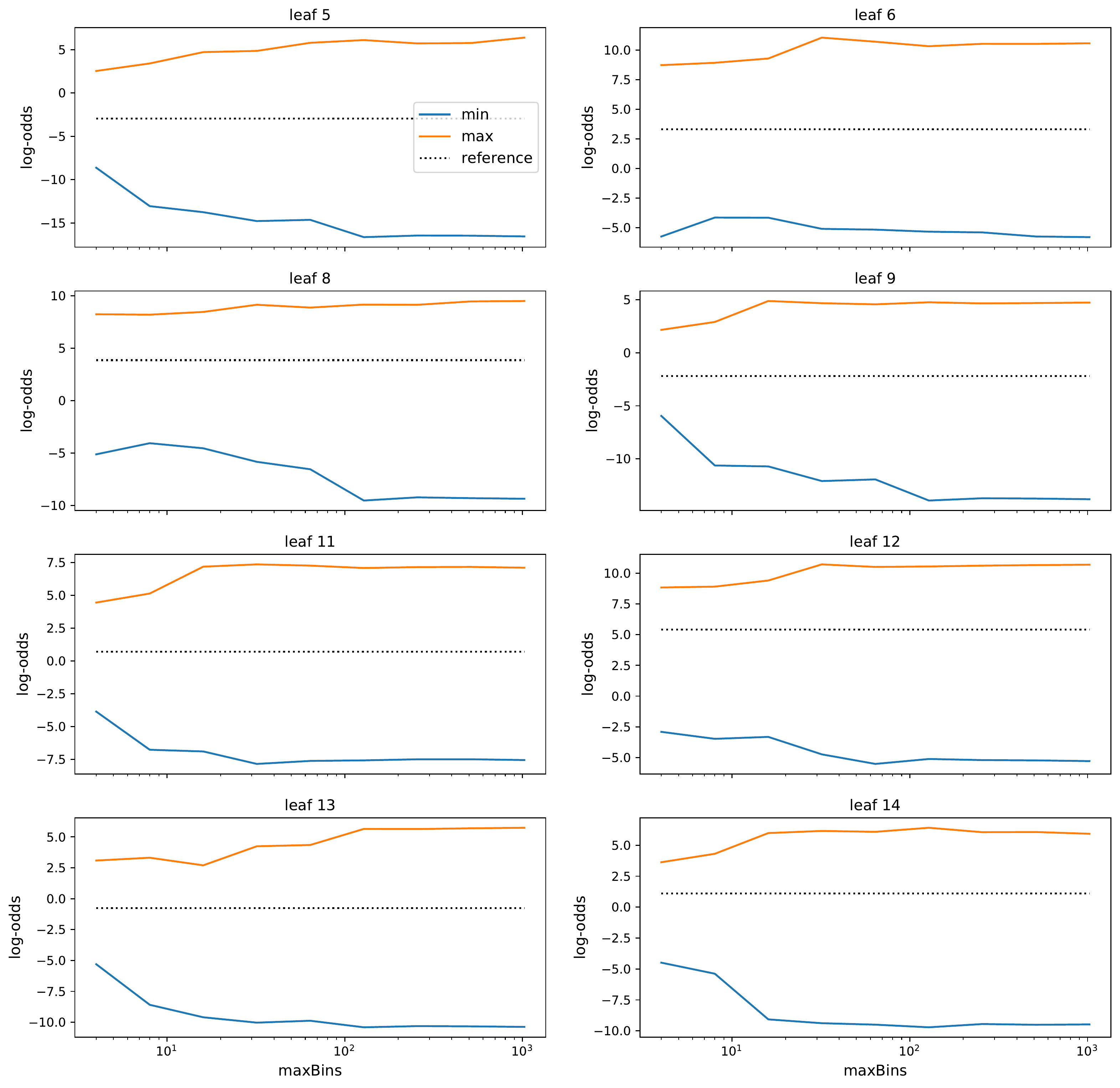}
    \caption{Minimum and maximum predicted log-odds for Explainable Boosting Machines with different \texttt{max\_bins} values, broken down by leaves of the decision tree reference model. The certification set $\ell_\infty$ ball radius is $r = 0.2$.}
    \label{fig:Adult_leaves_GAM_maxBins}
\end{figure}

\clearpage

\begin{figure}[ht]
    \centering
    \includegraphics[width=0.495\textwidth]{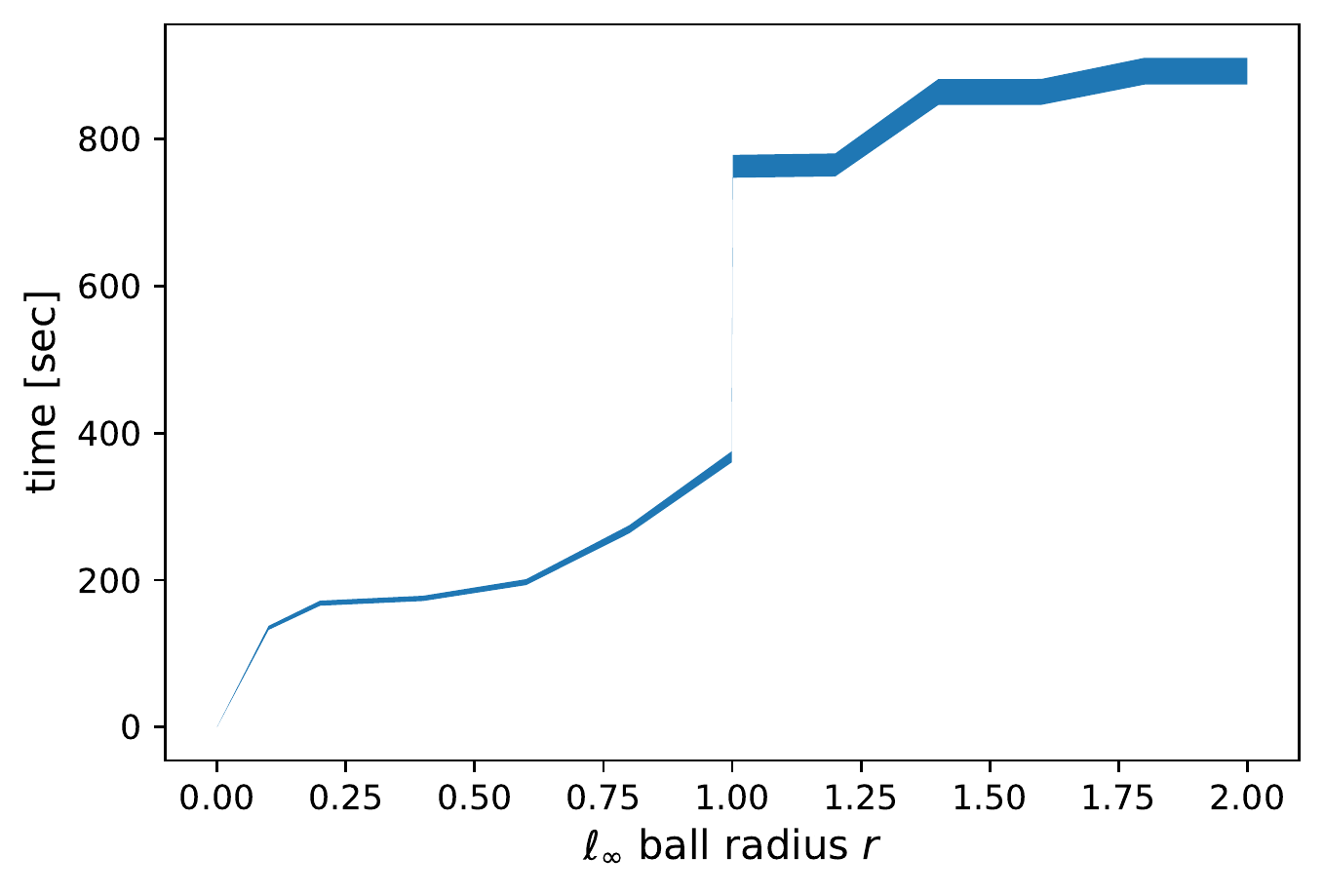}
    \includegraphics[width=0.495\textwidth]{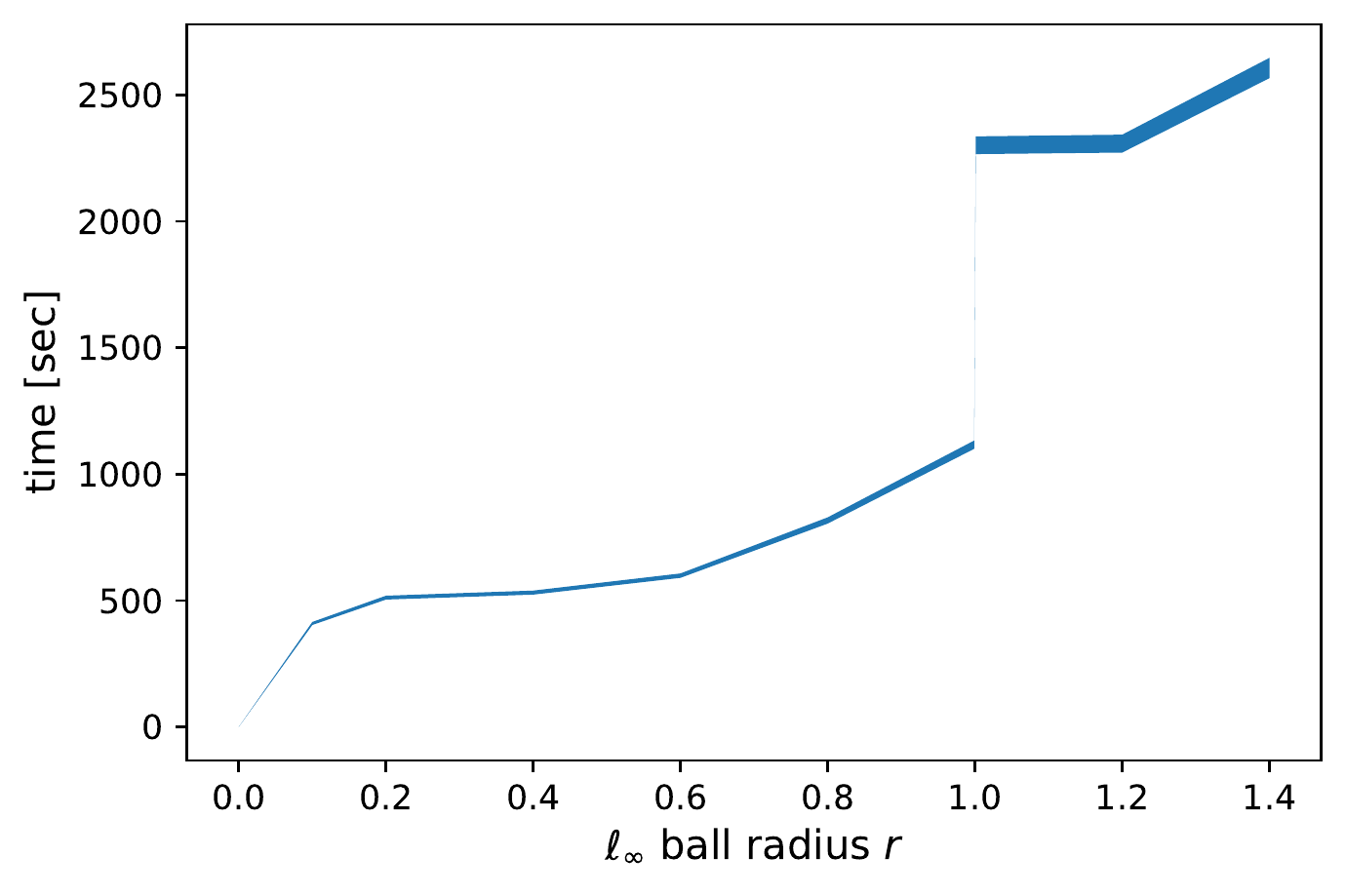}
    \caption{Time to compute maximum deviation for logistic regression models (left) and Explainable Boosting Machines (right) on the Adult Income dataset as a function of certification set size (radius $r$). The filled-in region shows the min-max variation with model complexity ($\ell_1$ norm for LR, \texttt{max\_bins} for EBM).}
    \label{fig:Adult_time}
\end{figure}

\paragraph{Running time} Figure~\ref{fig:Adult_time} shows the time required to compute the maximum deviation for LR and GAM on the Adult Income dataset. The same observations apply as in Figure~\ref{fig:HMDA_time} earlier.

\paragraph{Primal bounds for RF} The primal bound for max-deviation shown in Figure \ref{fig:Adult_RF_Nest} for RF is updated each time the algorithm finds a $K+1$ partite, i.e. has examined all trees in the Random Forest. The max-deviation computed for such a partite is a valid deviation. To prove if its optimal, Algorithm \ref{alg:tree-ensembles:enum} needs to run to completion which may not be feasible. Figure \ref{fig:Adult_RF_Nest} shows that as the RF models get larger (number of partites increase), it gets harder to find primal solutions. 

\paragraph{Maximal cliques evaluated for DT, RF} To investigate the effectiveness of pruning by bounds in Algorithm \ref{alg:tree-ensembles:enum}, we investigate the number of times all the decision trees in the Random Forest have to be processed. This represents number of times the state could not be pruned and needed to be evaluated fully. 

Figure \ref{fig:expt:pruning} shows two aspects at play. (a) Pruning by bound is effective in restricting the search space more so for Random Forests than for decision trees, and (b) for larger graphs, more time is spent in computing bounds in Eq. \eqref{eq:tree-ensemble:nodesel}. 

\begin{figure}[ht]
     \centering
     \begin{subfigure}[b]{0.4\textwidth}
         \centering
         \includegraphics[width=\textwidth]{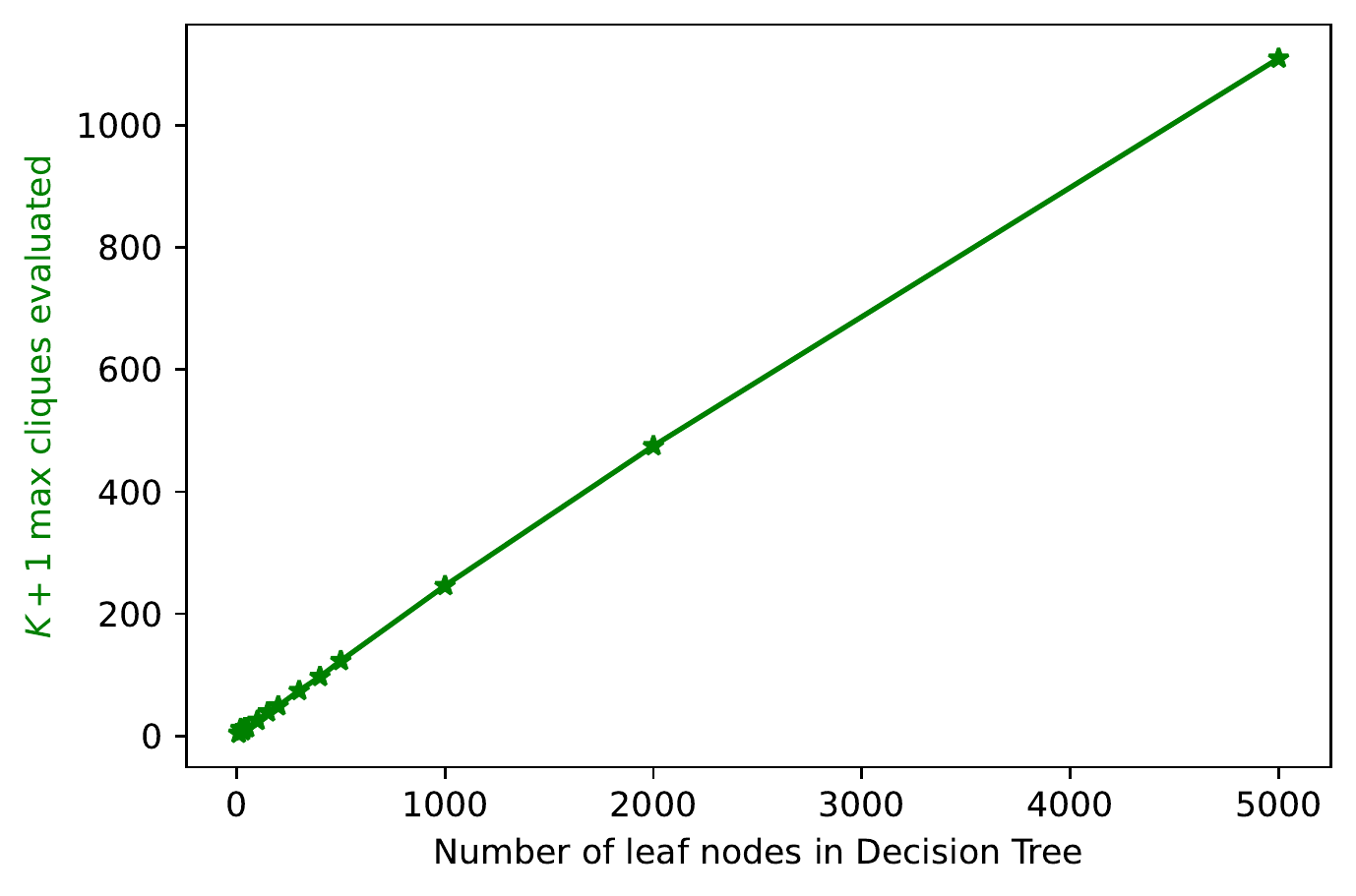}
         \caption{Decision Tree}
         \label{fig:expt:dt:size}
     \end{subfigure}
     \begin{subfigure}[b]{0.45\textwidth}
         \centering
         \includegraphics[width=\textwidth]{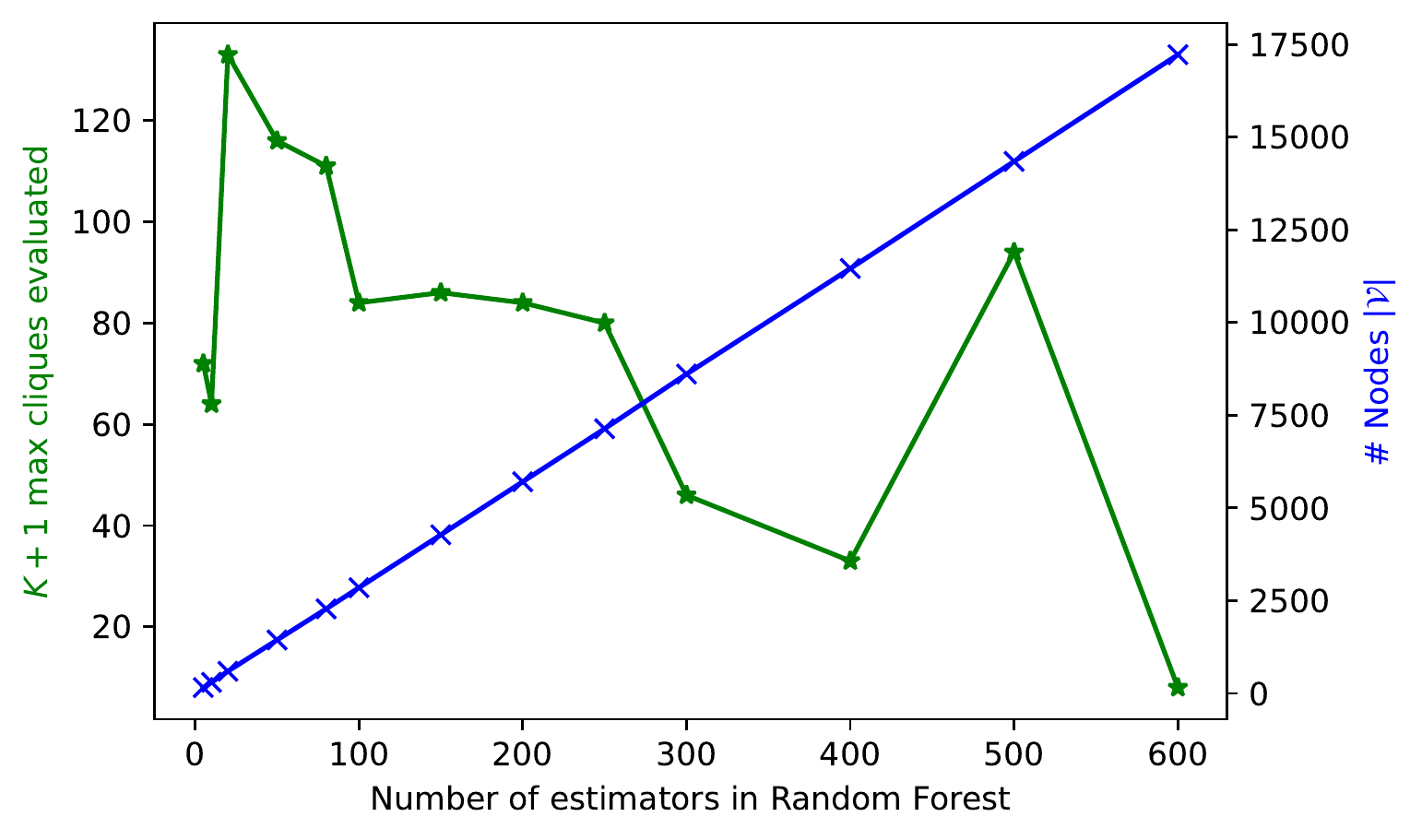}
         \caption{Random Forest}
         \label{fig:expt:ensemble:size}
     \end{subfigure}
        \caption{Effectiveness of pruning by bound for tree-based models}
        \label{fig:expt:pruning}
\end{figure}

\paragraph{Feature combinations that maximize deviation} In Tables~\ref{tab:Adult_LR} and \ref{tab:Adult_GAM}, we report feature values that maximize deviation over selected leaves of the DT reference model, for LR and GAM respectively. For Table~\ref{tab:Adult_LR}, we have chosen the minimum log-odds over leaf 6 (corresponding to Figure~\ref{fig:Adult_leaves_LR_r}, leaf 6, blue curve), which is one of the two leaves $m$ in \eqref{eqn:maxDevAdditiveTree} that maximize the deviation overall (the other being leaf 8). For Table~\ref{tab:Adult_GAM}, the minimum log-odds over leaf 12 is chosen (corresponding to Figure~\ref{fig:Adult_leaves_GAM_r}, leaf 12, blue curve) because this maximizes the deviation overall for most values of $r$. The tables show the 6 features that contribute most to the minimum log-odds. These contributions are again determined using \eqref{eqn:maxAdditiveSep} (with $\max$ replaced by $\min$); since the minimum log-odds occurs in one of the $\ell_\infty$ ball-leaf intersections and this intersection is a Cartesian product, the decomposition in \eqref{eqn:maxAdditiveSep} applies. The contribution of feature $j$ is then $\min_{x_j \in \mathcal{S}_j} f_j(x_j)$. As in Tables~\ref{tab:HMDA_GAM_min} and \ref{tab:HMDA_GAM_max}, we take an average of the contributions over $r$ to give a single ranking of features for all $r$. 

\begin{table}[ht]
    \small
    \centering
    \begin{tabular}{lrrrrrr}
    \toprule
    $r$&EducationNum&HoursPerWeek&Age&MaritalStatus&Occupation&Relationship\\
    \midrule
    0.0&9.0&15.0&23.0& Never-married& Sales& Own-child\\
    0.1&4.7&33.8&26.6& Never-married& Transport-moving& Not-in-family\\
    0.2&4.5&32.5&25.3& Never-married& Transport-moving& Not-in-family\\
    0.4&4.0&30.1&22.5& Never-married& Transport-moving& Not-in-family\\
    0.6&3.5&27.6&19.8& Never-married& Transport-moving& Not-in-family\\
    0.8&1.0&26.1&17.0& Never-married& Farming-fishing& Not-in-family\\
    0.999&1.0&7.7&17.0& Never-married& Other-service& Own-child\\
    1.001&1.0&1.0&17.0& Never-married& Other-service& Own-child\\
    1.2&1.0&1.0&17.0& Never-married& Other-service& Own-child\\
    1.4&1.0&1.0&17.0& Never-married& Other-service& Own-child\\
    1.6&1.0&1.0&17.0& Never-married& Other-service& Own-child\\
    1.8&1.0&1.0&17.0& Never-married& Other-service& Own-child\\
    2.0&1.0&1.0&17.0& Never-married& Other-service& Own-child\\
    $\infty$&1.0&1.0&17.0& Never-married& Other-service& Own-child\\
    \bottomrule
    \end{tabular}
    \caption{Feature values that minimize log-odds for a logistic regression model ($C = 0.01$) over leaf 6 of the decision tree reference model. The 6 features that contribute most to the minimum are shown as a function of certification set radius $r$.}
    \label{tab:Adult_LR}
\end{table}

As $r$ increases, the predominant trend of the values of continuous features is toward extremes of the domain $\mathcal{X}$, depending on the sign of the corresponding LR coefficient $w_j$ or shape of the GAM function $f_j$. For example, EducationNum (education on an ordinal scale), hours per week, and age decrease toward minimum values, while capital gain occupies the minimal interval permitted for leaf 12 (see Figure~\ref{fig:Adult_DT}). (These examples make sense since the log-odds of high income is being minimized.) This movement toward extremes is expected in the LR case because the functions $w_j x_j$ are either increasing or decreasing, and it is also true for GAM if the function $f_j$ is mainly increasing or decreasing. The values sometimes change abruptly in the opposite direction, for example hours per week in both Tables~\ref{tab:Adult_LR}, \ref{tab:Adult_GAM}, and age in the latter. These abrupt changes are due to the minimum jumping from one ball in \eqref{eqn:unionBalls} to another as $r$ increases, but the overall trend eventually prevails. For categorical features, the trend is toward values that minimize $f_j(x_j)$, e.g., \textit{Never-married} marital status, \textit{Without-pay} work class. While the contribution of each of these features to minimizing log-odds may be limited, together they do add up.

\begin{table}[ht]
    \small
    \centering
    \begin{tabular}{lrrrrrr}
    \toprule
    $r$&CapitalLoss&Age&HoursPerWeek&WorkClass&CapitalGain&Race\\
    \midrule
    0.0&0&29.0&40.0& Private&7298& White\\
    0.1&[ 0 40]&[53.6 56.4]&[18.8 20.5]& ?&[5095 5119]& White\\
    0.2&[ 0 78]&[23.3 23.5]&[4.5 9.5]& State-gov&[5095 5119]& White\\
    0.4&[ 0 78]&[20.5 23.5]&[ 2.1 11.9]& State-gov&[5095 5119]& White\\
    0.6&[ 0 78]&[28.8 29.5]&[30.6 31.5]& Private&[5095 5119]& Asian-Pac-Islander\\
    0.8&[1598 1759]&[20.1 23.5]&[30.1 31.5]& State-gov&[5095 5119]& White\\
    0.999&[1598 1759]&[21.4 23.5]&[27.7 31.5]& Private&[5095 5119]& Amer-Indian-Eskimo\\
    1.001&[1598 1759]&[17.  23.5]&[11.6 20.5]& Without-pay&[5095 5119]& Amer-Indian-Eskimo\\
    1.2&[1598 1759]&[17.  23.5]&[ 9.2 20.5]& Without-pay&[5095 5119]& Amer-Indian-Eskimo\\
    1.4&[1598 1759]&[17.  23.5]&[ 6.7 20.5]& Without-pay&[5095 5119]& Amer-Indian-Eskimo\\
    $\infty$&[1598 1759]&[17.  23.5]&[ 1.  20.5]& Without-pay&[5095 5119]& Amer-Indian-Eskimo\\
    \bottomrule
    \end{tabular}
    \caption{Feature values that minimize log-odds for an Explainable Boosting Machine (\texttt{max\_bins} $= 8$) over leaf 12 of the decision tree reference model. The 6 features that contribute most to the minimum are shown as a function of certification set radius $r$. For $r > 0$, the minimizing values of continuous features form an interval because the corresponding functions $f_j$ are piecewise constant.}
    \label{tab:Adult_GAM}
\end{table}

A notable exception to the trend toward extremes is capital loss in Table~\ref{tab:Adult_GAM}. This was discussed in the ``Identification of an artifact'' example in Section~\ref{sec:expt}.

Given the results in Tables~\ref{tab:Adult_LR} and \ref{tab:Adult_GAM}, one question that arises is whether the feature combinations are indeed possible, if not the ones for $r \to \infty$, then at least for some finite value of $r$. For the top features shown in the two tables, while some combinations may appear improbable (for example, EducationNum $= 1$ and $1$ hour per week), we submit that none appear \emph{impossible}. However, if one considers features beyond the top 6, then some ``impossible'' combinations do occur (e.g., a female husband), although the contributions of these features to the minimum log-odds are much less. We touch upon this issue in Appendix~\ref{sec:discuss}. 

The next question one might consider is the implication of these maximal deviations. From Figure~\ref{fig:Adult_DT}, it is seen that leaf 6 classifies individuals with high capital gains as high income with high probability ($0.965$). Leaf 12 adds the attributes of married status and high education, and hence classifies as high income with even higher probability ($0.996$). At the same time, the feature values in Tables~\ref{tab:Adult_LR} and \ref{tab:Adult_GAM}, which minimize log-odds for LR and GAM, also make sense according to basic domain knowledge. For example, few hours per week and young age are associated with lower income, as are \textit{Without-pay} work class and \textit{Amer-Indian-Eskimo} race in the United States. When these conflicting associations occur in combination and the combination does not appear impossible, the question may be which one prevails. Such a question might be resolvable by a domain expert. Alternatively, the disagreement between models $f$ and $f_0$ on the extreme examples in Tables~\ref{tab:Adult_LR}, \ref{tab:Adult_GAM} may be reason to be cautious about using either of the models in these cases. This might lead to a way of combining the models or abstaining from prediction altogether. Lastly, the anomalously low region in the CapitalLoss function identified in Table~\ref{tab:Adult_GAM} is a clear, concrete example where further investigation is warranted.

\clearpage

\subsection{Lending Club dataset}
\label{sec:expt_add:Lending}

This dataset consists of 2.26 million rows with 14 features on loans. The target variable is whether a loan will be paid-off or defaulted on. Features describe the terms of the loan, e.g. duration, grade, purpose, etc. and borrower financial information such as credit history and income. For this case study, we consider a loan approval scenario using only information available at the time of application. In particular, we exclude the feature `total\_pymnt' (total payment over time on the loan), which becomes known at essentially the same time as the target variable. (When `total\_pymnt' is included as a feature, the prediction task becomes easy and accuracies in the high $90\%$ range are possible.)

\begin{figure}[ht]
    \centering
    \includegraphics[width=\textwidth]{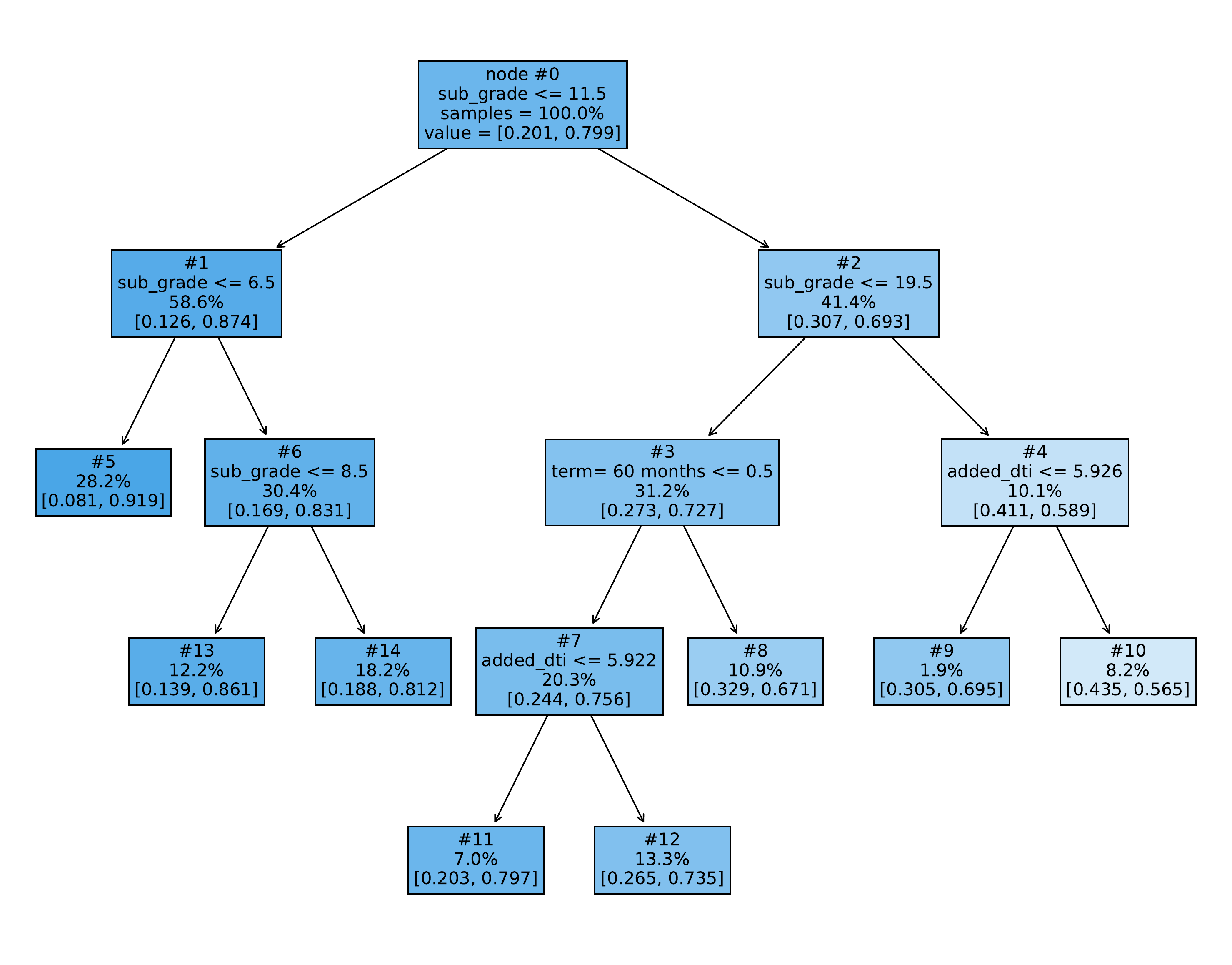}
    \caption{Decision tree reference model with 8 leaves for the Lending Club dataset.}
    \label{fig:Lending_DT}
\end{figure}

\paragraph{Reference model} Figure~\ref{fig:Lending_DT} depicts the $8$-leaf DT reference model for the Lending Club dataset. Most of the splits partition the `sub\_grade' feature, which is a measure of the quality of the loan ($0$--$34$ range, lower is better). Node 3 differentiates between 60-month terms and 36-month terms (the only alternative), while nodes 4 and 7 split on `added\_dti' (added debt-to-income ratio), which is the ratio between 12 months worth of the loan's payment installments and the borrower's annual income. While the structure of the DT agrees with domain knowledge (lower `sub\_grade' and lower `added\_dti' correlate with higher repayment probability), the test set accuracy of $79.8\%$ is no better than that of the trivial predictor that always returns the majority class of ``paid off''. The DT's AUC of $0.689$ however does indicate an improvement over the trivial predictor.

\begin{table}[ht]
    \small
    \centering
    \begin{tabular}{rrrrr}
    \toprule
    $C$&nonzeros&$\ell_1$ norm&accuracy&AUC\\
    \midrule
    1e-4&1&0.4&0.798&0.693\\
    3e-4&2&0.6&0.799&0.695\\
    1e-3&6&1.0&0.799&0.702\\
    3e-3&8&1.3&0.799&0.702\\
    1e-2&12&1.6&0.800&0.702\\
    3e-2&17&2.1&0.799&0.703\\
    1e-1&22&3.1&0.799&0.703\\
    3e-1&24&3.7&0.799&0.703\\
    1e+0&26&4.1&0.799&0.703\\
    3e+0&26&4.2&0.799&0.703\\
    \bottomrule
    \end{tabular}
    \caption{Number of nonzero coefficients, $\ell_1$ norm of coefficients, test set accuracy, and AUC for logistic regression models on the Lending Club dataset as a function of inverse $\ell_1$ penalty $C$.}
    \label{tab:Lending_LR_stats}
\end{table}

\begin{table}[ht]
    \small
    \centering
    \begin{tabular}{rrr}
    \toprule
    max\_bins&accuracy&AUC\\
    \midrule
    4&0.798&0.696\\
    8&0.798&0.703\\
    16&0.799&0.704\\
    32&0.799&0.705\\
    64&0.799&0.705\\
    128&0.799&0.705\\
    256&0.799&0.705\\
    512&0.799&0.705\\
    1024&0.799&0.705\\
    \bottomrule
    \end{tabular}
    \caption{Test set accuracy and AUC for Explainable Boosting Machines on the Lending Club dataset as a function of \texttt{max\_bins} parameter.}
    \label{tab:Lending_GAM_stats}
\end{table}

\paragraph{LR and GAM models} Tables~\ref{tab:Lending_LR_stats} and \ref{tab:Lending_GAM_stats} show the statistics of the LR and GAM classifiers that were trained on the Lending Club data. Similar to the DT reference model, the difference compared to Tables~\ref{tab:Adult_LR_stats}, \ref{tab:Adult_GAM_stats} for the Adult Income dataset is that the accuracies remain no better than that of the trivial predictor, while the AUC does not show much increase either. These statistics suggest that the prediction task is difficult with the features available. Figures~\ref{fig:Lending_LR_coef} and \ref{fig:Lending_GAM_functions} display plots for the LR model with $C = 0.01$ and GAM with \texttt{max\_bins} $= 8$, which are again chosen as representative models. The GAM in particular shows sensible monotonic behavior as functions of `sub\_grade', `int\_rate' (interest rate), `dti' (debt-to-income ratio), etc., despite the unimpressive accuracy.

\begin{figure}[ht]
    \centering
    \includegraphics[width=0.62\textwidth]{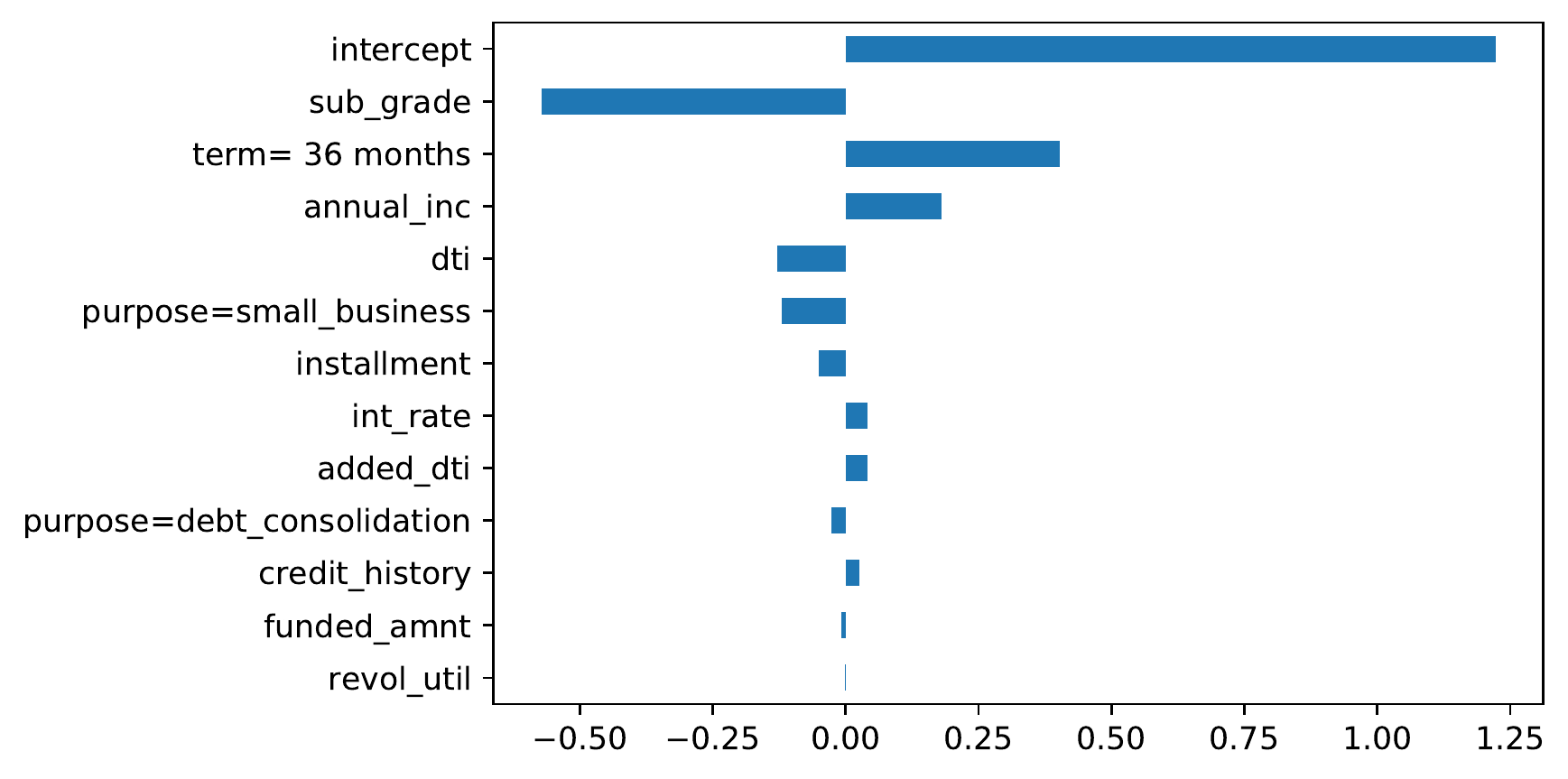}
    \caption{Coefficient values of the logistic regression model with $C = 0.01$ (12 nonzeros) for the Lending Club dataset.}
    \label{fig:Lending_LR_coef}
\end{figure}

\begin{figure}
    \centering
    \includegraphics[width=0.495\textwidth]{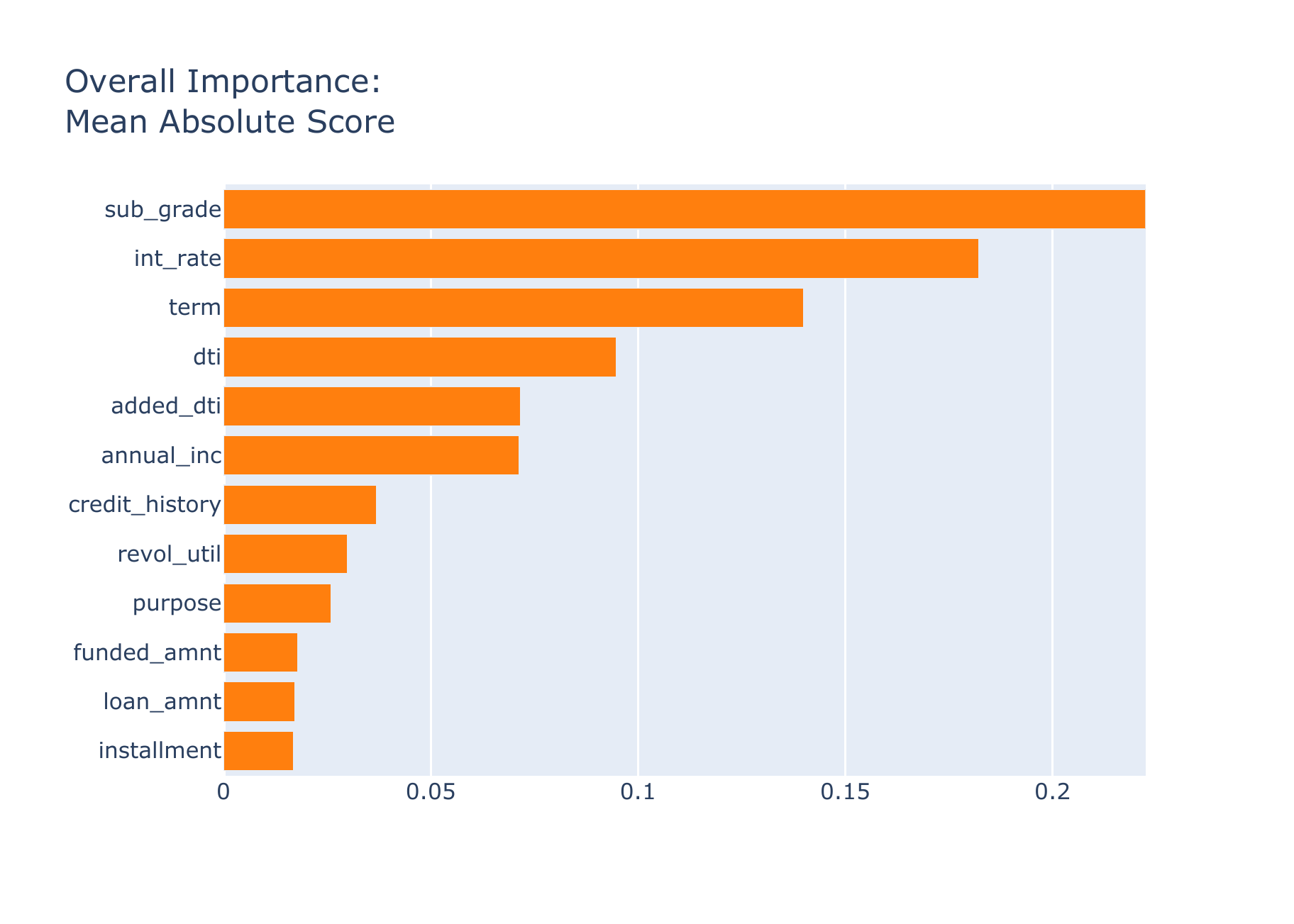}
    \includegraphics[width=0.495\textwidth]{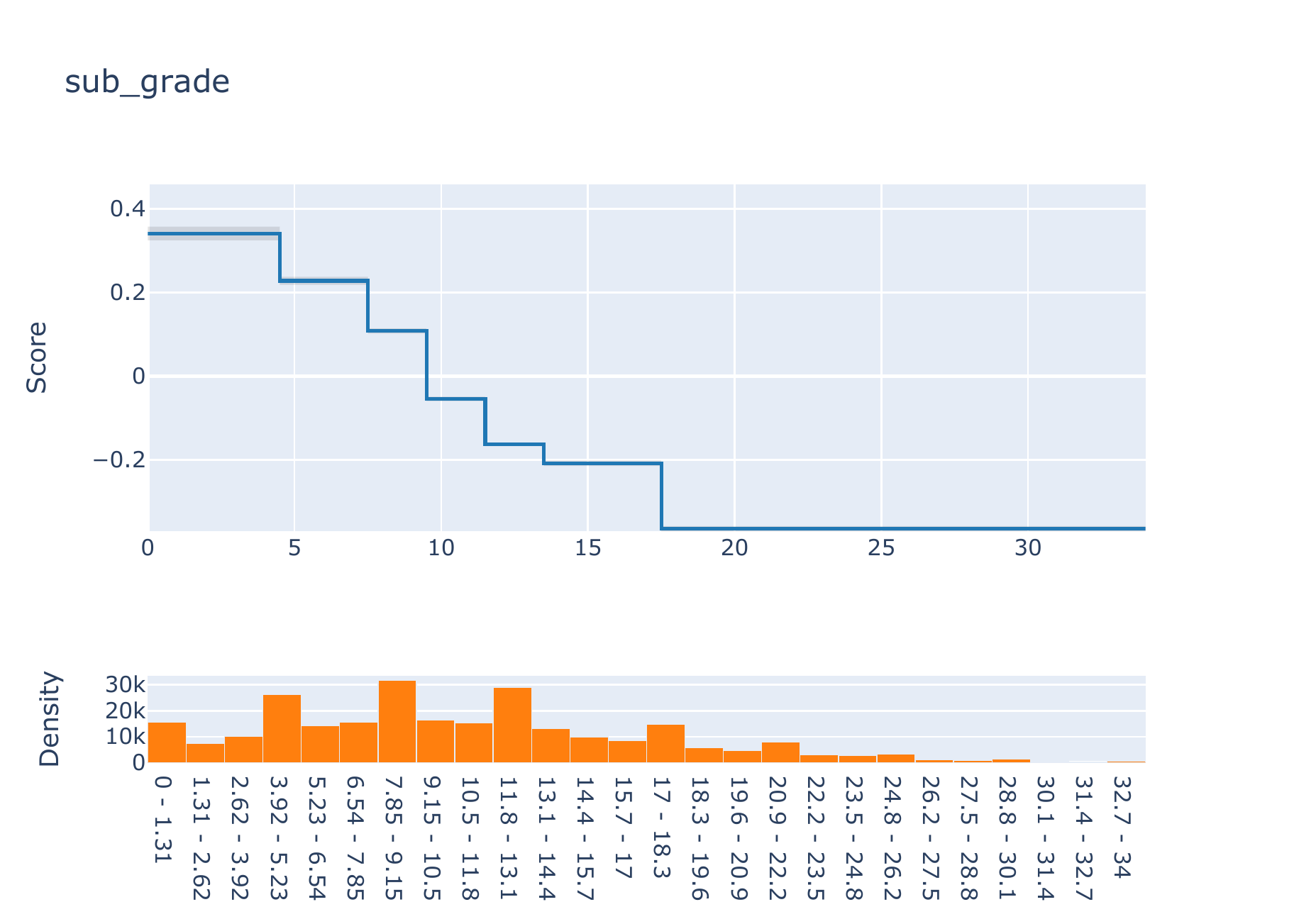}
    \includegraphics[width=0.495\textwidth]{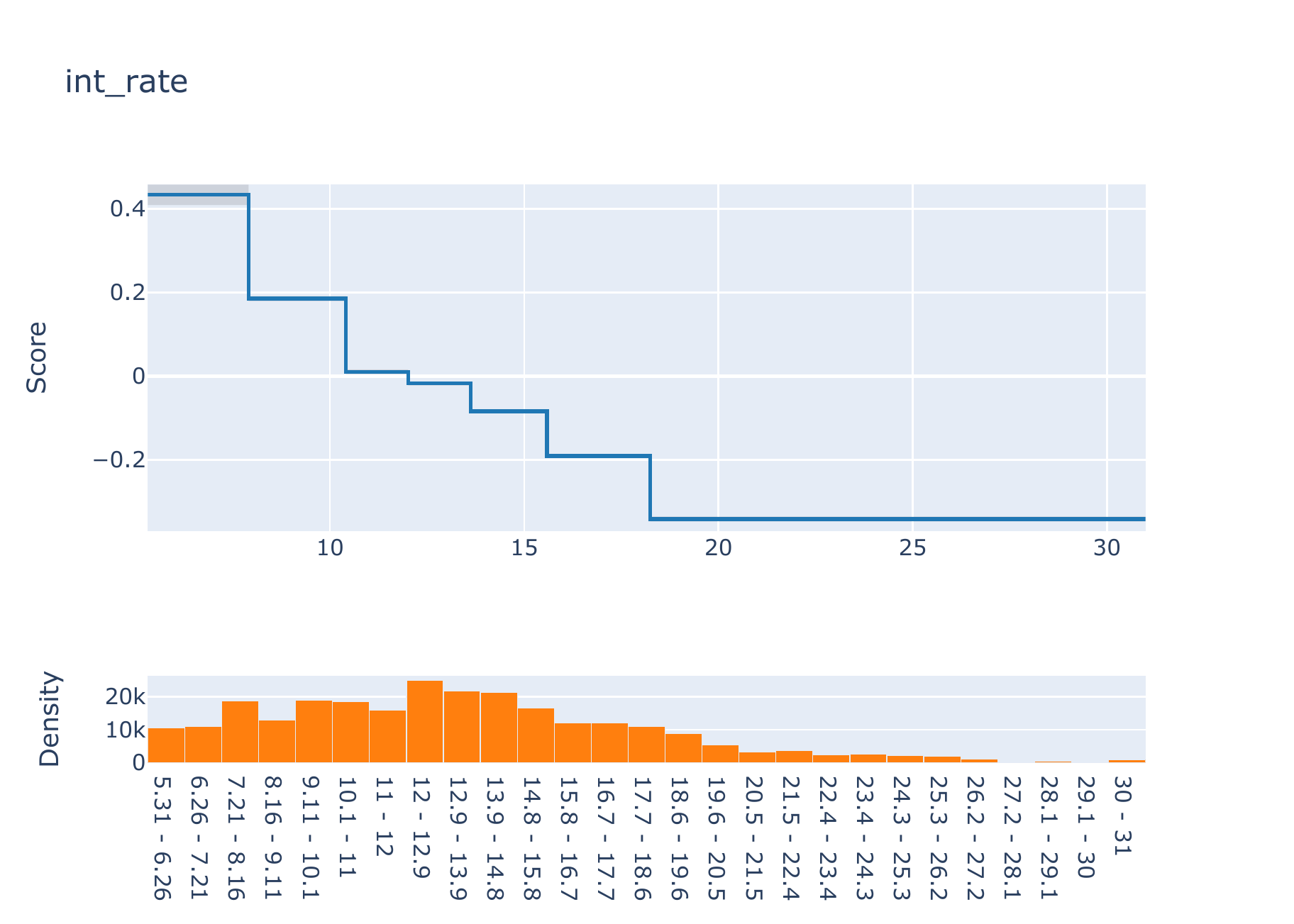}
    \includegraphics[width=0.495\textwidth]{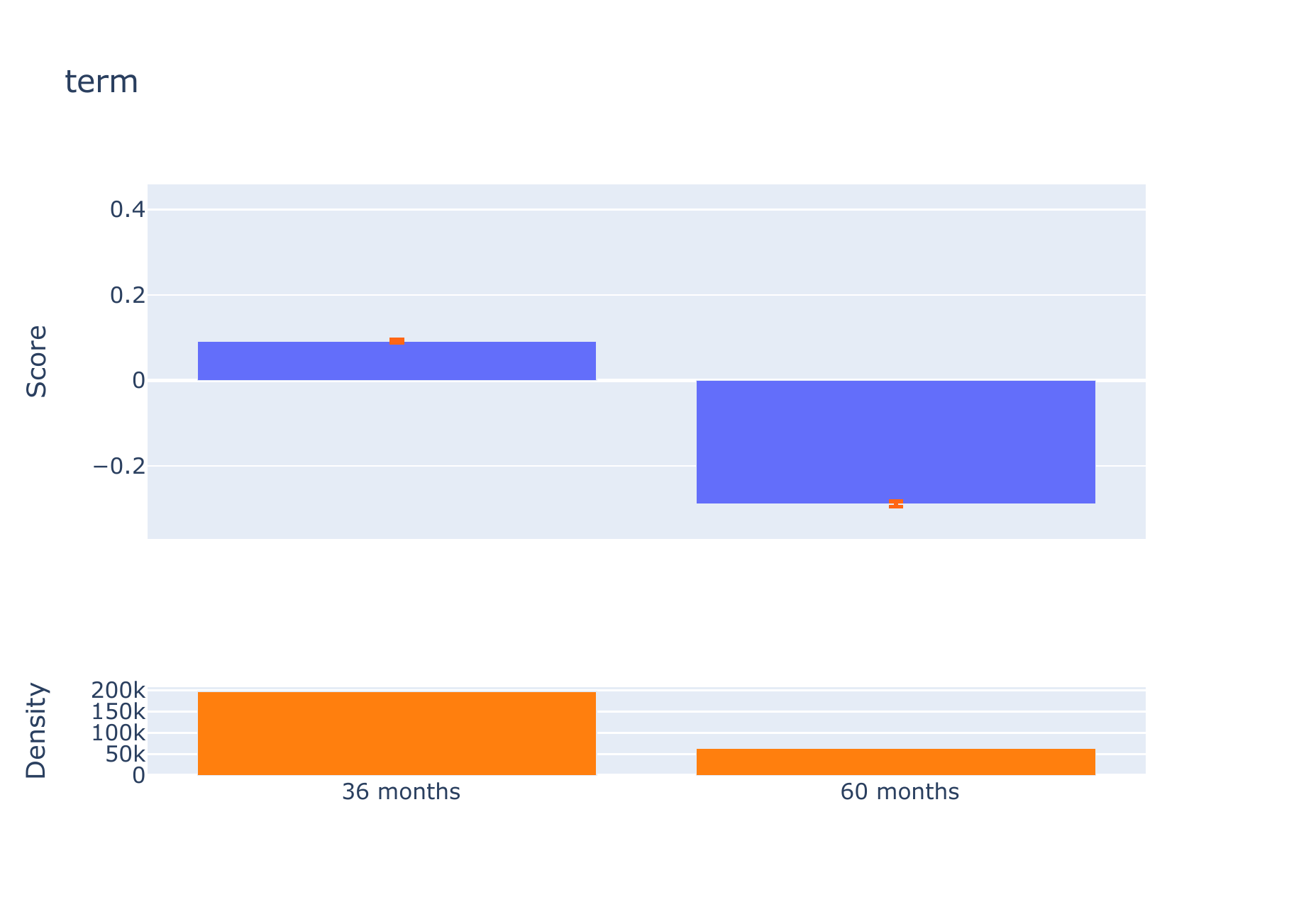}
    \includegraphics[width=0.495\textwidth]{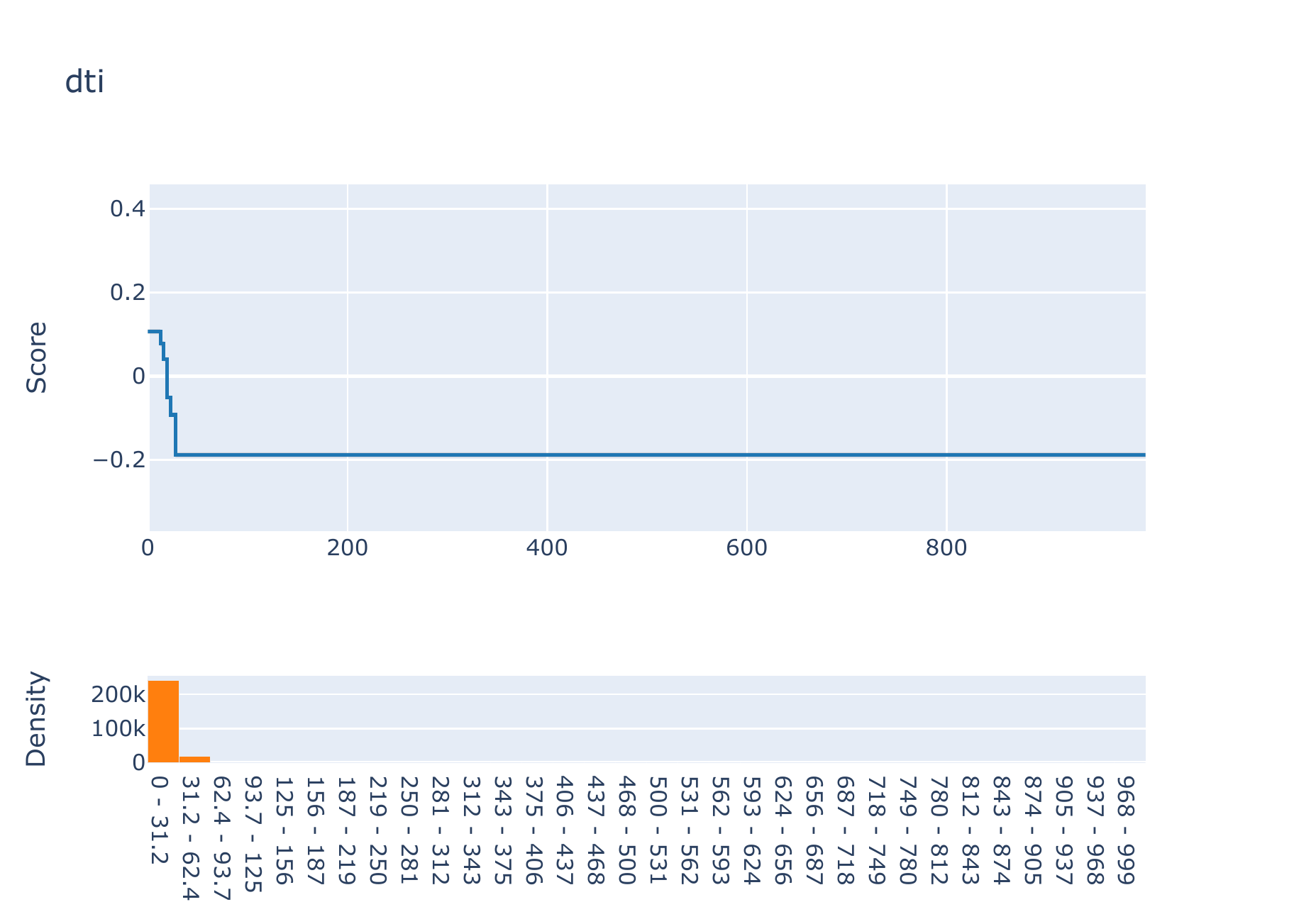}
    \includegraphics[width=0.495\textwidth]{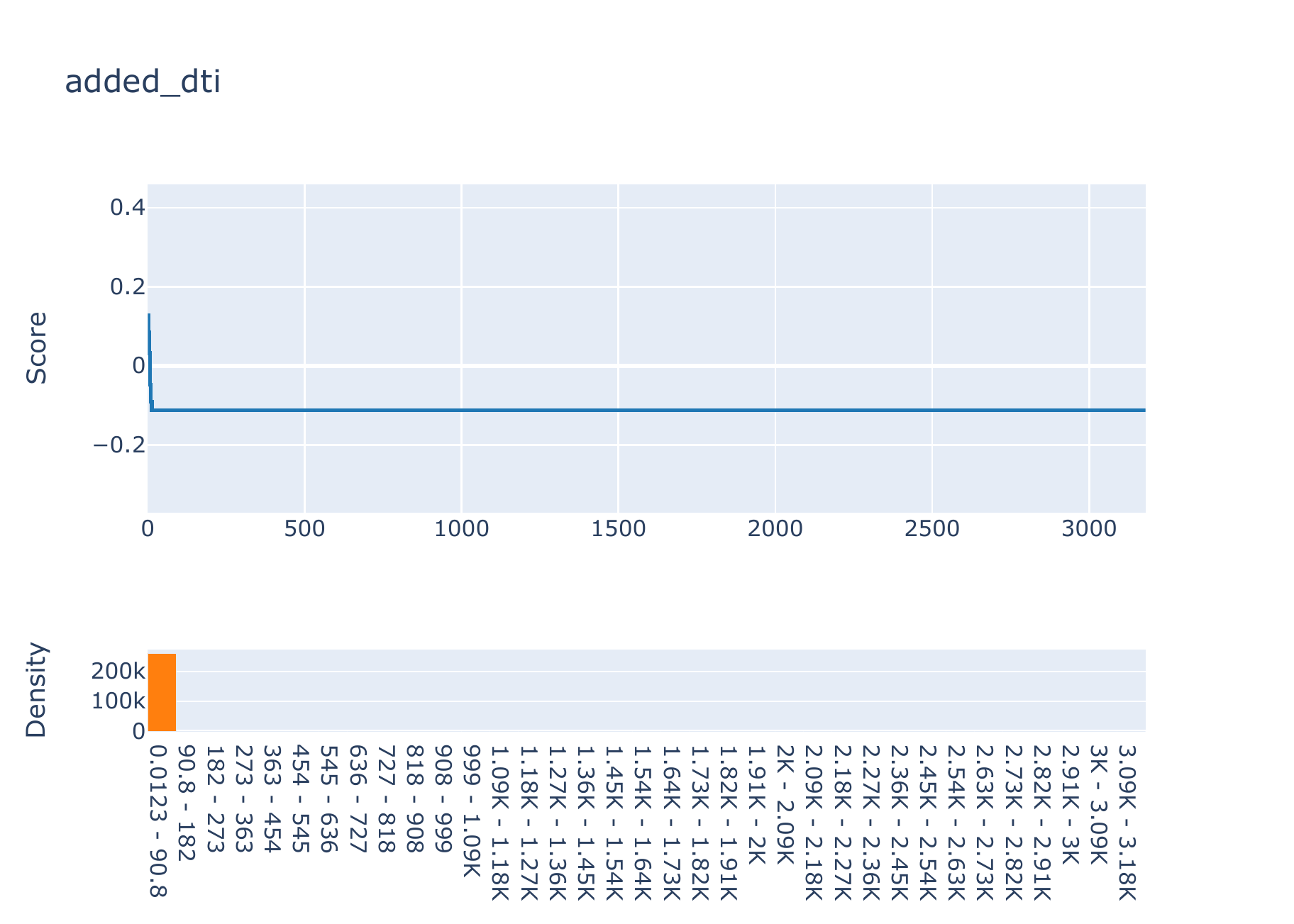}
    \includegraphics[width=0.495\textwidth]{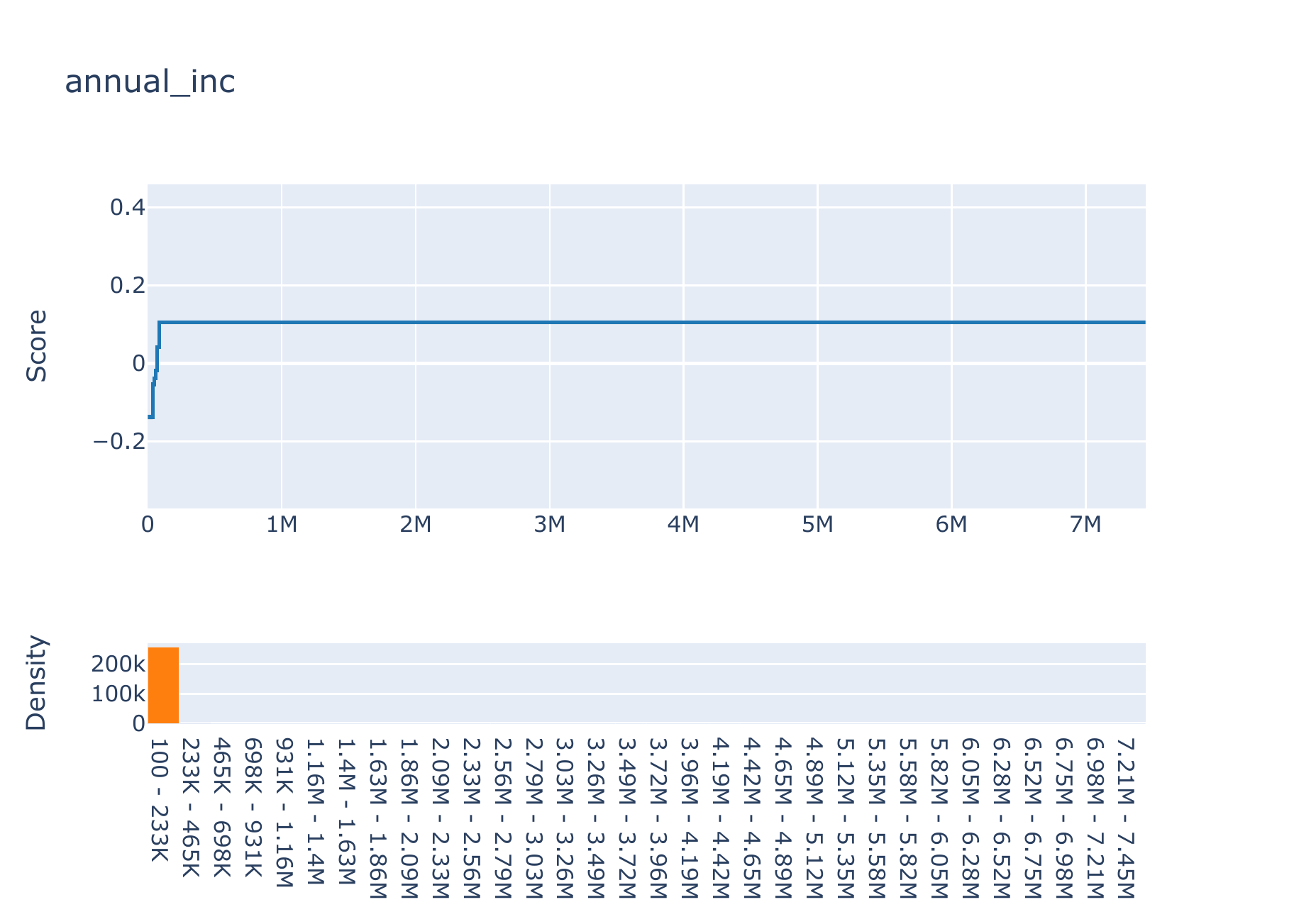}
    \includegraphics[width=0.495\textwidth]{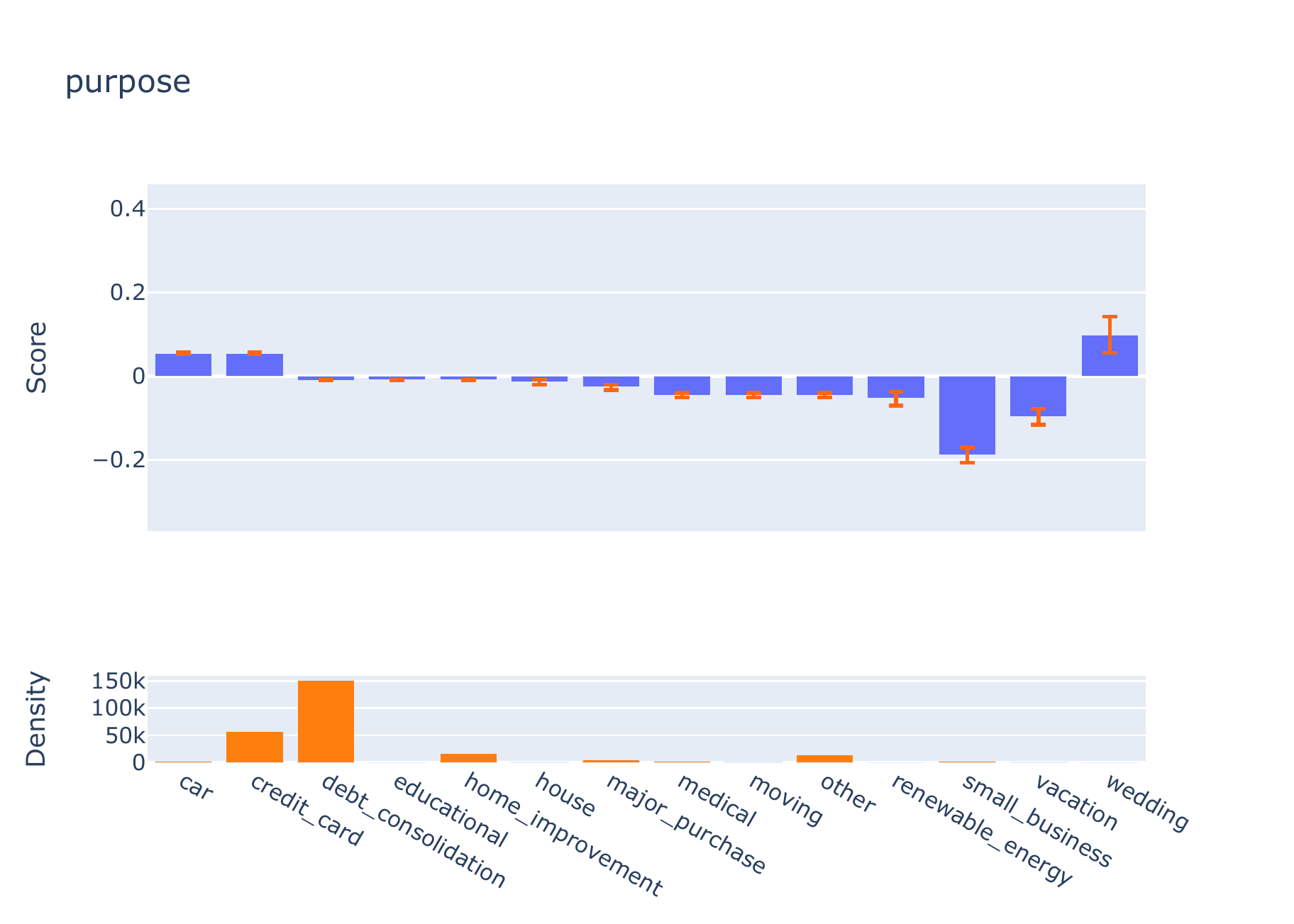}
    \caption{Feature importances and selected univariate functions $f_j$ for the Explainable Boosting Machine with \texttt{max\_bins} $= 8$ on the Lending Club dataset.}
    \label{fig:Lending_GAM_functions}
\end{figure}

\clearpage

\begin{figure}[ht]
    \centering
    \includegraphics[width=0.66\textwidth]{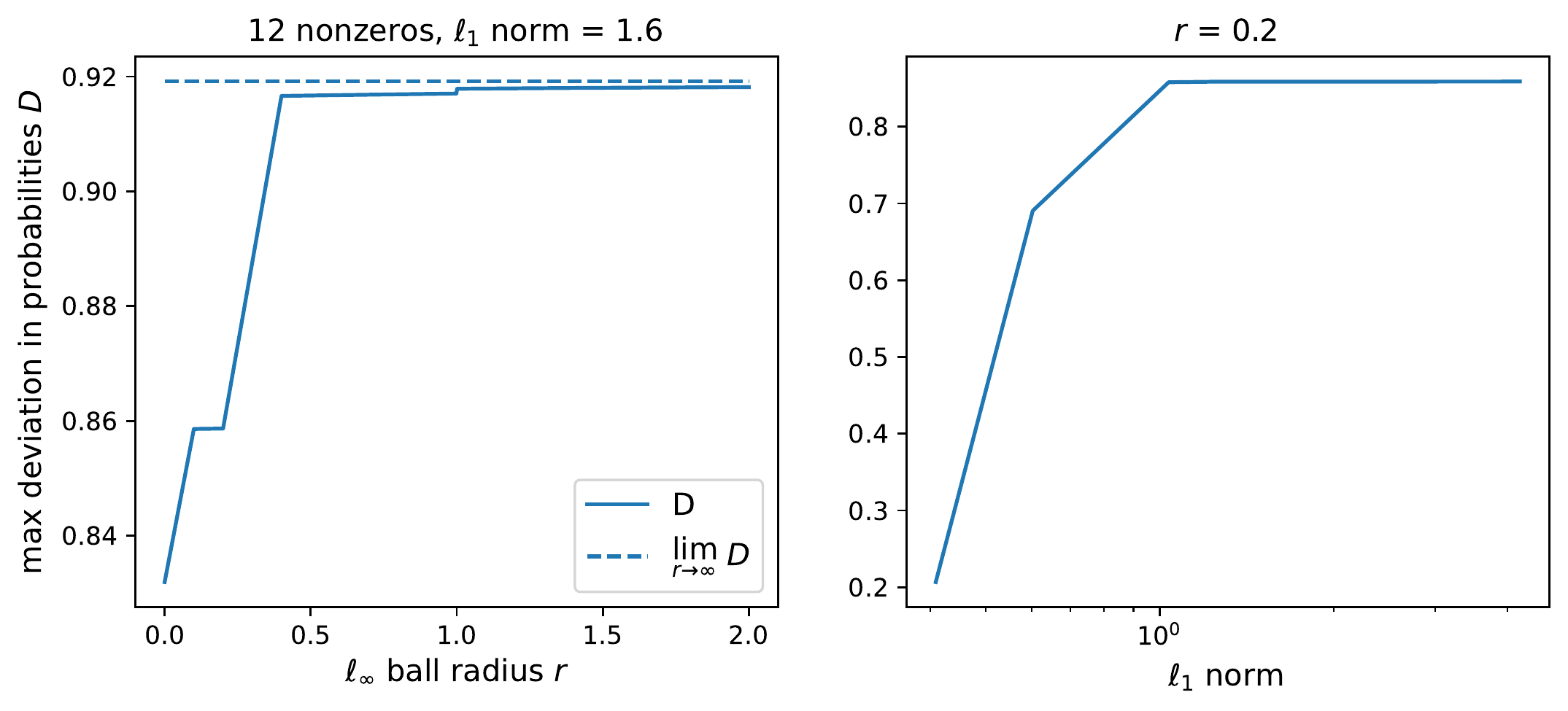}
    \caption{Maximum deviation $D$ for logistic regression models on the Lending Club dataset as a function of certification set size (radius $r$) and model smoothness ($\ell_1$ norm).}
    \label{fig:Lending_LR}
\end{figure}

\begin{figure}[ht]
    \centering
    \includegraphics[width=0.66\textwidth]{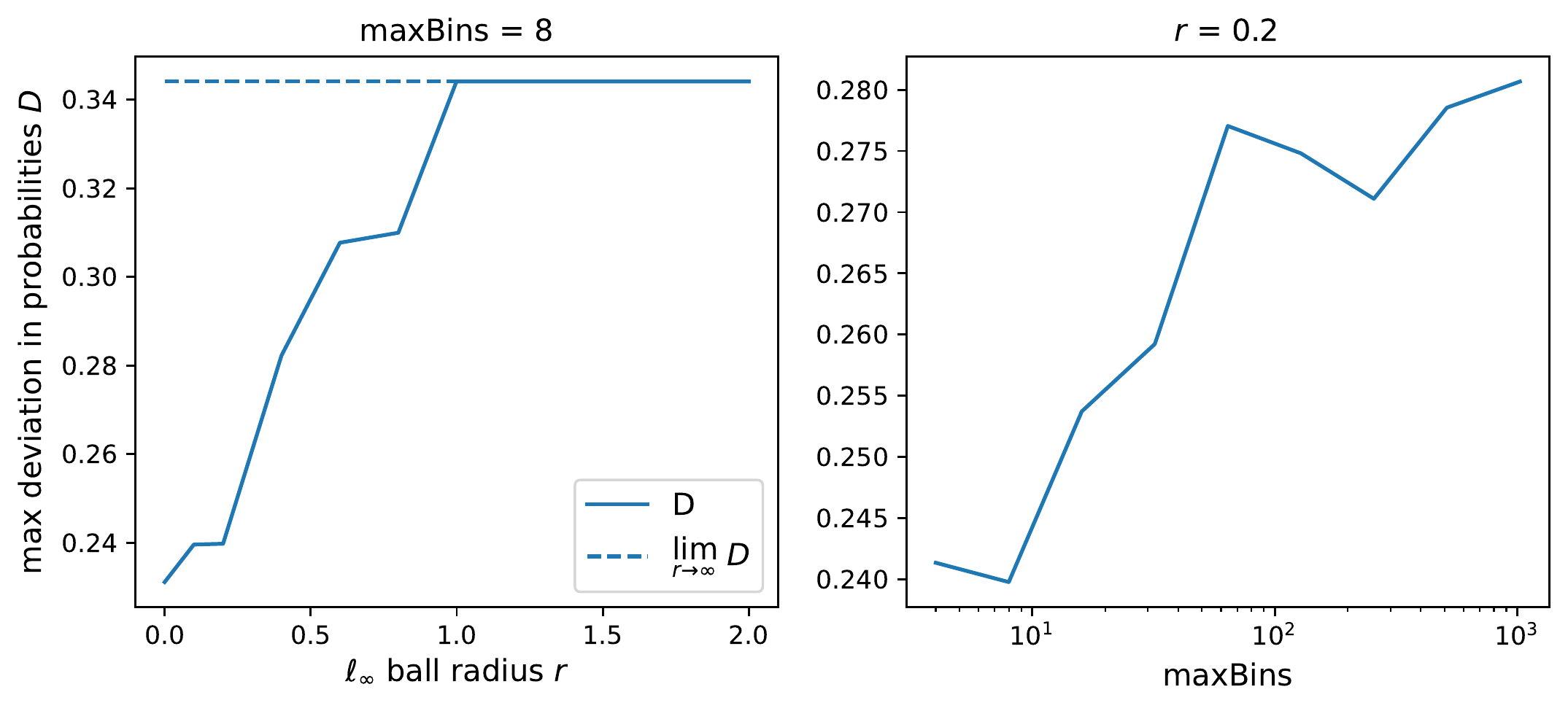}
    \caption{Maximum deviation $D$ for Explainable Boosting Machines on the Lending Club dataset as a function of certification set size (radius $r$) and model smoothness (\texttt{max\_bins} parameter).}
    \label{fig:Lending_GAM}
\end{figure}

\begin{figure}[ht]
     \centering
     \begin{subfigure}[b]{0.45\textwidth}
         \centering
         \includegraphics[width=\textwidth]{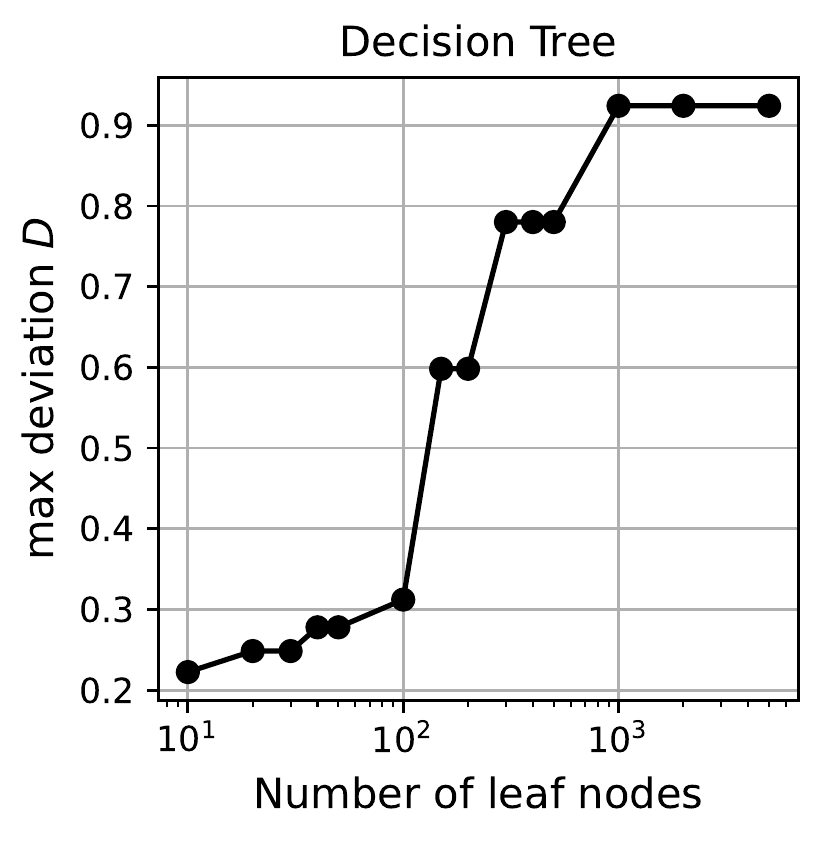}
         \caption{Decision Tree}
         \label{fig:expt:dt:lending:d}
     \end{subfigure}
     \begin{subfigure}[b]{0.45\textwidth}
         \centering
         \includegraphics[width=\textwidth]{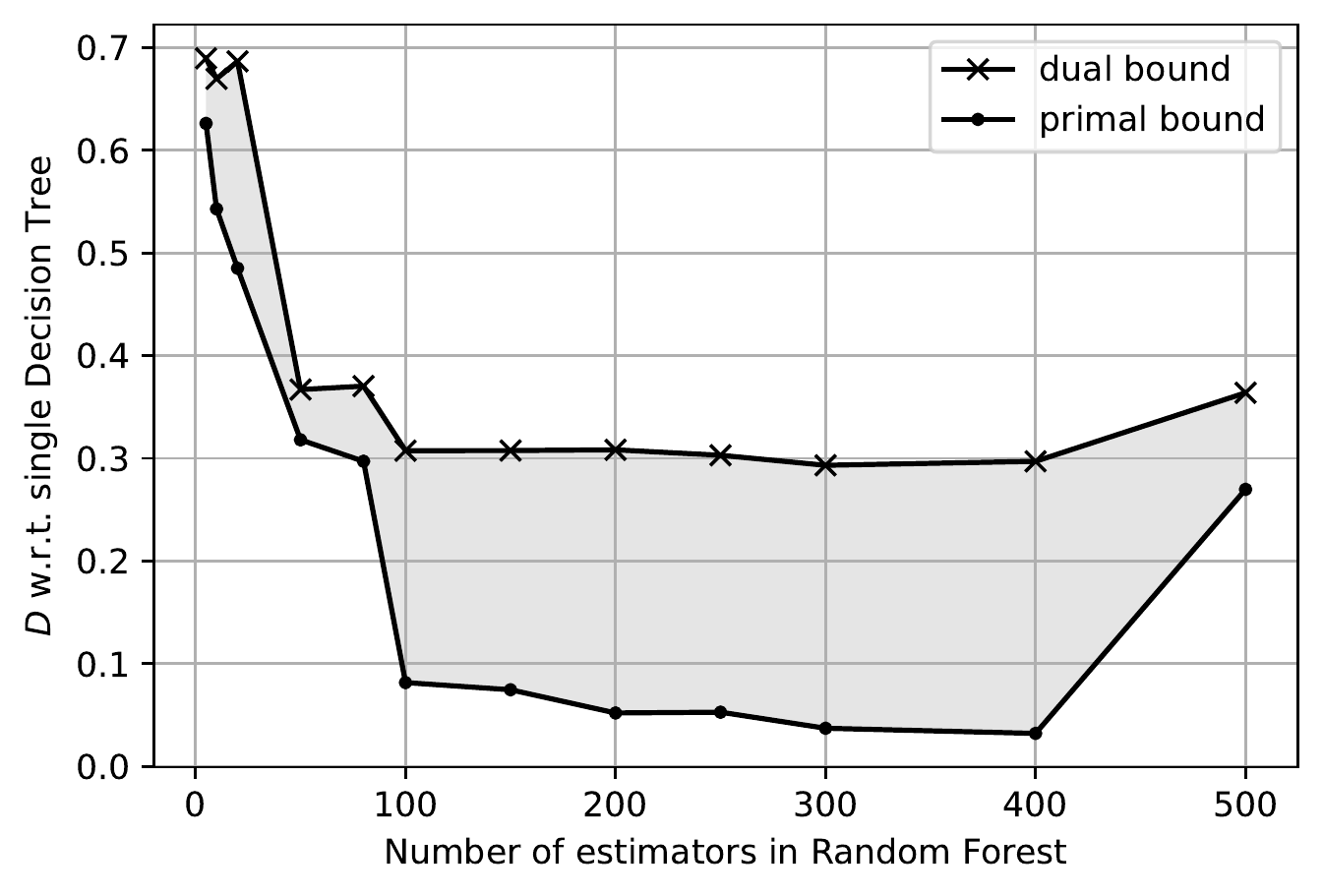}
         \caption{Random Forest}
         \label{fig:expt:ensemble:lending:d}
     \end{subfigure}
        \caption{Maximum deviations computed for tree and tree ensembles on the Lending Club dataset}
        \label{fig:expt:lending:d}
\end{figure}

\paragraph{Maximum deviation summary} Figures~\ref{fig:Lending_LR}--\ref{fig:expt:lending:d} show maximum deviation as functions of certification set radius $r$ and model complexity parameters, in a similar manner as  Figure~\ref{fig:Adult_all_add} for the Adult Income dataset. The qualitative patterns are similar to before: increasing maximum deviation in all cases except with the number of RF estimators in Figure~\ref{fig:expt:ensemble:lending:d}, where the upper bound is stable around $0.7$. A major quantitative difference is that the maximum deviations for the GAM in Figure~\ref{fig:Lending_GAM} are much lower than for the other models, in particular LR in Figure~\ref{fig:Lending_LR}. This is likely due to the fact that the GAM functions $f_j$ in Figure~\ref{fig:Lending_GAM_functions} are bounded while still being monotonic, unlike the linear functions $w_j x_j$ in the LR model.

\paragraph{Breakdown by leaves of $f_0$} Figures~\ref{fig:Lending_LR_leaves_r}--\ref{fig:Lending_EBM_leaves_maxBins} show a breakdown of the deviations for LR and GAM by leaves of the reference model, similar to Figures~\ref{fig:Adult_leaves_LR_r}--\ref{fig:Adult_leaves_GAM_maxBins} and again on the log-odds scale. One difference is that in Figure~\ref{fig:Lending_LR_leaves_r}, the deviations for finite $r$ do not come close to their $r \to \infty$ counterparts in most cases. In Figure~\ref{fig:Lending_EBM_leaves_r} however, the $r \to \infty$ values are all attained when $r$ is slightly greater than $1$.

\begin{figure}[t]
    \centering
    \includegraphics[width=\textwidth]{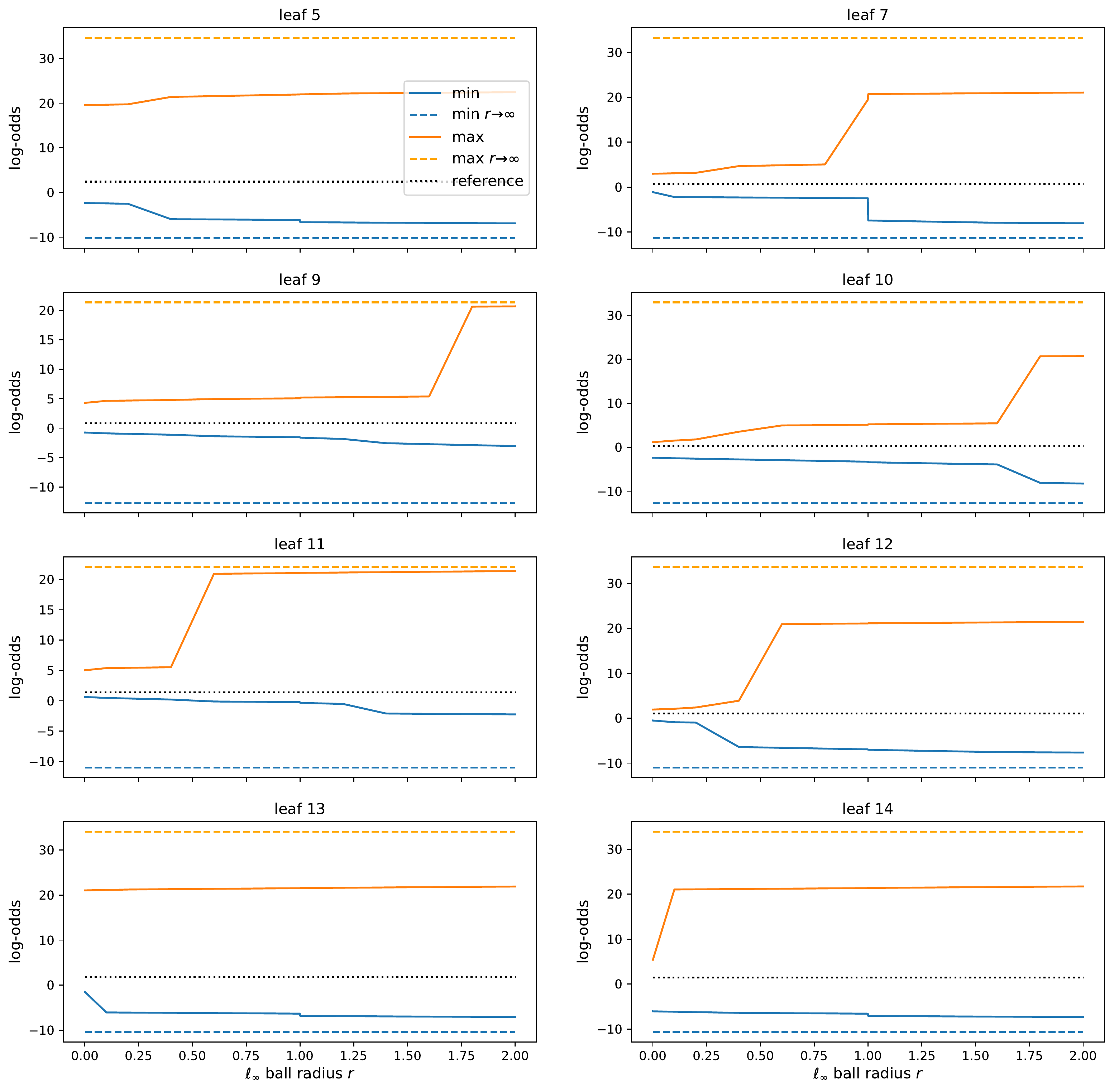}
    \caption{Minimum and maximum predicted log-odds for a logistic regression model with inverse $\ell_1$ penalty $C = 0.01$ on the Lending Club dataset, as a function of certification set size (radius $r$) and broken down by leaves of the decision tree reference model.}
    \label{fig:Lending_LR_leaves_r}
\end{figure}

\begin{figure}[t]
    \centering
    \includegraphics[width=\textwidth]{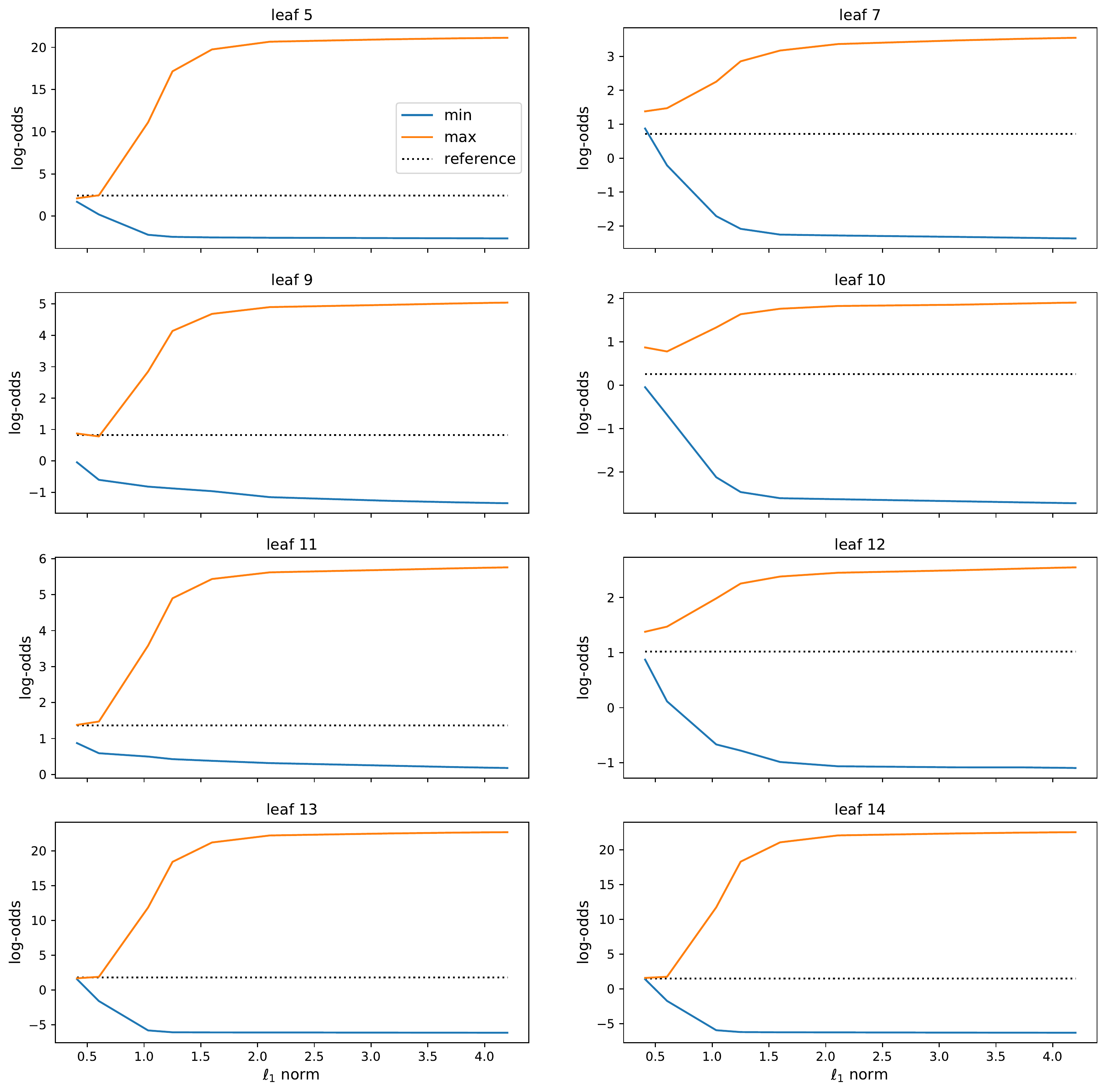}
    \caption{Minimum and maximum predicted log-odds for logistic regression models with different $\ell_1$ penalties $C$ on the Lending Club dataset, broken down by leaves of the decision tree reference model. The certification set $\ell_\infty$ ball radius is $r = 0.2$.}
    \label{fig:Lending_LR_leaves_C}
\end{figure}

\begin{figure}[t]
    \centering
    \includegraphics[width=\textwidth]{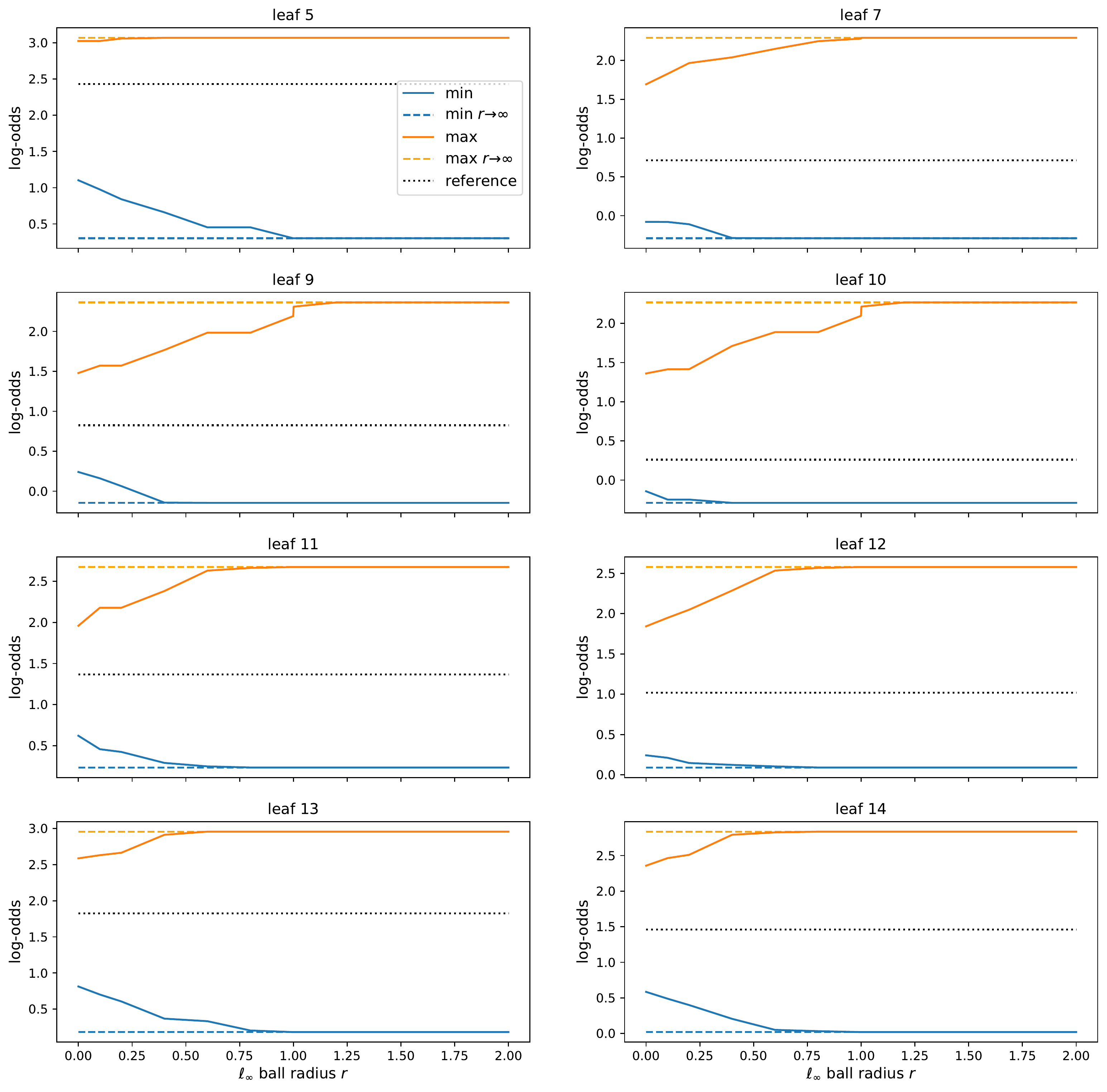}
    \caption{Minimum and maximum predicted log-odds for an Explainable Boosting Machine with \texttt{max\_bins} $= 8$ on the Lending Club dataset, as a function of certification set size (radius $r$) and broken down by leaves of the decision tree reference model.}
    \label{fig:Lending_EBM_leaves_r}
\end{figure}

\begin{figure}[t]
    \centering
    \includegraphics[width=\textwidth]{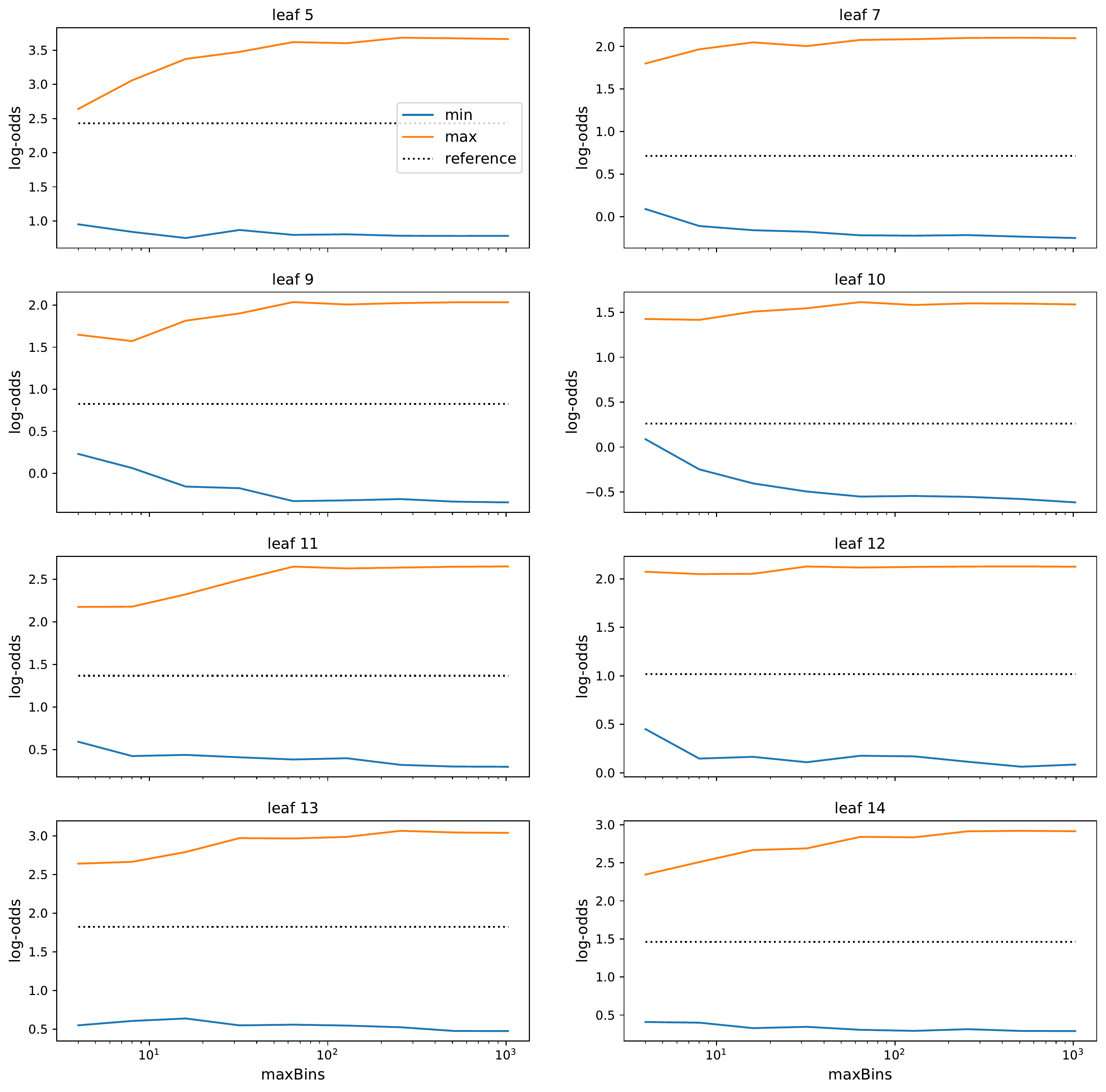}
    \caption{Minimum and maximum predicted log-odds for Explainable Boosting Machines with different \texttt{max\_bins} values on the Lending Club dataset, broken down by leaves of the decision tree reference model. The certification set $\ell_\infty$ ball radius is $r = 0.2$.}
    \label{fig:Lending_EBM_leaves_maxBins}
\end{figure}

\clearpage

\paragraph{Running time, maximal cliques evaluated} Figure~\ref{fig:Lending_time} shows the time required to compute the maximum deviation for LR and GAM on the Lending Club dataset. Figure~\ref{fig:expt:lending:pruning} shows the number of $K+1$-maximal cliques evaluated for DT and RF as well as the number of nodes in the graph. The observations are the same as in Figures~\ref{fig:Adult_time} and \ref{fig:expt:pruning}.

\begin{figure}[ht]
    \centering
    \includegraphics[width=0.495\textwidth]{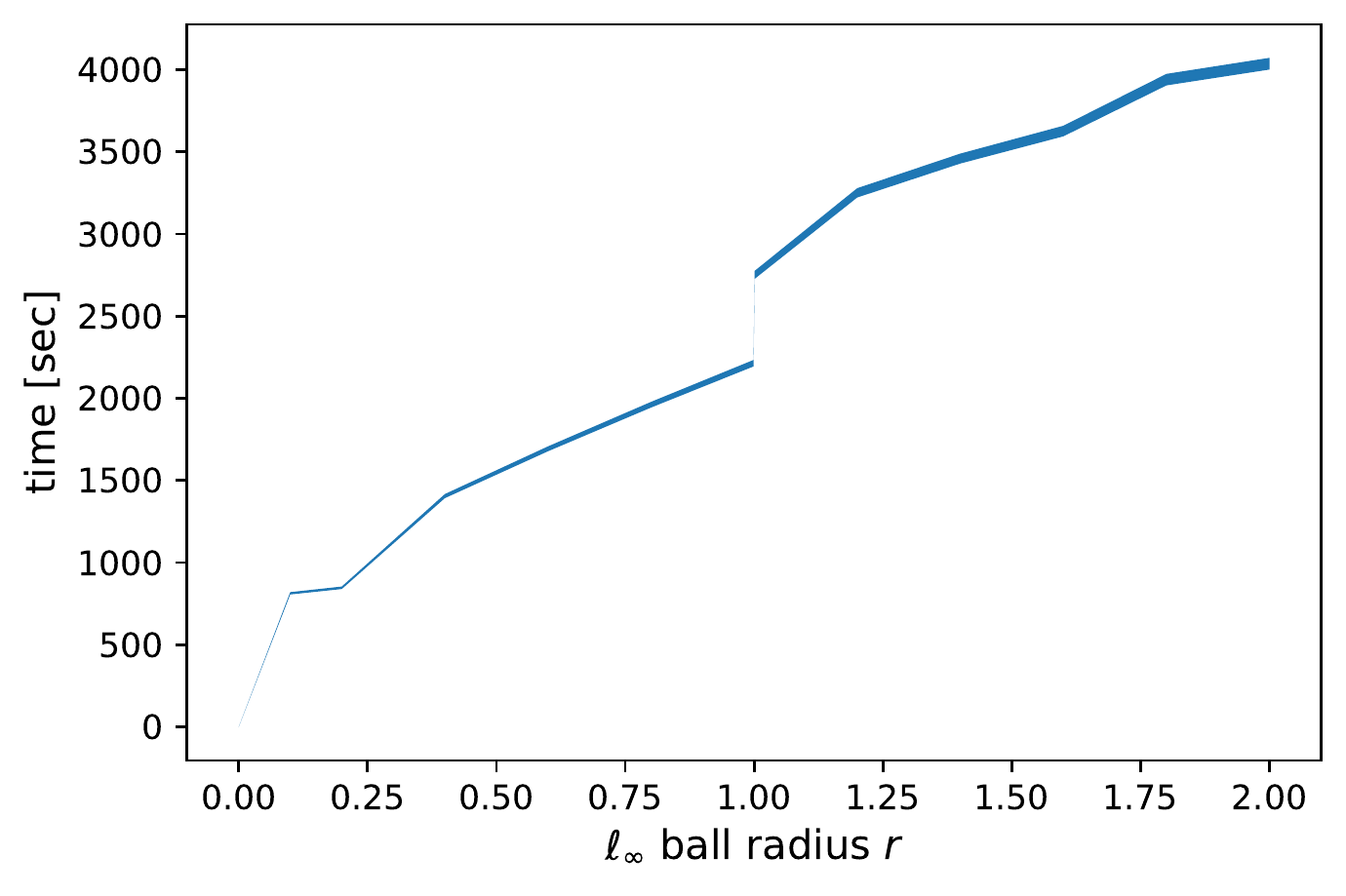}
    \includegraphics[width=0.495\textwidth]{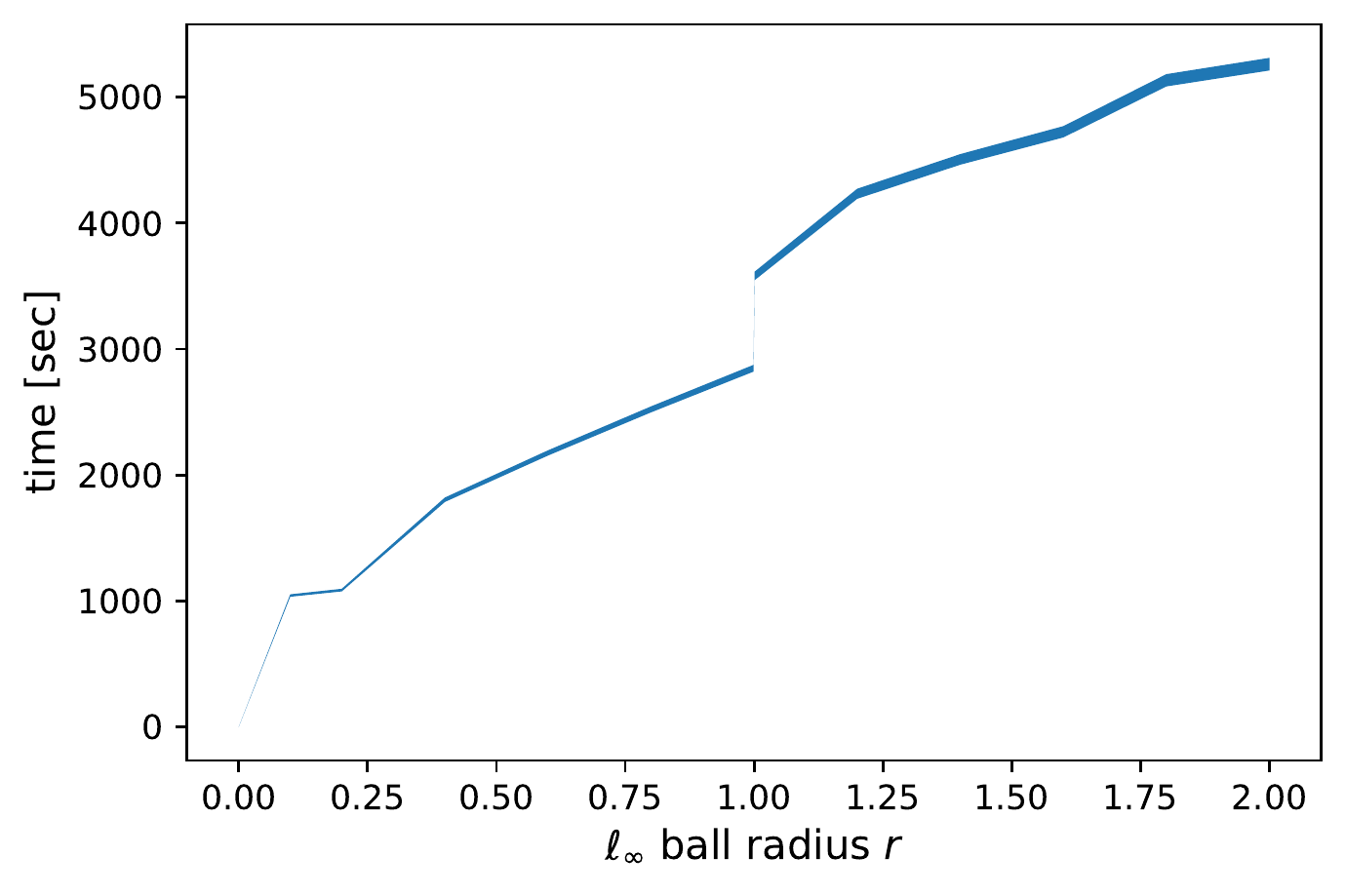}
    \caption{Time to compute maximum deviation for logistic regression models (left) and Explainable Boosting Machines (right) on the Lending Club dataset as a function of certification set size (radius $r$). The filled-in region shows the min-max variation with model complexity ($\ell_1$ norm for LR, \texttt{max\_bins} for EBM).}
    \label{fig:Lending_time}
\end{figure}

\begin{figure}[ht]
     \centering
     \begin{subfigure}[b]{0.4\textwidth}
         \centering
         \includegraphics[width=\textwidth]{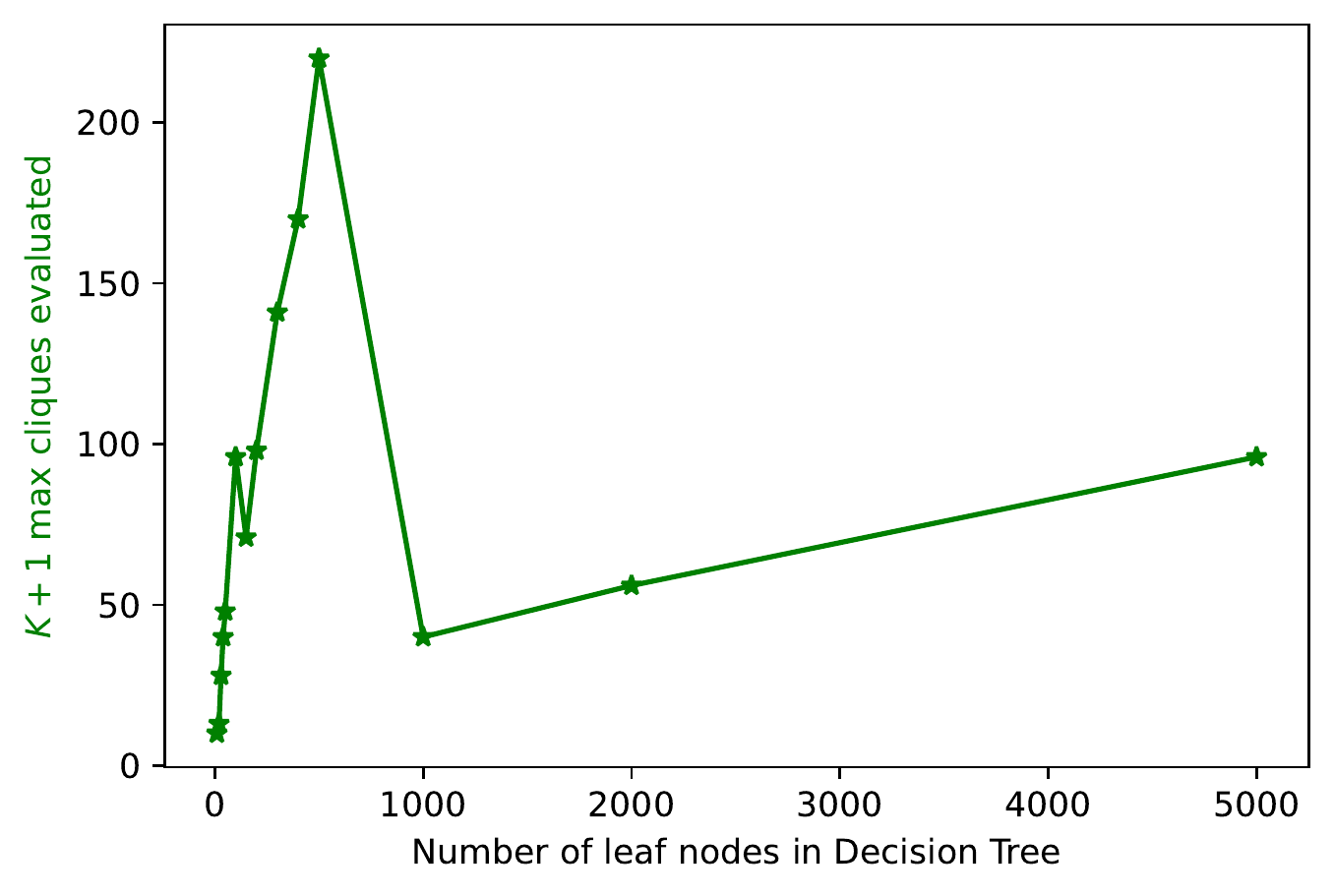}
         \caption{Decision Tree}
         \label{fig:expt:lending:dt:size}
     \end{subfigure}
     \begin{subfigure}[b]{0.45\textwidth}
         \centering
         \includegraphics[width=\textwidth]{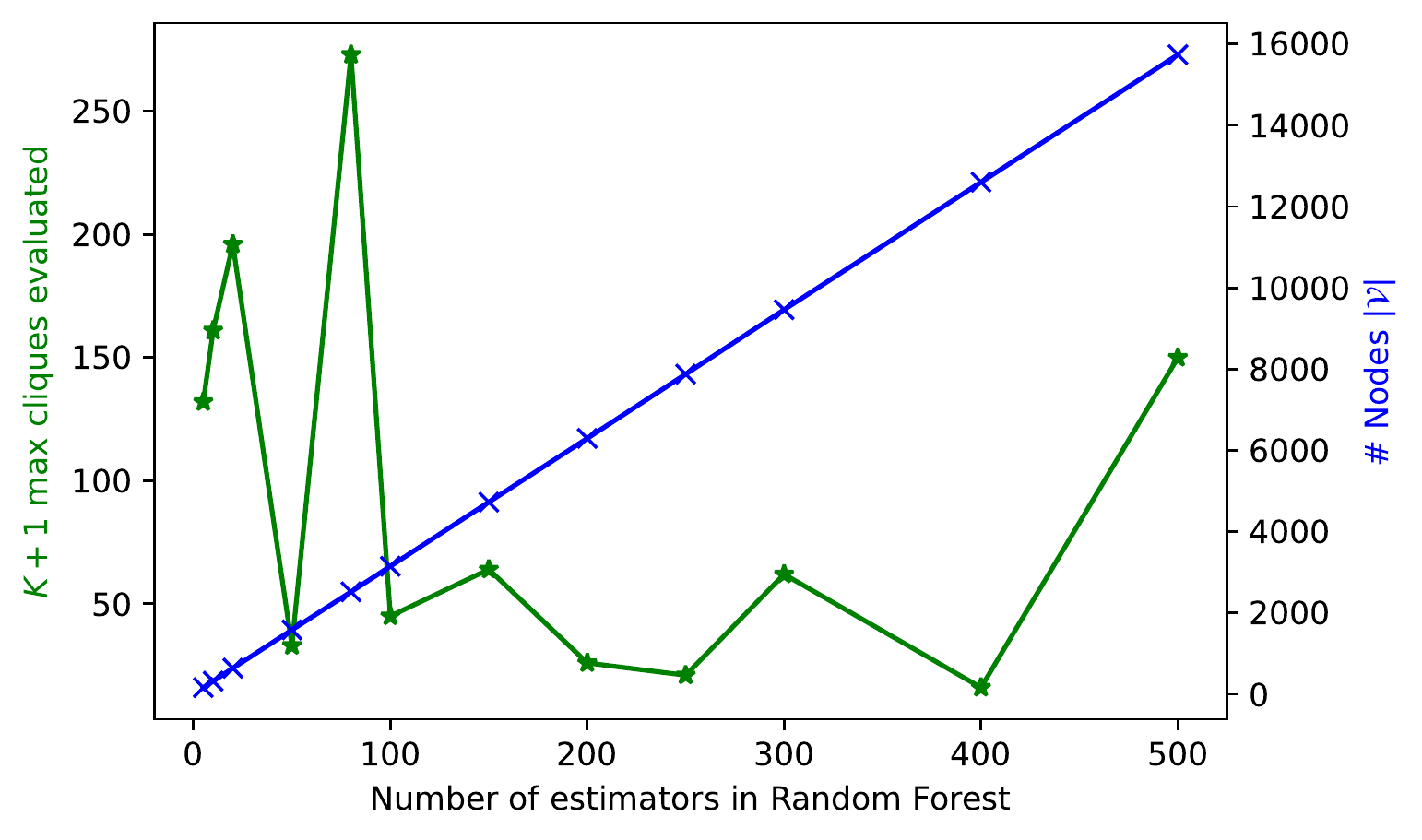}
         \caption{Random Forest}
         \label{fig:expt:lending:ensemble:size}
     \end{subfigure}
        \caption{Effectiveness of pruning by bound for tree-based models on the Lending Club dataset}
        \label{fig:expt:lending:pruning}
\end{figure}

\paragraph{Feature combinations that maximize deviation} Tables~\ref{tab:Lending_LR} and \ref{tab:Lending_GAM} present the feature values that maximize deviation for LR and GAM respectively. For both tables, the minimum log-odds over leaf 5 of the DT reference model is selected (corresponding to Figures~\ref{fig:Lending_LR_leaves_r} and \ref{fig:Lending_EBM_leaves_r}, leaf 5, blue curve) because this choice maximizes the deviation overall for most values of $r$. 

\begin{table}[ht]
    \small
    \centering
    \begin{tabular}{lrrrrrr}
    \toprule
    $r$&dti&annual\_inc&term&purpose&int\_rate&credit\_history\\
    \midrule
    0.0&505&1700& 36 months&credit\_card&7.4&2557\\
    0.1&506&100& 36 months&credit\_card&6.9&2283\\
    0.2&507&100& 36 months&credit\_card&6.4&2009\\
    0.4&884&100& 36 months&debt\_consolidation&10.1&3897\\
    0.6&886&100& 36 months&debt\_consolidation&9.1&3349\\
    0.8&889&100& 36 months&debt\_consolidation&8.2&2801\\
    0.999&891&100& 36 months&debt\_consolidation&7.2&2256\\
    1.001&891&100& 60 months&small\_business&7.2&2251\\
    1.2&893&100& 60 months&small\_business&6.3&1706\\
    1.4&895&100& 60 months&small\_business&5.3&1158\\
    1.6&898&100& 60 months&small\_business&5.3&1095\\
    1.8&900&100& 60 months&small\_business&5.3&1095\\
    2.0&902&100& 60 months&small\_business&5.3&1095\\
    $\infty$&999&100& 60 months&small\_business&5.3&1095\\
    \bottomrule
    \end{tabular}
    \caption{Feature values that minimize log-odds for a logistic regression model ($C = 0.01$) over leaf 5 of the decision tree reference model for the Lending Club dataset. The 6 features that contribute most to the minimum are shown as a function of certification set radius $r$.}
    \label{tab:Lending_LR}
\end{table}

The previous trend toward extreme feature values that minimize log-odds also holds here as $r$ increases. For example, debt-to-income ratios increase to outlier values above $100\%$, annual income drops to the minimum of $\$100$, the term changes to 60 months in Table~\ref{tab:Lending_LR} for $r > 1$, and the purpose changes to small business, the category with the lowest log-odds in Figure~\ref{fig:Lending_GAM_functions}. The decrease in income and increases in debt-to-income ratios are qualitatively in accordance with each other. However, these quantities are related to each other by deterministic formulas (at least in theory) that also involve the interest rate and installment amount. It is not clear whether the values in Tables~\ref{tab:Lending_LR} and \ref{tab:Lending_GAM} violate these relationships. This may be an instance that could benefit from constraints on possible feature combinations, as briefly mentioned in Appendix~\ref{sec:discuss}.

\begin{table}[ht]
    \small
    \centering
    \begin{tabular}{lrrrrrr}
    \toprule
    $r$&term&int\_rate&dti&purpose&annual\_inc&added\_dti\\
    \midrule
    0.0& 60 months&9.9&46.4&debt\_consolidation&35000&25.5\\
    0.1& 60 months&[10.4 11. ]&[28.6 30.9]&debt\_consolidation&[31800 38001]&[10.  11.2]\\
    0.2& 60 months&[12.  13.1]&[27.8 31.8]&debt\_consolidation&[36600 38001]&[10.  11.7]\\
    0.4& 60 months&[13.6 14.3]&[27.8 30.4]&small\_business&[  100 38001]&[12.7 15.3]\\
    0.6& 60 months&[15.6 16.3]&[27.8 36.1]&small\_business&[19801 38001]&[12.7 15. ]\\
    0.8& 60 months&[15.6 16.4]&[27.8 28.4]&small\_business&[14402 38001]&[12.7 15.7]\\
    0.999& 60 months&[18.2 18.4]&[27.8 38.8]&small\_business&[33069 38001]&[12.7 14.5]\\
    1.001& 60 months&[18.2 18.8]&[27.8 35.3]&small\_business&[  935 38001]&[12.7 18.5]\\
    1.2& 60 months&[18.2 20.2]&[27.8 33.3]&small\_business&[ 7602 38001]&[12.7 18.1]\\
    1.4& 60 months&[18.2 21.2]&[27.8 35.6]&small\_business&[  100 38001]&[12.7 20.4]\\
    1.6& 60 months&[18.2 19.2]&[27.8 29.9]&small\_business&[  100 38001]&[12.7 24.5]\\
    1.8& 60 months&[18.2 26.1]&[27.8 28.5]&small\_business&[  100 38001]&[12.7 24.4]\\
    2.0& 60 months&[18.2 27.1]&[27.8 30.7]&small\_business&[  100 38001]&[12.7 26.6]\\
    $\infty$& 60 months&[18.2 31. ]&[ 27.8 999. ]&small\_business&[  100 38001]&[  12.7 3179.3]\\
    \bottomrule
    \end{tabular}
    \caption{Feature values that minimize log-odds for an Explainable Boosting Machine (\texttt{max\_bins} $= 8$) over leaf 5 of the decision tree reference model for the Lending Club dataset. The 6 features that contribute most to the minimum are shown as a function of certification set radius $r$. For $r > 0$, the minimizing values of continuous features form an interval because the corresponding functions $f_j$ are piecewise constant.}
    \label{tab:Lending_GAM}
\end{table}